\documentclass{article}
\PassOptionsToPackage{numbers,sort&compress}{natbib}
\usepackage[main,final]{neurips_2026}
\usepackage{palatino}
\usepackage{mathpazo}
\usepackage{inconsolata}

\usepackage[utf8]{inputenc} % allow utf-8 input
\usepackage[T1]{fontenc}    % use 8-bit T1 fonts
\usepackage{url}            % simple URL typesetting
\usepackage{booktabs}       % professional-quality tables
\usepackage{amsfonts}       % blackboard math symbols
\usepackage{nicefrac}       % compact symbols for 1/2, etc.
\usepackage{microtype}      % microtypography
\usepackage{xcolor}         % colors

\usepackage{microtype}
\usepackage{graphicx}
\usepackage{subcaption}
\usepackage{booktabs} % for professional tables
\usepackage{multicol,multirow}
\usepackage{enumitem}
\usepackage{caption,subcaption}
\usepackage{wrapfig}

\usepackage[
  colorlinks=true,
  linkcolor=magenta,
  citecolor=magenta,
  urlcolor=magenta
]{hyperref}

\usepackage{amsmath}
\usepackage{amssymb}
\usepackage{mathtools}
\usepackage{amsthm}
\usepackage{algorithm}
\usepackage{algorithmic}
\usepackage[capitalize,noabbrev]{cleveref}

\theoremstyle{plain}
\newtheorem{theorem}{Theorem}[section]

\theoremstyle{definition}

\theoremstyle{remark}

\usepackage{xspace}
\newcommand{\method}{\textsc{SS-eSOAP}\xspace}
\newcommand{\ssbfgs}{\textsc{SS-BFGS}\xspace}

\title{\method: Self-Scaled Adaptive Preconditioning for Physics-Informed Learning}

\author{
\small
Guangyuan Wang\textsuperscript{1,2} \quad
Mads Toftrup\textsuperscript{3} \quad
Sebastian Loeschcke\textsuperscript{4} \\
\small
\bf Yixuan Wang\textsuperscript{2} \quad
Anima Anandkumar\textsuperscript{2}
\\[0.35em]
\small
\textsuperscript{1}McGill University \quad
\textsuperscript{2}California Institute of Technology \\
\small
\textsuperscript{3}Aarhus University \quad
\textsuperscript{4}University of Copenhagen
}

\begin{document}

\maketitle

\begin{abstract}
Physics-informed neural networks (PINNs) often face ill-conditioned objectives
that limit high-accuracy training. Dense quasi-Newton methods improve local
conditioning but require expensive optimizer state, while Kronecker-factored
methods such as SOAP scale to larger networks but rely on periodic basis
updates. We introduce \method, which augments SOAP-style preconditioning with
a scalar secant-energy correction adapted to Kronecker geometry and an
adaptive basis update followed by variance-state downscaling. We characterize
the directional secant matching induced by the scalar correction and give a
bound on variance-state mismatch across basis changes. Across eight PDE
benchmarks, \method attains the lowest final residual on six, including Burgers
and Boussinesq, while SOAP-family baselines perform better on Gray-Scott and
Ginzburg-Landau. On Boussinesq, \method reaches a residual of $10^{-5}$ in
4.1 hours with 9.2 GB peak VRAM, while Adam does not reach this target within
14 hours. Three-seed $L^2$ and $H^1$ errors on four representative PDEs support
the link between lower residuals and improved solution accuracy. These results
position \method as a scalable option for stiff, high-accuracy physics-informed training,
rather than a uniform replacement for existing optimizers.
\end{abstract}

\section{Introduction}
% \seb{We should use "\ + method" instead of writing out method all the time.}
High-precision optimization is central to scientific machine learning. One
example is the search for finite-time singularities in the three-dimensional
Euler and Navier-Stokes equations. PINNs have been used to reconstruct
candidate self-similar profiles in unbounded domains
\citep{wang2023asymptotic,wang2025discovery,wang2025high}. Such candidates
require low residuals before computer-assisted analysis. PINN objectives are
often ill-conditioned near this accuracy regime, where standard first-order
training may plateau
\citep{wang2021understanding,krishnapriyan2021characterizing,
rathore2024challenges,wang2025gradient,xu2025fp64}.

% More broadly, the optimization of PINNs presents unique challenges compared to standard supervised learning \citep{krishnapriyan2021characterizing}. The loss landscape of PINNs is characterized by extreme stiffness and pathological curvature caused by competing objectives of boundary conditions and differential operators \citep{karniadakis2023stiffpinns} \seb{does this reference say "extreme stiffness?}. 
% % In these settings, standard first-order optimizers like Adam often stagnate in local minima
% % % \yw{really? more like second order right} \alex{adam is first order}
% % or suffer from slow convergence in these high-curvature regimes \citep{urban2025unveiling}.
% Standard first-order optimizers like Adam often exhibit slow or unstable convergence due to severe curvature anisotropy, particularly in high-precision regimes \citep{urban2025unveiling}.
More broadly, PINN optimization presents challenges that differ from standard supervised learning. PINN losses combine PDE residuals with boundary and initial-condition terms, which can induce gradient imbalance, numerical stiffness, and ill-conditioned loss landscapes~\citep{wang2021understanding,krishnapriyan2021characterizing,rathore2024challenges}. As a result, standard first-order optimizers such as Adam often converge slowly or plateau in high-accuracy regimes, motivating curvature-aware methods for PINN training~\citep{urban2025unveiling}.

Second-order methods use curvature information via Hessian-approximation updates, enabling more efficient descent directions than first-order methods in many non-convex problems. Quasi-Newton methods such as BFGS~\citep{bfgs_broyden} and L-BFGS~\citep{liu1991limited} build these approximations from gradient differences, avoiding explicit second derivatives while retaining favorable local convergence. Self-scaled variants, such as \ssbfgs, rescale the Hessian approximation to better match the local curvature spectrum~\citep{oren1974selfscaling,albaali1992efficient,nocedal1993selfscaling,urban2025unveiling}. However, these methods operate on the flattened vector, requiring $O(d_{\mathrm{in}}^2d_{\mathrm{out}}^2)$ for a single $d_{\mathrm{out}}\times d_{\mathrm{in}}$ matrix parameter, making exact self-scaled quasi-Newton updates impractical for deep networks.
Kronecker-factored preconditioners provide a scalable alternative. Methods such as K-FAC~\citep{grosse2016kfac}, Shampoo~\citep{gupta2018shampoo}, and SOAP~\citep{vyas2024soap} approximate curvature or gradient covariance through tensor products of smaller matrices, enabling structured preconditioning without maintaining a dense Hessian approximation. Rather than enforcing a quasi-Newton secant condition, these methods use second-order statistics such as empirical Fisher information or uncentered gradient covariance to rescale the gradient.
However, existing Kronecker-factored methods still face limitations in non-stationary regimes. Standard Shampoo and SOAP typically update the preconditioner eigenspace at fixed intervals, which may not match the rate at which curvature statistics evolve during PINN training. Moreover, accumulated second-moment statistics are often reprojected after basis updates, which can be unstable when the basis rotates substantially~\citep{eschenhagen2025purifying}.

In this work, we propose \method, a layerwise Kronecker-factored optimizer for
high-accuracy PINN training. The method adds two operations to SOAP-style
preconditioning. First, a scalar factor matches the Kronecker metric to the
secant energy observed along the latest parameter displacement. Second, an
EShampoo-style off-diagonal criterion triggers a basis update
\citep{eschenhagen2025purifying}. At each triggered update, \method reprojects
the momentum state and downscales the coordinatewise second moment instead of
reprojecting that second moment.

For a $d_{\mathrm{out}}\times d_{\mathrm{in}}$ matrix parameter, \method stores
$O(d_{\mathrm{in}}^2+d_{\mathrm{out}}^2)$ Kronecker factors in addition to
Adam-style states. The scalar secant calculation costs $O(d_{\mathrm{in}}
d_{\mathrm{out}})$ after the required eigenspace quantities are available, or
$O(d_{\mathrm{in}}d_{\mathrm{out}})$ when counted over all matrix entries. The
basis check and eigendecomposition retain cubic layer-width costs, analyzed in
Appendix~\ref{app:compute_efficiency}.

Across eight PDE benchmarks, \method records the lowest final residual on six.
On Boussinesq, it reaches $10^{-5}$ in 4.1 hours, while Adam does not reach the
same target within a 14-hour run. The gains are regime-dependent. Purifying
Shampoo performs better on Gray-Scott, and SOAP performs better on
Ginzburg-Landau. We therefore study \method as a specialized optimizer for
stiff, high-accuracy regimes.

\textbf{Our contributions are:}
\begin{itemize}[leftmargin=1.2em,itemsep=0em,topsep=0em]
  \item We derive a directional secant-energy correction for a
  Kronecker-factored metric.
  \item We pair an adaptive basis trigger with momentum reprojection and a
  variance-state transition designed for abrupt basis changes.
  \item We state the scope of two theoretical results: a one-directional
  secant-matching result and a steady-state variance-mismatch bound.
  \item We compare against first-order, structured, PINN-specific, and
  self-scaled quasi-Newton baselines, with wall-clock, memory, multi-seed, and
  physical-error measurements.
  \item We identify regimes where self-scaling helps and regimes where it hurts.
\end{itemize}

\vspace{-0.2cm}
\section{Background and Related Work}
\vspace{-0.1cm}
% We briefly review PINNs, matrix-valued preconditioning, and self-scaling quasi-Newton methods.

\subsection{Physics-Informed Neural Networks (PINNs)}
\vspace{-0.2cm}
% Deep learning has emerged as a powerful paradigm for surrogate modeling and solving PDEs in mathematical physics. PINNs represent the state-of-the-art in this domain, leveraging the universal approximation capabilities of neural networks to parameterize solution manifolds.
% Formally, given a PDE of the form $\mathcal{F}[u](x, t) = 0$ over a domain $\Omega$, PINNs approximate the solution $u(x, t)$ via a neural network $u_\theta$. The key mechanism involves automatic differentiation, which allows for the exact evaluation of differential operators without discretization errors. The network is trained by minimizing a composite loss function $\mathcal{L}(\theta) = \lambda_{r} \mathcal{L}_{r}(\theta) + \lambda_{b} \mathcal{L}_{b}(\theta) + \lambda_{d} \mathcal{L}_{d}(\theta)$,
% where $\mathcal{L}_{r}$ penalizes the PDE residual at a set of collocation points, and $\mathcal{L}_{b}$ and $\mathcal{L}_{d}$ enforce boundary conditions and assimilate observational data, respectively. REPLACED WITH BELOW:

Physics-informed neural networks (PINNs)~\citep{raissi2019physics,karniadakis2021physics} solve PDE-constrained learning problems by representing the solution field with a neural network and enforcing the governing equations through differentiable residual losses. Formally, given a PDE of the form $\mathcal{F}[u](x,t)=0$ over a domain $\Omega$, a PINN approximates the solution $u(x,t)$ by a neural network $u_\theta(x,t)$. Automatic differentiation evaluates $\mathcal{F}[u_\theta]$, and the network is trained by minimizing a composite objective
\begin{align*}
\mathcal{L}(\theta)
=
\lambda_r \mathcal{L}_r(\theta)
+
\lambda_b \mathcal{L}_b(\theta)
+
\lambda_d \mathcal{L}_d(\theta),
\end{align*}
where $\mathcal{L}_r$ penalizes the PDE residual at collocation points, $\mathcal{L}_b$ enforces boundary or initial conditions, and $\mathcal{L}_d$ incorporates observational data when available.

\subsection{Optimization for Physics-Informed Learning}

PINN objectives combine residual, boundary, initial-condition, and data terms.
Their gradients may differ sharply in scale and direction
\citep{wang2021understanding,wang2022and,krishnapriyan2021characterizing}.
Loss-balancing methods address this issue through learning-rate annealing,
NTK-based reweighting, or parameterwise normalization
\citep{wang2021understanding,wang2022and,yao2023multiadam}. Other work targets
the geometry of the full objective. Examples include NysNewton-CG
\citep{rathore2024challenges}, PDE preconditioning
\citep{liu2024preconditioning}, energy and Gauss-Newton natural gradients
\citep{muller2023achieving,jnini2024gauss}, ANaGRAM
\citep{schwencke2025anagram}, dual natural gradients
\citep{jnini2025dual}, and K-FAC for PINNs
\citep{dangel2024kronecker}. Recent work also studies gradient alignment under
SOAP~\citep{wang2025gradient}, self-scaled BFGS and Broyden methods
\citep{kiyani2025optimizing,jnini2026curvature}.

\method targets a different point in this design space. It does not reweight
the PINN loss components or solve a global parameter-space or residual-space
system. It combines layerwise Kronecker statistics with a directional secant
correction. Dense self-scaled methods remain stronger references when their
state and line-search costs are tractable.

% By casting PDE solving as a non-convex optimization problem, PINNs offer distinct advantages over traditional mesh-based methods (e.g., FEM, FDM). They are naturally mesh-free, enabling effective scaling to high-dimensional problems and complex geometries. Furthermore, they are particularly well-suited for {inverse problems}, where unknown physical parameters or constitutive laws can be jointly inferred alongside the solution field.

% Recently, this framework has been extended to the rigorous study of finite-time singularities in fluid dynamics. \citep{wang2025high} uses PINNs to reconstruct self-similar blow-up profiles for the 1D Burgers equation and the 2D Boussinesq equations. In this regime, the ability of PINNs to handle unbounded domains and adaptively resolve sharp gradients is critical. However, accurately capturing these singular structures requires suppressing the PDE residual to near machine precision, which exposes the limitations of first-order optimizers on the ill-conditioned landscapes characteristic of multi-scale physical dynamics \citep{wang2021understanding}.

\subsection{Structured Matrix Preconditioning}

Shampoo approximates full-matrix AdaGrad with Kronecker factors
\citep{gupta2018shampoo}. For a matrix parameter $W_t\in\mathbb{R}^{m\times n}$
and gradient $G_t$, it maintains
\begin{equation}
L_t=\beta_2L_{t-1}+(1-\beta_2)G_tG_t^\top,\qquad
R_t=\beta_2R_{t-1}+(1-\beta_2)G_t^\top G_t.
\label{eq:shampoo-update}
\end{equation}
SOAP applies Adam in the eigenbasis of these factors
\citep{vyas2024soap}. Muon instead orthogonalizes matrix updates and is related
to a zero-decay Shampoo limit~\citep{jordan2024muon,
bernstein2024oldoptimizernewnorm}.

SOAP commonly refreshes its basis on a fixed schedule. Purifying Shampoo
separates eigenvalue and eigenbasis errors and develops an adaptive stopping
criterion for warm-started basis updates~\citep{eschenhagen2025purifying}.
\method adopts this off-diagonal basis diagnostic. Its new elements are the
Kronecker-adapted secant correction and the second-moment transition used after
a triggered basis change.

Muon \citep{jordan2024muon} is a recent optimization algorithm designed to approximate second-order scaling for large-scale neural network training at a cost comparable to first-order methods. The method uses Newton-Schulz iterations to orthogonalize the gradient as $G'=UV^\top$, where $G=U\Sigma V^\top$ is the SVD of the matrix gradient of a linear layer. Additionally, it can be seen as a special case of Shampoo with $\beta _2=0$ as $(GG^\top)^{-1/4}G(G^\top G)^{-1/4}=UV^T$ \cite{bernstein2024oldoptimizernewnorm}. 

\iffalse
The core mechanism of Muon relies on iterative Newton-Schulz (NS) refinement to project the momentum buffer onto the nearest semi-orthogonal matrix.
This operation effectively flattens the spectrum of the update, allowing the optimizer to traverse flat plateaus and saturation regions (which are common in language models) more efficiently than Adam.
However, because Muon imposes a hard spectral constraint rather than adapting to the specific local curvature magnitude of the loss function, it may struggle in regimes where preserving the relative scale of eigenvalues is critical for stability, such as the stiff, multi-scale landscapes of high-precision PDE solvers.
\fi

\subsection{Limitations of SOAP, Purifying Shampoo, and Adaptive Eigenvalue Correction}
Standard implementations of SOAP update the eigenbasis $(Q_L, Q_R)$ at fixed intervals. This static schedule fails to distinguish between different stages of training or layers with varying dynamics. Furthermore, when the basis is updated, standard approaches reproject the accumulated second-moment estimator $V_k$. We argue that reprojecting $V_k$ via similarity transformation $V_{k}^{new} = (Q_L^{new})^T Q_L V_k Q_R^T Q_R^{new}$ assumes variance preservation under rotation, which is invalid when the principal directions of the Hessian shift significantly, leading to unstable step sizes.

Recent analysis of the Shampoo algorithm has identified that its practical success relies heavily on two heuristics: \textit{stale preconditioning} (updating the eigenbasis at fixed, infrequent intervals) and \textit{learning rate grafting} (forcing the update magnitude to match Adam's to mitigate spectral errors). \citep{eschenhagen2025purifying} formalized these issues by decomposing the preconditioner update into its eigenvector (basis) and eigenvalue (curvature) components, termed the ``Purifying'' framework. 
% The core insight of Purifying Shampoo is that the approximation error induced by stale eigenbases can be explicitly monitored rather than managed via fixed schedules. The authors introduce an adaptive eigenbasis trigger that monitors the \textit{off-diagonal mass} of the preconditioner in the current basis. Specifically, for a preconditioner $L_t$ and current basis $Q_t$, the algorithm tracks the relative Frobenius norm of the off-diagonal elements as $\rho(L_t, Q_t) = {\| Q_t^\top L_t Q_t - \operatorname{diag}(Q_t^\top L_t Q_t)\|_F}/{\|L_t\|_F}$.
% An eigenbasis update is triggered only when $\rho(L_t, Q_t)$ exceeds a specific threshold, indicating that the current basis no longer effectively diagonalizes the curvature. This decouples the compute cost from the iteration count, focusing computational resources on phases of training where the loss landscape rotates rapidly.

% They further demonstrate that learning rate grafting is essentially a crude proxy for correcting staleness in the eigenvalues. By updating the eigenvalues at every step (which is computationally cheap in the Kronecker-factored setting) while keeping the basis ``lazy,'' the optimizer can maintain accurate curvature scaling without the need for extrinsic grafting. This combination of adaptive basis updates and continuous eigenvalue correction referred to as {Purifying Shampoo} provably reduces the dependency on heuristics while recovering the performance of full-frequency updates.    % NOTE: REMOVED FOR NOW

\subsection{Self-Scaling in Quasi-Newton Iteration Algorithms}

A general class of quasi-Newton iteration algorithms can be cast under the self-scaled Broyden formula \citep{albaali1992efficient}. If we define the auxiliary variables
$$
\begin{aligned}
{s}_k & ={\Theta}_{k+1}-{\Theta}_k, \quad {y}_k =\nabla J({\Theta}_{k+1})-\nabla J({\Theta}_k), \quad {v}_k =\sqrt{{y}_k \cdot H_k {y}_k}\left[\frac{{s}_k}{{y}_k \cdot {s}_k}-\frac{H_k {y}_k}{{y}_k \cdot H_k {y}_k}\right],
\end{aligned}
$$
then, the next approximation of the inverse Hessian matrix at each iteration can be calculated as in \citep{albaali1993variational,albaali2005wide} by
$$
H_{k+1}=\frac{1}{\tau_k}\left[H_k-\frac{H_k {y}_k \otimes H_k {y}_k}{{y}_k \cdot H_k {y}_k}+\phi_k {v}_k \otimes {v}_k\right]+\frac{{s}_k \otimes {s}_k}{{y}_k \cdot {s}_k},
$$
where we define
\begin{align}
\label{eq:tau_k_1}
\tau_k=\min \left\{1, \frac{{y}_k \cdot {s}_k}{\alpha_k {s}_k \cdot H_k^{-1} s_k}\right\}.
\end{align}
Here, $\otimes$ denotes the tensor product of two vectors and $\tau_k, \phi_k$ are respectively the scaling and the updating parameters, which often change between iterations. For $\tau_k=1$ and $\phi_k=1$, one recovers the standard BFGS algorithm.
To efficiently compute the scaling parameter $\tau_k$, \citep{urban2025unveiling} set $H_k^{-1} {s}_k=-\alpha_k \nabla J({\Theta}_k)$ so the explicit dependence of $\tau_k$ on $H_k^{-1}$ disappears. This is feasible since the step length $\alpha_k$ and the gradient $\nabla J\left(\Theta_k\right)$ are available at iteration $k$. Thus, the calculation of $\tau_k$ only involves vector multiplications in $O(n)$.

The original scaling scheme of \citep{oren1974selfscaling} is $\tau_k^{(1)}=\frac{-{y}_k \cdot {s}_k}{\alpha_k {s}_k \cdot \nabla J(\Theta_k)}$, which reduces the condition number of the transformed Hessian in the new space. However, this choice was shown to be inferior in terms of performance compared to the standard BFGS, when combined with an inexact line search computation of $\alpha_k$, as confirmed by \citep{nocedal1993selfscaling}. The choice of taking the minimum value of both, as adapted by \citep{urban2025unveiling}, is essentially a switch between standard BFGS ($\tau_k=1$) and self-scaled BFGS which ensures theoretically super-linear convergence with inexact line searches.

\vspace{-0.2cm}
\section{Methodology: \method}
\vspace{-0.2cm}
We propose \method, which integrates a self-scaling factor $\tau_k$ and a purifying-style adaptive update strategy. Algorithm \ref{alg:ss_soap} presents a simplified overview, while Algorithm \ref{alg:ss_purifying_soap} in the appendix provides the full pseudocode.

{\small
\begin{algorithm}[tb]
   \caption{\method}
   \label{alg:ss_soap}
   \begin{algorithmic}
   \STATE \textbf{Input:} $W_0$, $\eta$, $\beta_1$, $\beta_2$,
$\tau_{\mathrm{trigger}}$, $L_0,R_0\leftarrow I$,
$Q_L,Q_R\leftarrow I$, $\widetilde M_0,V_0\leftarrow0$
   \FOR{$k=1, 2, \dots$}
   \STATE Compute gradient $G_k$ and $L_k, R_k$ via \eqref{eq:shampoo-update}
   \IF{$k \pmod{\text{interval}} = 0$}
       \STATE Compute off-diagonal ratio via \eqref{eq:ratio}
       \IF{$\max\{\rho(L_k,Q_L),\rho(R_k,Q_R)\}>
\tau_{\mathrm{trigger}}$}
  \STATE $Q_L^{\mathrm{old}}\leftarrow Q_L$,
  $Q_R^{\mathrm{old}}\leftarrow Q_R$
  \STATE Recompute $Q_L,Q_R$ from $L_k,R_k$
  \STATE $\widetilde M_{k-1}\leftarrow
  (Q_L^\top Q_L^{\mathrm{old}})\widetilde M_{k-1}
  ((Q_R^{\mathrm{old}})^\top Q_R)$
  \STATE $V_{k-1}\leftarrow\gamma V_{k-1}$
\ENDIF
   \ENDIF
   \STATE Project $G'_k = Q_L^T G_k Q_R$ and update $M'_k, V_k$ (Adam updates in eigenspace)
   \STATE Compute $\tau_k$ via \eqref{eq:tau_k} and precondition $N'_k \leftarrow \tau_k^{-1/2} M'_k / (\sqrt{V_k} + \epsilon)$
   \STATE Update $W_k \leftarrow W_{k-1} - \eta Q_L N'_k Q_R^T$
   \ENDFOR
   \end{algorithmic}
\end{algorithm}
}

\subsection{Self-Scaling Curvature Correction}
Inspired by Self-Scaling BFGS (\ssbfgs) \citep{urban2025unveiling}, we introduce a scalar $\tau_k$ to the inverse Hessian approximation to ensure the preconditioner spans the correct spectrum of the true Hessian. In the Kronecker structure, we approximate the inverse Hessian as $(\hat{H}_k^{SS})^{-1} = \tau_k^{-1} (L_k^{-1} \otimes R_k^{-1})$. Leveraging the secant condition properties, we define the scaling factor $\tau_k$ using the trace of the updates. Let $s_k = \text{vec}(W_k - W_{k-1})$ and $y_k = \text{vec}(G_k - G_{k-1})$ then for the case of a Kronecker approximation of the Hessian in \eqref{eq:tau_k_1} we get:
\begin{equation}
\tau_k = \min \left\{ 1, \frac{\operatorname{Tr}((G_k - G_{k-1})^T (W_k - W_{k-1}))}{\operatorname{Tr}((W_k - W_{k-1})^T L_k^{-1} (W_k - W_{k-1}) R_k^{-1})} \right\}.
\label{eq:tau_k}
\end{equation}
We observe that $\tau_k = 1$ is chosen for most training iterations, with self-scaling activated only infrequently. Nevertheless, these occasional adjustments are sufficient to yield substantial improvements in optimization performance.
This formulation ensures that the scale of the preconditioner aligns with the curvature along the most recent displacement. This is applied in the eigenspace update as $N'_k = \tau_k^{-1/2} {M'_k}/{\sqrt{V_k} + \epsilon}$,
where $M'_k$ is the momentum in the eigenspace.

\subsection{Adaptive Basis Updates and Variance-State Transition}
To address the stability issues of fixed intervals and reprojection, we implement an eigenvalue correction to adaptively control the approximation error of the eigenspace.

\textbf{Adaptive eigenbasis trigger.}
Following the stopping criterion studied by
\cite{eschenhagen2025purifying}, we monitor the relative off-diagonal mass of
each Kronecker factor in its current basis,
\begin{equation}
\rho(L_k,Q_L)=
\frac{\|Q_L^\top L_kQ_L-\operatorname{diag}(Q_L^\top L_kQ_L)\|_F}
{\|Q_L^\top L_kQ_L\|_F}.
\label{eq:ratio}
\end{equation}
The statistic measures how poorly the current basis diagonalizes the tracked
Kronecker factor. It is a heuristic proxy for preconditioning degradation, not
an exact Hessian-staleness measure. We recompute a basis when the maximum ratio
for the left and right factors exceeds $\tau_{\mathrm{trigger}}$. Appendix
Table~\ref{tab:trigger_diagnostic} compares this statistic with consecutive
gradient cosine similarity around a triggered update.

\textbf{State transition after a basis update.}
When the basis changes, we first reproject the momentum so that it represents
the same update direction in the new coordinates. We do not reproject the
coordinatewise second moment. Instead, we retain its coordinatewise shape and
apply $V_{k-1}\leftarrow\gamma V_{k-1}$, with
$\gamma\in\{0.25,0.5,0.75\}$ selected from the observed off-diagonal ratio.
Because $V$ appears in the denominator of the Adam-style update, $\gamma<1$
raises the immediate effective step magnitude before new gradients refresh the
state. We treat this transition as an empirical transient heuristic. Theorem
\ref{thm:stability_downscaling} gives a steady-state mismatch bound and does not
establish the optimality of $\gamma<1$. Appendix
Figure~\ref{fig:gamma_transient} reports the targeted transient ablation.

% \subsection{Purifying Update Strategy}
% To address the stability issues of fixed intervals and reprojection, we implement the ``purifying'' strategy:

% \textbf{Adaptive Eigenbasis Trigger}
% We monitor the approximation error of the stale eigenbasis. Specifically, we compute the relative Frobenius norm of the off-diagonal elements of the preconditioner in the current basis:
% \begin{equation}
% \text{ratio} = \frac{\| Q_L^T L_k Q_L - \operatorname{diag}(Q_L^T L_k Q_L) \|_F}{\| Q_L^T L_k Q_L \|_F}.
% \label{eq:ratio}
% \end{equation}
% If this ratio exceeds a threshold $\tau_{\text{trigger}}$ (e.g., 0.7), it indicates significant basis misalignment, which triggers a recomputation of $Q_L$.

% \textbf{Variance Downscaling vs. Reprojection} Instead of reprojecting the second moment $V_k$ into the new basis, which disrupts the variance history, we apply a downscaling factor $\gamma \in (0, 1)$. This preserves the historical structure while reducing confidence in the estimate, acting as a ``soft reset''. The downscaling factor is determined adaptively based on the off-diagonal ratio
% \begin{equation}
% \gamma = 
% \begin{cases} 
% 0.25 & \text{if ratio} > 0.8 \\
% 0.5 & \text{if ratio} > 0.5 \\
% 0.75 & \text{otherwise.}
% \end{cases}
% \end{equation}
% This approach effectively decouples the frequency of updates across different layers and training stages.

\vspace{-0.2cm}
\section{Theoretical Analysis}
\vspace{-0.2cm}

We analyze self-scaling and variance downscaling theoretically, with proofs in Appendix~\ref{app:proofs}.
\vspace{-0.2cm}
% We provide a rigorous theoretical analysis of these mechanisms.The proofs of both theorems are in Appendix \ref{app:proofs}.
% We provide a theoretical analysis of these mechanisms, with proofs in \ref{app:proofs}.
% We show the optimality of the self-scaling factor under Kronecker constraints and that it minimizes the spectral disparity between the secant approximation and the true curvature. Furthermore, we derive stability bounds for the variance update, showing that our downscaling strategy yields bounded error under basis rotation, whereas standard reprojection schemes can lead to unbounded instability when the principal directions of the hessian shift.

\subsection{Spectral Stability Analysis of Variance Adaptation}
% \yw{please make sure Thms 4 and 5 are correct}
\begin{theorem}[Variance-state stability under basis changes]
\label{thm:stability_downscaling}
Let $C_t=\mathbb{E}[g_t g_t^{\top}]$ denote the population gradient second-moment matrix at step $t$. For an orthogonal basis $Q \in O(d)$, define the diagonal second-moment operator $\mathcal{D}_Q(C):=\operatorname{diag}\left(Q^{\top} C Q\right)$. Let $V_t=\mathcal{D}_{Q_t}\left(C_t\right)$ and $V_{t+1}^*=\mathcal{D}_{Q_{t+1}}\left(C_{t+1}\right)$,
where $V_t$ is the diagonal second-moment state in the old basis $Q_t$, and $V_{t+1}^*$ is the oracle diagonal second-moment state in the new basis $Q_{t+1}$. Define the basis transition matrix $U=Q_{t+1}^{\top} Q_t$. Consider the scalar soft-reset update $V_{\text {scale }}=\gamma V_t, 0<\gamma \leq 1$. Then
$$
\left\|V_{\text {scale }}-V_{t+1}^*\right\|_F \leq(1-\gamma)\left\|V_t\right\|_F+2\left\|C_t\right\|_2\|I-U\|_F+\left\|C_{t+1}-C_t\right\|_F .
$$
In particular, if the gradient second moment evolves slowly, so that $\left\|C_{t+1}-C_t\right\|_F \leq \Delta_t$, then
$$
\left\|V_{\text {scale }}-V_{t+1}^*\right\|_F \leq(1-\gamma)\left\|V_t\right\|_F+2\left\|C_t\right\|_2\|I-U\|_F+\Delta_t .
$$
\end{theorem}

\subsection{Self-Scaling Spectral Analysis}

\begin{theorem}[Kronecker secant-energy matching]
\label{thm:self_scaling_convergence}
Let \(S_k=W_k-W_{k-1}\), \(Y_k=\nabla \mathcal{L}(W_k)-\nabla \mathcal{L}(W_{k-1})\), and let
\(K_k\succ 0\) denote the Kronecker metric whose quadratic form is $s_k^\top K_k s_k = \operatorname{Tr}\!\left(S_k^\top L_k^{-1}S_kR_k^{-1}\right)$, where \(s_k=\operatorname{vec}(S_k)\). Assume \(s_k\neq 0\) and
\(s_k^\top y_k>0\). Define $a_k=s_k^\top K_k s_k$ and $c_k=s_k^\top y_k$. Then the scaling factor
\begin{align*}
\tau_k
=
\min\left\{
1,
\frac{c_k}{a_k}
\right\}
=
\min\left\{
1,
\frac{s_k^\top y_k}{s_k^\top K_k s_k}
\right\}
\end{align*}
is the unique solution of the clipped one-dimensional log-secant matching
problem
\begin{align*}
\tau_k
=
\arg\min_{0<\tau\le 1}
\left[
\log
\frac{s_k^\top y_k}
{s_k^\top(\tau K_k)s_k}
\right]^2 .
\end{align*}
% Equivalently, in matrix form,
% \begin{align*}
% \tau_k
% =
% \arg\min_{0<\tau\le 1}
% \left[
% \log
% \frac{\operatorname{Tr}(Y_k^\top S_k)}
% {\tau\,\operatorname{Tr}(S_k^\top L_k^{-1}S_kR_k^{-1})}
% \right]^2 .
% \end{align*}
If \(y_k=H_ks_k+r_k\), then, whenever the unclipped solution is active,
the scaled Kronecker quadratic form satisfies $s_k^\top(\tau_kK_k)s_k = s_k^\top H_ks_k+s_k^\top r_k$. Thus the mismatch between the scaled Kronecker secant energy and the true
directional curvature is exactly the directional secant error \(s_k^\top r_k\).
\end{theorem}

\vspace{-0.3cm}
\section{Numerical Experiments}
\vspace{-0.2cm}
We evaluate \method on a suite of challenging optimization tasks, including matrix regression, low-rank completion tasks and nonlinear evolutionary PDEs. The matrix tasks provide controlled settings where the curvature structure is known, allowing us to isolate the effect of self-scaling on the optimizer dynamics. In contrast, the nonlinear PDE benchmarks evaluate its robustness against the curvature anisotropy and stiffness typically encountered in frontier AI for Science applications~\citep{wang2021understanding,krishnapriyan2021characterizing,rathore2024challenges}.
% \seb{Could use a bit of motivation why there settings.}
\vspace{-0.3cm}

\subsection{Baselines}

The main comparisons include Adam~\citep{kingma2015adam}, Muon
\citep{jordan2024muon}, SOAP~\citep{vyas2024soap}, and the adaptive
Shampoo variant implemented from \cite{eschenhagen2025purifying}. We also
evaluate L-BFGS, SS-BFGS, SS-Broyden, K-FAC for PINNs, NTK reweighting, and
MultiAdam where their memory and training requirements fit the benchmark
\citep{liu1991limited,kiyani2025optimizing,dangel2024kronecker,
wang2022and,yao2023multiadam}. Appendix~\ref{app:hparam_setup} reports each
search space, selected setting, precision, stopping rule, hardware, and seed.

% \seb{can we point to appendix with more details on setup and hparams we tested}

% We also conduct a comprehensive empirical validation of \method against Adam and other recent curvature-aware and preconditioned optimizers. We benchmark performance on matrix regression and low-rank completion tasks to analyze spectral properties, as well as on challenging stiff PDE benchmarks such as the Allen-Cahn and 2D Boussinesq equations. Across these settings, \method consistently matches or improves upon baselines, achieving up to two orders of magnitude reductions in final residual error on PDE benchmarks, while ablation studies confirm robustness across mini-batch sizes on regression tasks.

% \alex{need to show that the approximation error of the Kronecker factors evolves during training and impacts convergence, depending on both the training
% stage and parameter’s properties. Using an adaptive update frequency can improve Shampoo’s
% training efficiency, especially when more frequent eigenbasis computations accelerate convergence}

\vspace{-0.2cm}
\subsection{Matrix Regression Tasks}
\vspace{-0.2cm}
To evaluate the robustness and versatility of {\method} across varying optimization landscapes, we conduct a comprehensive suite of matrix optimization experiments. These tasks span the spectrum of convexity and gradient stochasticity, ranging from strongly convex problems and convex objectives to challenging non-convex settings. This diverse testbed allows us to rigorously benchmark {\method} against baselines in both deterministic and stochastic regimes. Detailed analysis is in Appendix \ref{sec:details_matrix_optim_tasks}.

\begin{figure*}[h]
  \centering    
  \begin{subfigure}[t]{\linewidth}
    \centering
    \includegraphics[width=\linewidth]{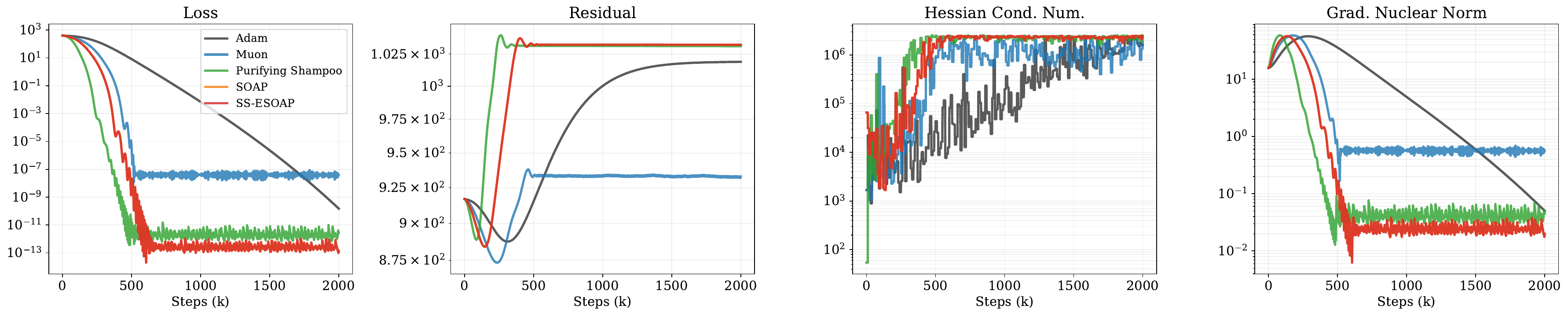}
\vspace{-2em}
% \caption{Matrix Quadratic Regression: Losses, accuracy, gradient condition numbers, nuclear norms.}
\label{fig:new_quadratic_comparison_singled_out}
\end{subfigure}
\caption{{Matrix Quadratic Regression Task.} \method consistently achieves faster convergence and lower residuals in the quadratic regime where the self-scaling mechanism captures the spectrum.
% \seb{why is SOAP legend here but not shown in plots?} % will correct later
}
\label{fig:matrix_tasks_singled_out}
\vspace{-0.2cm}
\end{figure*}

\textbf{Quadratic regression.} We consider a matrix quadratic regression objective $f(X):=\frac{1}{2}\|A X B-C\|_{\mathrm{F}}^2$, where $X \in \mathbb{R}^{m \times n}, A \in \mathbb{R}^{p \times m}, B \in \mathbb{R}^{n \times q}$ and $C \in \mathbb{R}^{p \times q}$. Then its gradient is $\nabla f(X)=A^{\top}(A X B-C) B^{\top}$, its Hessian is $\nabla^2 f(X)=(B B^{\top}) \otimes(A^{\top} A) \in \mathbb{R}^{m n \times m n}$. The Hessian structure implies that ideal preconditioning isolates the residual condition number $\kappa(E)$ from the constant spectral skew of $A$ and $B$.
% Unlike semi-orthogonal projections that discard curvature information from the residual $E$ to enforce unit conditioning ($\kappa=1$), {\method} preserves this structural information.
We set $(m, n, p, q) = (500, 100, 1000, 250)$ so that $f$ is strongly convex.
% As shown in Figure~\ref{fig:new_quadratic_comparison}, this allows {\method} to mimic a Newton-like trajectory, resolving the low-rank structure significantly faster than Adam and standard SOAP. While Adam's gradient condition number $\kappa(\nabla f(X_k))$ grows unstably, {\method} maintains a bounded condition number throughout training. The residual condition number $\kappa(E_k)$ correlates perfectly with the loss plateaus, validating our theoretical insight that conditioning on the residual spectrum is crucial for convergence in matrix sensing tasks. Muon's plateau suggests it fails to fully resolve this residual structure compared to the curvature-adaptive {\method}.

As illustrated in Figure \ref{fig:matrix_tasks_singled_out}, \method demonstrates a clear advantage in navigating the structured curvature of the quadratic objective.
% While Adam's gradient condition number grows unstably, \method maintains a bounded condition number throughout the training trajectory. This bounded conditioning directly translates to a faster and deeper descent in the residual loss, confirming that the self-scaling mechanism effectively captures and isolates the relevant spectral structure, allowing it to bypass the stagnation phases observed in both Adam and Muon.

The quadratic task isolates the setting where the Hessian has an exact
Kronecker form. It therefore tests whether the scalar correction recovers the
directional scale predicted by the analysis without the nonlinear and
multi-objective effects of a PINN. Logistic regression and low-rank completion
results appear in Appendix~\ref{sec:details_matrix_optim_tasks}.

% Figure~\ref{fig:matrix_tasks_singled_out} \seb{We need to comment on these results}
% From Figure \ref{fig:new_completion_comparison}, we observe that {\method} demonstrates superior convergence efficiency: Compared to Adam and standard SOAP, {\method} achieves the fastest descent rate and reaches the lowest final residual.
% Muon suffers from stagnation despite rapid initial progress: While Muon's whitening effect provides a strong initial acceleration outperforming Adam, its convergence plateaus earlier than curvature-adaptive methods. This suggests that orthogonalizing gradients without explicitly accounting for the magnitude of singular values limits its ability to resolve the fine-grained structure of the low-rank manifold in the terminal phase.

\subsection{PINN Benchmarks}

We apply \method to standard PINN benchmarks including the Allen-Cahn, 1D Burgers, and 2D Boussinesq equations, assessing performance across varying degrees of physical complexity.

Table~\ref{tab:complexity} shows that \method achieves the lowest final residual on six of eight PINN benchmarks. The largest gains appear on Wave, where \method reaches $4.12\times10^{-5}$ compared to $1.18\times10^{-3}$ for the best baseline, and on Boussinesq, where it reaches $8.42\times10^{-7}$ compared to $2.85\times10^{-6}$. \method also improves over the strongest baselines on Burgers ($5.54\times10^{-10}$ vs.\ $1.12\times10^{-9}$), Korteweg-de Vries ($3.52\times10^{-10}$ vs.\ $8.12\times10^{-10}$), and Lid-Driven Cavity ($8.23\times10^{-8}$ vs.\ $2.10\times10^{-7}$). The exceptions are Gray-Scott, where Purifying Shampoo obtains $5.12\times10^{-9}$ compared to $7.46\times10^{-7}$ for \method, and Ginzburg-Landau, where SOAP obtains $9.12\times10^{-12}$ compared to $1.75\times10^{-11}$ for \method. Overall, these results indicate that \method is most beneficial in regimes where additional scale calibration and damping improve the stability of structured preconditioning, while the gains are problem-dependent and not uniform across all PDEs.

\textbf{1D Burgers and 2D Boussinesq.} To validate our optimizer's capability in this frontier regime, we adopt the rigorous high-precision formulations for the 1D Burgers and 2D Boussinesq equations established by \citep{wang2025high}.
In these stiff settings, standard first-order methods typically stall due to spectral cliffs. However, \method reaches substantially lower residuals, driving the residual down to $5.54\times 10^{-10}$ precision for the Burgers equation, which is an improvement of about 3.9 orders of magnitude over Adam. Similarly, on the 2D Boussinesq equation, \method achieves a final residual of $8.42\times 10^{-7}$, successfully resolving the singular blow-up profiles where baselines like Muon and Adam stagnate at $2.57 \times 10^{-2}$ and $5.12 \times 10^{-4}$ respectively.

% \textbf{Allen-Cahn equation.}
% The Allen-Cahn equation, $u_t - 0.0001 u_{xx} + 5u^3 - 5u = 0$, exhibits sharp phase transitions driven by the stiff reaction term $5u^3$. \method achieves a final residual of ${1.85 \times 10^{-9}}$, which is two orders of magnitude lower than Adam. This performance highlights the efficacy of the self-scaling curvature correction in navigating the pathological curvature induced by the competing reaction and diffusion operators.

% \textbf{Navier-Stokes (Lid-Driven Cavity).}
% For the 2D Navier-Stokes equation at $Re=5000$, capturing the fine-grained vortical structures requires resolving dynamics across multiple scales. \method achieves a residual of ${8.23 \times 10^{-8}}$, significantly outperforming Adam and Muon. We attribute this robustness to the adaptive eigenbasis update strategy, which allows the optimizer to maintain stable preconditioners for different network layers with independent update frequencies, effectively balancing the optimization of the coupled pressure and velocity fields.

\textbf{Additional benchmarks.}
Additional results for Lid-Driven Cavity, Wave, Korteweg-de Vries, Gray-Scott, and Ginzburg-Landau are provided in Appendix Figure~\ref{fig:aggre_plots_all}. Detailed problem formulations and hyperparameter settings are given in Appendix~\ref{sec:benchmarks_description}.

\begin{table*}[h]
\vspace{-0.5em}
\caption{Final PDE residual after the fixed training budget. Lower is better.
Each entry reports the selected run under the tuning protocol in
Appendix~\ref{app:hparam_setup}. Bold marks the lowest value within each row.
Residual is an optimization metric. Physical solution errors and three-seed
statistics appear in Appendix~\ref{tab:physical_errors}.}
\centering
\footnotesize
\begin{tabular}{lccccc}
\toprule
\sc{Benchmark} & \sc{Adam} & \sc{Muon} & \sc{SOAP} & \sc{Purifying Shampoo} & \method (Ours) \\
\midrule
Wave & $2.15 \times 10^{-3}$ & $5.12 \times 10^{4}$ & $1.45 \times 10^{-3}$ & $1.18 \times 10^{-3}$ & $\mathbf{4.12 \times 10^{-5}}$ \\
Burgers & $4.21 \times 10^{-6}$ & $3.85 \times 10^{-4}$ & $2.15 \times 10^{-9}$ & $1.12 \times 10^{-9}$ & $\mathbf{5.54 \times 10^{-10}}$ \\
Allen-Cahn & $2.10 \times 10^{-7}$ & $1.25 \times 10^{-7}$ & $8.45 \times 10^{-7}$ & $3.15 \times 10^{-9}$ & $\mathbf{1.85 \times 10^{-9}}$ \\
Boussinesq & $5.12 \times 10^{-4}$ & $2.57 \times 10^{-2}$ & $9.15 \times 10^{-6}$ & $2.85 \times 10^{-6}$ & $\mathbf{8.42 \times 10^{-7}}$ \\
Korteweg-de Vries & $1.85 \times 10^{-8}$ & $1.93 \times 10^{-3}$ & $8.12 \times 10^{-10}$ & $1.75 \times 10^{-7}$ & $\mathbf{3.52 \times 10^{-10}}$ \\
Gray-Scott & $3.10 \times 10^{-6}$ & $4.20 \times 10^{-4}$ & $2.10 \times 10^{-7}$ & $\mathbf{5.12 \times 10^{-9}}$ & $7.46 \times 10^{-7}$ \\
Ginzburg-Landau & $1.15 \times 10^{-7}$ & $5.40 \times 10^{-8}$ & $\mathbf{9.12 \times 10^{-12}}$ & $2.45 \times 10^{-11}$ & ${1.75 \times 10^{-11}}$ \\
Lid-Driven Cavity & $1.88 \times 10^{-6}$ & $1.25 \times 10^{-3}$ & $3.81 \times 10^{-7}$ & $2.10 \times 10^{-7}$ & $\mathbf{8.23 \times 10^{-8}}$ \\
\bottomrule
\end{tabular}
\label{tab:complexity}
\end{table*}

PDE residual alone does not establish solution accuracy. We therefore compare
predictions with high-fidelity reference solutions on four representative PDEs.
Appendix Table~\ref{tab:physical_errors} reports relative $L^2$ and $H^1$
errors over three seeds for general-purpose, PINN-specific, and quasi-Newton
baselines. The error ranking agrees with the residual ranking on the stiff
Burgers and Boussinesq cases, while the full table records the remaining
method-dependent differences.

% \vspace{-1.5em}
\subsection{Component Analysis and Ablation Studies}
\vspace{-0.1cm}
% \seb{We write these paragraphs in past tense. Instead write as "table X suggests this".}

To isolate the impact of the specific mechanisms introduced in \method, namely variance downscaling, Kronecker-aware self-scaling, and the adaptive trigger, we conduct targeted ablation studies on the stiff Allen-Cahn and 2D Boussinesq benchmarks.

\textbf{Variance-state transition.}
Figure~\ref{fig:tau_and_eig_ablation} compares downscaling, reprojection, and a
reset on Allen-Cahn and Boussinesq. Reprojection produces large transient loss
spikes in these runs. A reset avoids the largest spikes but converges to a
higher final loss. The targeted $\gamma$ experiment in Appendix
Figure~\ref{fig:gamma_transient} separates behavior immediately before and
after a triggered basis update.
% \seb{In this part we are talking about a theorem in the appendix}

\textbf{Kronecker-adapted secant scaling.}
We compare the directional correction in Equation~\eqref{eq:tau_k} with an
unscaled update and a scalar correction that omits the Kronecker metric.
Figure~\ref{fig:tau_and_eig_ablation} reports the resulting trajectories. This
experiment tests directional scale matching. It does not compare a full dense
SS-BFGS update, which appears only in the tractable small-network experiment in
Appendix Table~\ref{tab:self_scaled_small}.

Figure~\ref{fig:tau_and_eig_ablation} summarizes the ablations for the proposed structural choices. Figures~\ref{fig:bous_stability_ablation}, \ref{fig:ac_stability_ablation} show that adaptive variance downscaling stabilizes basis updates. Figures~\ref{fig:bous_ablation_final}, \ref{fig:ac_ablation_final} show that the Kronecker-adapted scaling factor $\tau_k$ improves stability and convergence relative to unscaled and scalar SS-BFGS-style variants. Together, these results suggest that \method captures part of the step-size regulation benefit of self-scaled quasi-Newton methods without requiring dense Hessian cross-terms.
% {This ablation shows that SS-ESOAP successfully preserves the core step-size regulation benefits of self-scaling quasi-Newton methods. Even without the exact, dense Hessian cross-terms utilized by methods like SS-Broyden, our Kronecker-adapted $\tau_k$ formulation provides sufficient structure-aware curvature correction to dramatically stabilize and accelerate convergence over unscaled baselines.} The visual evidence for these structural choices is presented in Figure \ref{fig:tau_and_eig_ablation}. Figures \ref{fig:bous_stability_ablation} and \ref{fig:ac_stability_ablation} demonstrate the critical role of adaptive variance downscaling.
% ; whereas exact reprojection triggers catastrophic divergence spikes upon basis rotation, our soft reset mechanism maintains a smooth, stable loss trajectory. 
% Figures \ref{fig:bous_ablation_final} and \ref{fig:ac_ablation_final} validate the efficacy of our proposed Kronecker-adapted scaling factor ($\tau_k$).
% The blue trajectory, representing our structure-aware $\tau_k$, exhibits a significantly steeper and more stable descent compared to both standard Oren-Luenberger scaling and unscaled baselines, confirming that proper step-size regulation is indispensable for navigating anisotropic curvature.

\textbf{Efficiency of the adaptive trigger.} 
% Finally, we analyzed the trade-off between update frequency and computational cost by varying the off-diagonal threshold $\tau_{\text{trigger}}$, see Table \ref{tab:trigger_sensitivity}. Setting $\tau_{\text{trigger}} = 0$ (i.e., recomputing eigenbasis update every step) yielded marginal convergence gains but incurred prohibitive wall-clock costs. Conversely, ``lazy'' updates ($\tau_{\text{trigger}} \ge 0.8$) degraded solution accuracy significantly as the eigenbasis became misaligned with the evolving curvature. Our default threshold of $\tau_{\text{trigger}} = 0.2$ achieved a pareto-optimal balance and matched the convergence rate of full-frequency updates while reducing total training time compared to fixed-interval baselines.
Finally, we analyze the sensitivity of the eigenbasis update threshold $\tau_{\text{trigger}}$. As shown in Table \ref{tab:trigger_sensitivity} in the Appendix, extreme values degrade performance: $\tau_{\text{trigger}} = 0$ (updates every step) incurs prohibitive computational cost, while $\tau_{\text{trigger}} \ge 0.8$ (lazy updates) leads to stale curvature estimates and poor convergence. However, we observe a broad stable operating region $\tau_{\text{trigger}} \in [0.15, 0.3]$ where performance remains consistent. We use the fixed default values ($\tau_{\text{trigger}}=0.2$ and discrete $\gamma \in \{0.25, 0.5, 0.75\}$) across all experiments in this paper, ranging from simple regression tasks to the stiff Burgers and Boussinesq PDEs. \method achieves strong results across diverse problems without problem-specific tuning, suggesting that its adaptive mechanisms are robust across scales and stiffness regimes.

\begin{figure}[h]
  \centering
  \begin{subfigure}[t]{0.45\linewidth}
    \centering
    \includegraphics[width=1\linewidth, trim=0 0 0 1.1cm, clip]{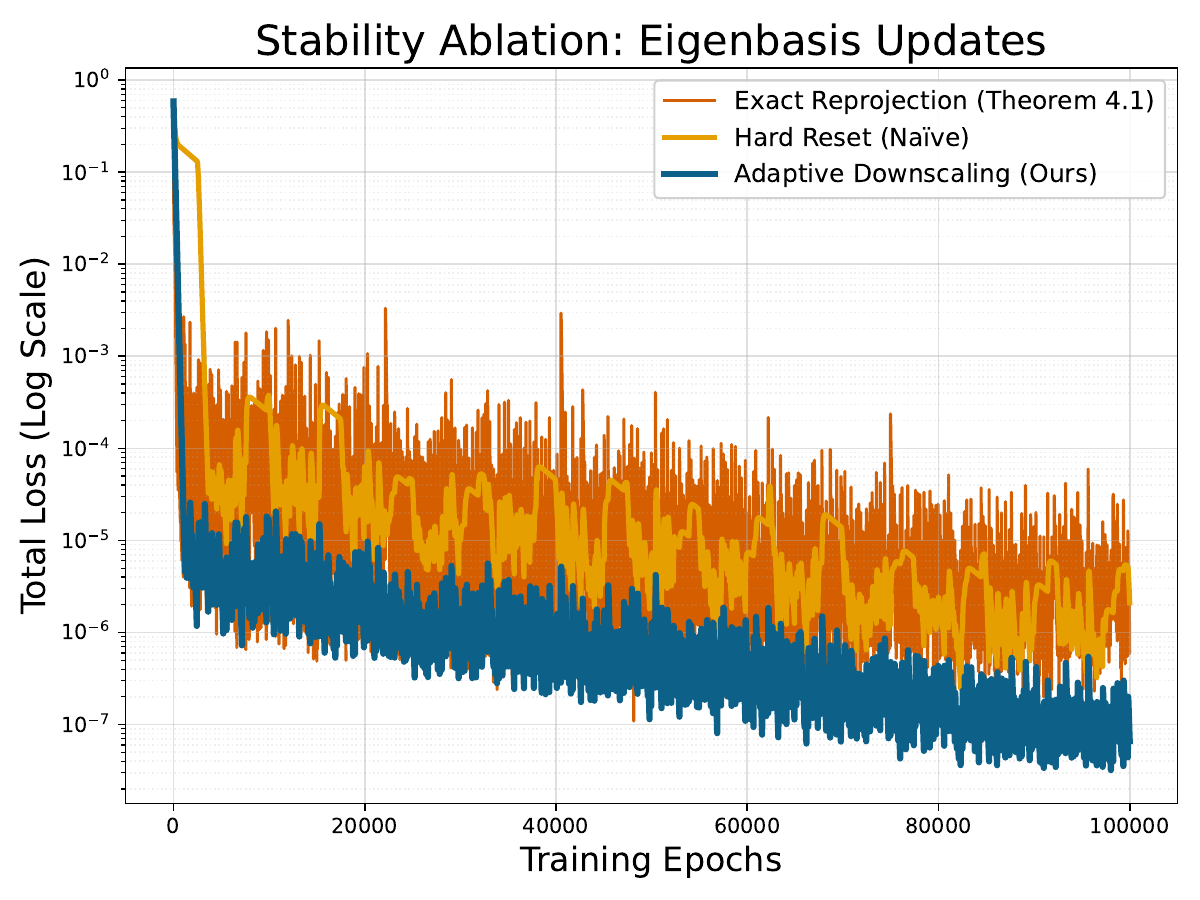}
    \caption{Stability ablation, adaptive eigenbasis, Boussinesq.}
    \label{fig:bous_stability_ablation}
  \end{subfigure}
  % \hspace{1.5pt}
  % \hfill
  \begin{subfigure}[t]{0.45\linewidth}
    \centering
    \includegraphics[width=1\linewidth, trim=0 0 0 1.1cm, clip]{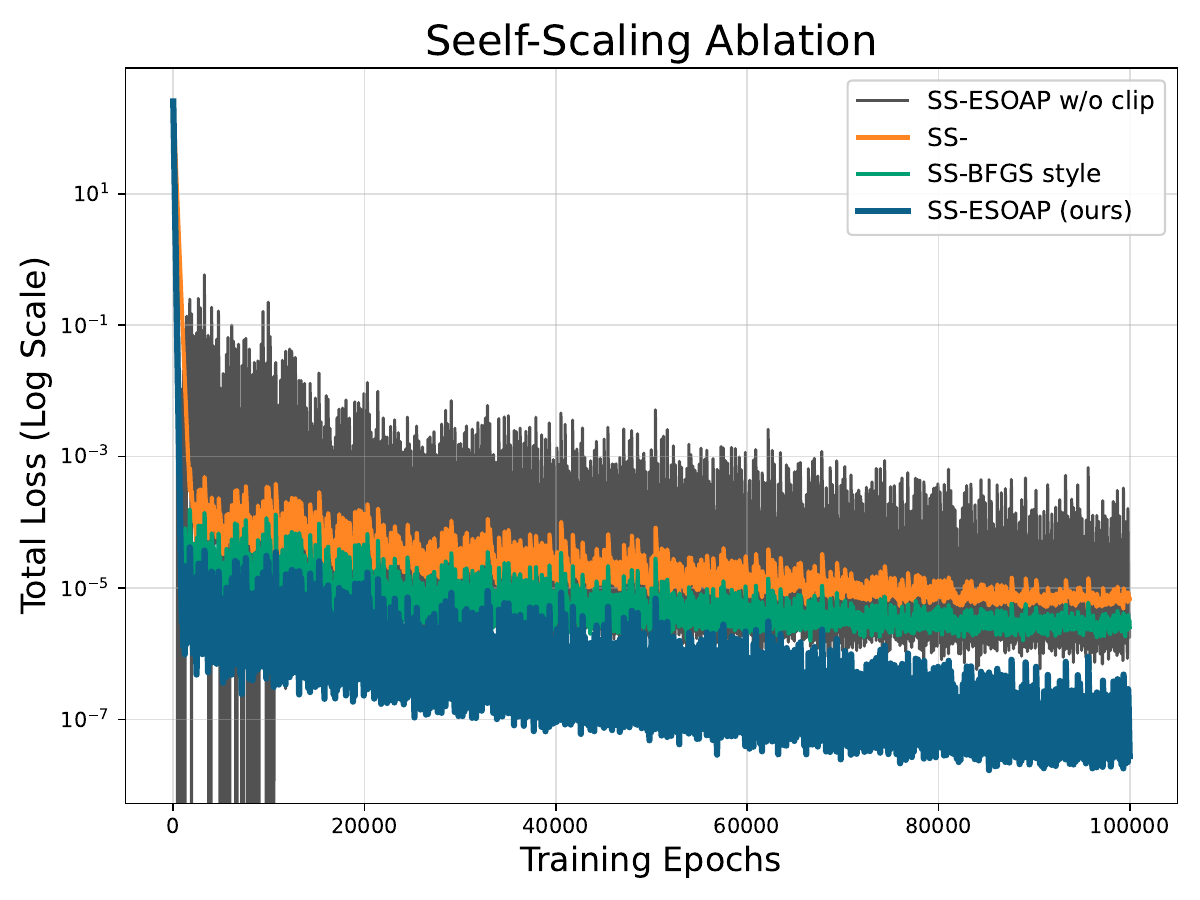}
    \caption{Self-scaling, Boussinesq.}
    \label{fig:bous_ablation_final}
  \end{subfigure}
  
  \begin{subfigure}[t]{0.45\linewidth}
    \centering
    \includegraphics[width=1\linewidth, trim=0 0 0 1.1cm, clip]{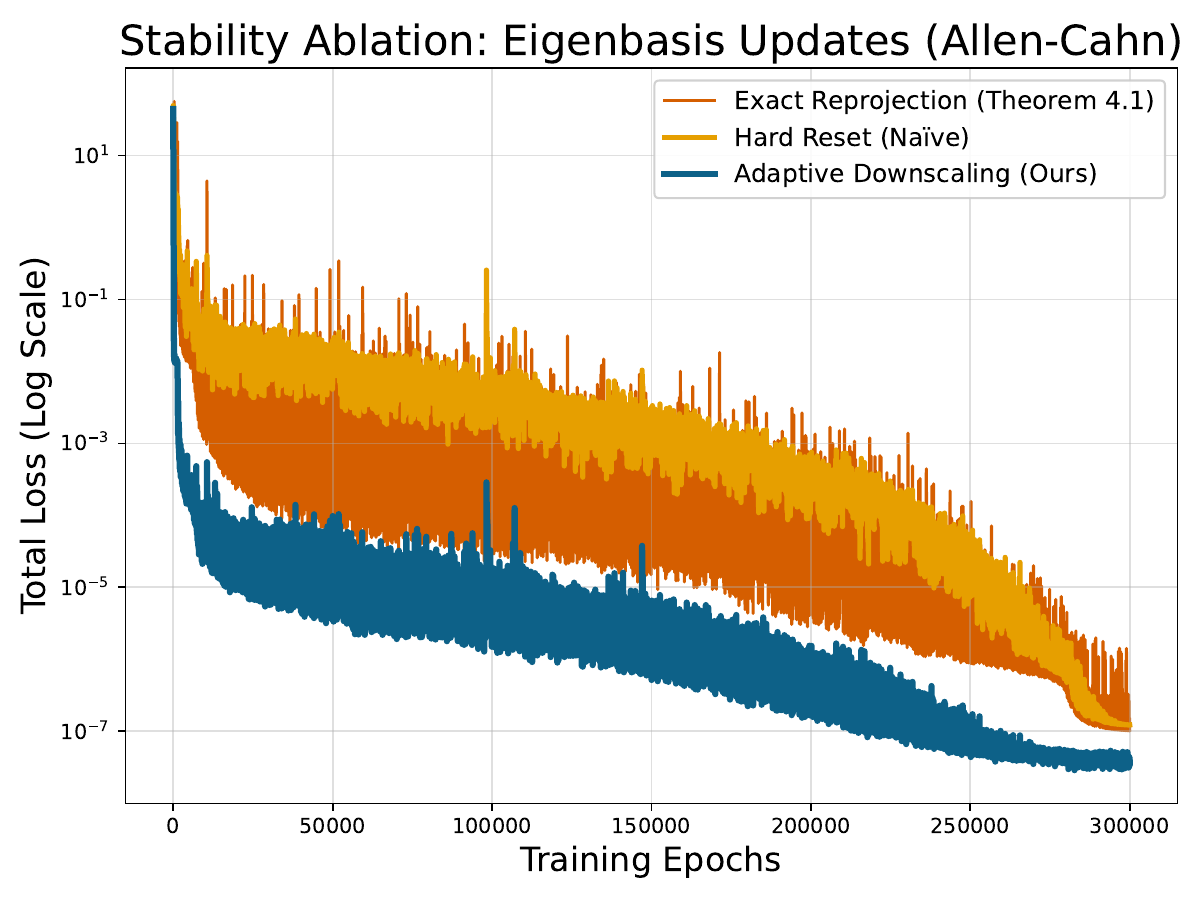}
    \caption{Stability ablation, adaptive eigenbasis, Allen-Cahn.}
    \label{fig:ac_stability_ablation}
  \end{subfigure}
  % \hspace{1.5pt}
  % \hfill
  \begin{subfigure}[t]{0.45\linewidth}
    \centering
    \includegraphics[width=1\linewidth, trim=0 0 0 1.1cm, clip]{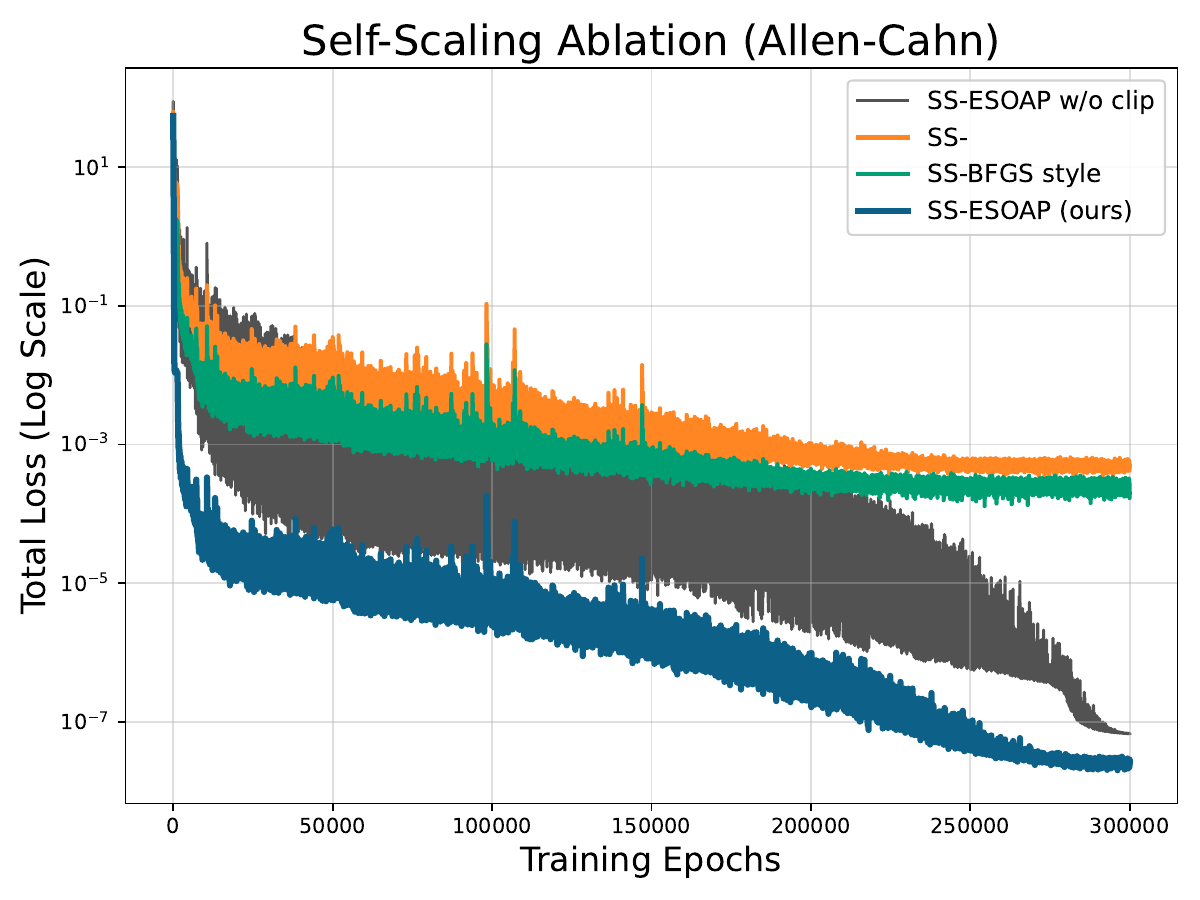}
    \caption{Self-scaling, Allen-Cahn.}
    \label{fig:ac_ablation_final}
  \end{subfigure}
\caption{\textbf{Ablation studies on 2D Boussinesq (top) and Allen-Cahn (bottom).} {(a) and (c):} Stability analysis of the eigenbasis update. {(b) and (d):} Efficacy of the self-scaling curvature correction. The proposed Kronecker-adapted factor $\tau_k$ (in blue) accelerates and stabilizes convergence compared to the unclipped counterpart, standard Oren-Luenberger scaling ($\tau_{OL}$, SS-BFGS style), and the unscaled baseline ($\tau=1$, SS-), validating structure-aware step-size regulation.}
  % Comparison between the reference solution and model predictions.
  \label{fig:tau_and_eig_ablation}
  \vspace{-1em}
\end{figure}

\vspace{-0.2cm}
\subsection{Computational Efficiency Summary}
\vspace{-0.2cm}
\label{sec:compute_summary}
While \method introduces a modest per-step computational overhead compared to first-order methods, its curvature-aware updates drastically reduce the total iterations required for convergence on stiff problems. As summarized in Table \ref{tab:compute_adam_comparison}, \method maintains a memory footprint highly competitive with Adam by leveraging Kronecker factorization. This effectively bypasses the prohibitive storage costs of full quasi-Newton methods. Ultimately, this structural efficiency translates to a substantial reduction in time-to-solution for high-precision PDE training. A comprehensive analysis detailing structural efficiency and memory tradeoffs is provided in Appendix \ref{app:compute_efficiency}.
\begin{table}[H]
\centering
\vspace{-1em}
\caption{Compute and memory comparison between Adam and \method on 2D Boussinesq, targeting residual $10^{-5}$ on a single consumer GPU.}

\label{tab:compute_adam_comparison}
\small
\begin{tabular}{lcc}
\toprule
\textsc{Metric} & \textsc{Adam} & \textsc{\method (Ours)} \\
\midrule
{Storage Complexity (Per Layer)} & $O(d_{\text{in}} d_{\text{out}})$ & $O(d_{\text{in}} d_{\text{out}} + d_{\text{in}}^2 + d_{\text{out}}^2)$ \\
{Peak VRAM Footprint} & 8.2 GB & 9.2 GB \\
{Time to Target Solution} & DNC ($>14.0$ hours) & \textbf{4.1 hours} \\
{Scalability (10M+ Parameters)} & Tractable & Tractable \\
\bottomrule
\end{tabular}
\vspace{-1em}
\end{table}

% \alex{trigger sensitivity ablation plots - two side by side subplots: loss vs steps and loss vs wall-clock time}
% \alex{
% $\tau = 0.0$ (Ideal/Shampoo): Converges efficiently in steps (steepest descent) but is very slow in time because it computes an expensive eigendecomposition at every step.

% $\tau = 0.2$ (Ours): Matches the step-wise efficiency of $\tau=0.0$ almost perfectly but faster
% % but runs ~3-4x faster in wall-clock time because it skips the expensive update 85-90\% of the time. This is the "Sweet Spot".

% $\tau = 0.8$ (Lazy): Fast per step (rare updates), but the geometry becomes stale, which should cause the loss convergence to degrade significantly.

% $\tau = \infty$ (Frozen): Very fast steps, but optimization stalls early due to completely wrong curvature information.
% }

\section{Limitations}

Our analysis does not prove global convergence or sufficient decrease for the
full nonconvex algorithm. The basis trigger uses the tracked Kronecker factors
as a heuristic proxy for preconditioning degradation. Its $O(d^3)$ check and
eigendecomposition costs limit use in wide layers. Dense SS-BFGS or SS-Broyden
may offer better directions on networks small enough to store their state, and
modern residual-space natural-gradient methods form a strong alternative. The
empirical gains are not uniform: SOAP-family methods perform better on
Gray-Scott and Ginzburg-Landau. All main PDE experiments use MLP or PirateNet
backbones. Transfer to transformer-based physics models such as PINNsFormer,
whose reported loss surface is smoother~\citep{zhao2024pinnsformer}, remains
unevaluated. The three-seed study covers four representative PDEs, so broader
seed-level conclusions require more runs.

\section{Conclusion}

We introduced \method, a Kronecker-factored optimizer that combines
directional secant-energy matching with adaptive basis updates and a
variance-state transition. The method attains the lowest final residual on six
of eight PDE benchmarks and improves wall-clock time to a fixed target on the
stiff Boussinesq case. Three-seed physical errors on four PDEs support the
solution quality of these low-residual fits. Results on Gray-Scott and
Ginzburg-Landau show that the added scaling is not uniformly beneficial. The
main open questions are how to reduce the cubic basis-check cost, how to select
the state transition from online diagnostics, and how the method transfers to
smoother physics architectures.

\bibliography{example_paper}
\bibliographystyle{unsrt}
% \bibliographystyle{icml2026}

%%%%%%%%%%%%%%%%%%%%%%%%%%%%%%%%%%%%%%%%%%%%%%%%%%%%%%%%%%%%

% \newpage

\section*{Appendix}

\appendix

\section{Proofs of Main Theorems}
\label{app:proofs}

\subsection{Proof of Theorem \ref{thm:stability_downscaling}}

\begin{proof}
By the triangle inequality,
$$
\left\|\gamma V_t-V_{t+1}^*\right\|_F \leq\left\|(\gamma-1) V_t\right\|_F+\left\|V_t-V_{t+1}^*\right\|_F.
$$
Since $0<\gamma \leq 1$,
$$
\left\|(\gamma-1) V_t\right\|_F=(1-\gamma)\left\|V_t\right\|_F .
$$
It remains to bound the second term, $\left\|V_t-V_{t+1}^*\right\|_F$. Define $M_t:=Q_t^{\top} C_t Q_t$.
Since $U=Q_{t+1}^{\top} Q_t$, we have $Q_{t+1}=Q_t U^{\top}$. Therefore,
$$
Q_{t+1}^{\top} C_t Q_{t+1}=U Q_t^{\top} C_t Q_t U^{\top}=U M_t U^{\top} .
$$
Using the definitions of $V_t$ and $V_{t+1}^*$,
$$
V_t=\operatorname{diag}\left(Q_t^{\top} C_t Q_t\right)=\operatorname{diag}\left(M_t\right),
\quad
V_{t+1}^*=\operatorname{diag}\left(Q_{t+1}^{\top} C_{t+1} Q_{t+1}\right).
$$
Hence,
$$
\begin{aligned}
\left\|V_t-V_{t+1}^*\right\|_F= & \left\|\operatorname{diag}\left(Q_t^{\top} C_t Q_t\right)-\operatorname{diag}\left(Q_{t+1}^{\top} C_{t+1} Q_{t+1}\right)\right\|_F \\
\leq & \left\|\operatorname{diag}\left(Q_t^{\top} C_t Q_t\right)-\operatorname{diag}\left(Q_{t+1}^{\top} C_t Q_{t+1}\right)\right\|_F \\
& +\left\|\operatorname{diag}\left(Q_{t+1}^{\top} C_t Q_{t+1}\right)-\operatorname{diag}\left(Q_{t+1}^{\top} C_{t+1} Q_{t+1}\right)\right\|_F.
\end{aligned}
$$
For the first term, using $M_t=Q_t^{\top} C_t Q_t$ and $Q_{t+1}^{\top} C_t Q_{t+1}=U M_t U^{\top}$,
$$
\left\|\operatorname{diag}\left(Q_t^{\top} C_t Q_t\right)-\operatorname{diag}\left(Q_{t+1}^{\top} C_t Q_{t+1}\right)\right\|_F=\left\|\operatorname{diag}\left(M_t\right)-\operatorname{diag}\left(U M_t U^{\top}\right)\right\|_F .
$$
The diagonal projection is non-expansive in Frobenius norm, so
$$
\left\|\operatorname{diag}\left(M_t\right)-\operatorname{diag}\left(U M_t U^{\top}\right)\right\|_F \leq\left\|M_t-U M_t U^{\top}\right\|_F
$$
Now,
$$
M_t-U M_t U^{\top}=(I-U) M_t U^{\top}+M_t\left(I-U^{\top}\right).
$$
Therefore,
$$
\begin{aligned}
\left\|M_t-U M_t U^{\top}\right\|_F & \leq\left\|(I-U) M_t U^{\top}\right\|_F+\left\|M_t\left(I-U^{\top}\right)\right\|_F \\
& \leq\|I-U\|_F\left\|M_t\right\|_2+\left\|M_t\right\|_2\left\|I-U^{\top}\right\|_F .
\end{aligned}
$$
Since $\left\|I-U^{\top}\right\|_F=\|I-U\|_F$, and orthogonal similarity preserves spectral norm,
$$
\left\|M_t\right\|_2=\left\|Q_t^{\top} C_t Q_t\right\|_2=\left\|C_t\right\|_2,
$$
we obtain
$$
\left\|M_t-U M_t U^{\top}\right\|_F \leq 2\left\|C_t\right\|_2\|I-U\|_F.
$$
For the second term, observe that
$$
\begin{aligned}
& \left\|\operatorname{diag}\left(Q_{t+1}^{\top} C_t Q_{t+1}\right)-\operatorname{diag}\left(Q_{t+1}^{\top} C_{t+1} Q_{t+1}\right)\right\|_F \\
& =\left\|\operatorname{diag}\left(Q_{t+1}^{\top}\left(C_t-C_{t+1}\right) Q_{t+1}\right)\right\|_F .
\end{aligned}
$$
Again, diagonal projection is non-expansive in Frobenius norm, and Frobenius norm is orthogonally invariant. Hence,
$$
\left\|\operatorname{diag}\left(Q_{t+1}^{\top}\left(C_t-C_{t+1}\right) Q_{t+1}\right)\right\|_F \leq\left\|C_t-C_{t+1}\right\|_F .
$$
Combining the preceding bounds gives
$$
\left\|V_t-V_{t+1}^*\right\|_F \leq 2\left\|C_t\right\|_2\|I-U\|_F+\left\|C_{t+1}-C_t\right\|_F .
$$
Finally,
$$
\begin{aligned}
\left\|V_{\text {scale }}-V_{t+1}^*\right\|_F & =\left\|\gamma V_t-V_{t+1}^*\right\|_F \\
& \leq(1-\gamma)\left\|V_t\right\|_F+\left\|V_t-V_{t+1}^*\right\|_F \\
& \leq(1-\gamma)\left\|V_t\right\|_F+2\left\|C_t\right\|_2\|I-U\|_F+\left\|C_{t+1}-C_t\right\|_F
\end{aligned}
$$
If $\left\|C_{t+1}-C_t\right\|_F \leq \Delta_t$, then the simplified bound follows immediately
$$
\left\|V_{\text {scale }}-V_{t+1}^*\right\|_F \leq(1-\gamma)\left\|V_t\right\|_F+2\left\|C_t\right\|_2\|I-U\|_F+\Delta_t .
$$
This completes the proof.
\end{proof}

\subsection{Proof of Theorem \ref{thm:self_scaling_convergence}}
\begin{proof}
Let
\begin{align*}
a_k=s_k^\top K_k s_k,
\qquad
c_k=s_k^\top y_k .
\end{align*}
Because \(K_k\succ0\) and \(s_k\neq0\), we have \(a_k>0\). By the curvature
assumption, \(c_k>0\). Consider
\begin{align*}
J(\tau)
=
\left[
\log
\frac{c_k}{s_k^\top(\tau K_k)s_k}
\right]^2
=
\left[
\log
\frac{c_k}{\tau a_k}
\right]^2 .
\end{align*}
Let
\begin{align*}
u=\log\frac{c_k}{a_k},
\qquad
z=\log\tau .
\end{align*}
The constraint \(0<\tau\le1\) is equivalent to \(z\le0\). Therefore
\begin{align*}
J(\tau)
=
(u-z)^2 .
\end{align*}
The unique minimizer over \(z\le0\) is the Euclidean projection of \(u\) onto
\((-\infty,0]\), namely
\begin{align*}
z^\star=\min\{u,0\}.
\end{align*}
Hence
\begin{align*}
\tau_k^\star
=
\exp(z^\star)
=
\min\left\{
1,
\frac{c_k}{a_k}
\right\}
=
\min\left\{
1,
\frac{s_k^\top y_k}{s_k^\top K_k s_k}
\right\}.
\end{align*}
This proves that it is the unique clipped log-secant matching solution. Finally, suppose
\begin{align*}
y_k=H_ks_k+r_k .
\end{align*}
Then
\begin{align*}
s_k^\top y_k
=
s_k^\top H_ks_k+s_k^\top r_k .
\end{align*}
When the unclipped solution is active, \(\tau_k=c_k/a_k\), and therefore
\begin{align*}
s_k^\top(\tau_kK_k)s_k
=
\tau_ka_k
=
c_k
=
s_k^\top H_ks_k+s_k^\top r_k .
\end{align*}
Thus, under an exact secant relation, the scaled Kronecker quadratic form matches
the true directional curvature exactly; under an approximate secant relation,
the mismatch is precisely the directional secant error. If the unconstrained
solution exceeds one, the constraint \(0<\tau\le1\) selects the closest admissible
log-scale match, namely \(\tau_k=1\).
\end{proof}

\section{Algorithms Pseudocode}

To connect our \method presented in Algorithm \ref{alg:ss_purifying_soap} with the established algorithms in the second-order optimizer literature, we provide the following technical context.

\subsection{Relation to Idealized and Warm-Started Shampoo}

Our algorithm can be viewed as a curvature-corrected approximation of Idealized Eigenvalue-Corrected Shampoo (Algorithm 1 in \citep{eschenhagen2025purifying}).

In the idealized setting, the exact eigendecomposition of the preconditioners $L_t$ and $R_t$ is computed at every step to perform the update $W_{t+1}=W_t- \eta L_t^{-1 / 2} G_t R_t^{-1 / 2}$. Standard SOAP approximates this efficiently using Warm-started QR iteration (Algorithm 4 in \citep{eschenhagen2025purifying}), which performs a single step of power iteration followed by QR decomposition to update the eigenbasis.

\method improves upon this approximation in two critical ways.

First, instead of a fixed schedule or a single QR step, we use the Purifying trigger (Step 2 in Algorithm \ref{alg:ss_purifying_soap}). We monitor the off-diagonal mass of the preconditioner in the current basis, $\rho(A, Q)$, which serves as a proxy for the approximation error of the stale eigenbasis. We only trigger a computationally expensive re-decomposition when this error exceeds $\tau_{\text{trigger}}$, ensuring the eigenbasis is refined only when the curvature landscape shifts significantly.

Second, we introduce the scalar $\tau_k$ (Step 4), which acts as a global correction factor for the eigenvalues, ensuring the trace of the inverse Hessian approximation matches the secant condition along the most recent displacement.

\subsection{Relation to EShampoo and Grafting}
We highlight two key advantages of our strategy over EShampoo (Algorithm 2 in \citep{eschenhagen2025purifying}) and Shampoo with Adam grafting.

\textbf{Stability via Downscaling vs. Reprojection.}
EShampoo typically employs a fixed eigenbasis computation frequency $F$. When the basis is updated, standard approaches often reproject the optimizer state (e.g., the second moment $V_t$ ) into the new basis via similarity transformation. As detailed in our analysis, this is numerically unstable: a significant rotation of the eigenbasis can map historical low-variance directions to high-curvature directions in the new basis, causing overly large updates. \method replaces reprojection with variance downscaling (Step 2, Soft Reset). By keeping the historical $V_t$ but reducing its magnitude by $\gamma$, we effectively increase the damping uniformly, preserving optimization history while conservatively forgetting the precise directional variance that is no longer valid.

\textbf{Intrinsic vs. Extrinsic Scaling.}
Shampoo with Adam grafting attempts to stabilize second-order updates by forcing the update magnitude or direction to align with Adam. This is an ad-hoc extrinsic constraint. In contrast, \method achieves intrinsic stability via the self-scaling factor $\tau_k$. This factor naturally damps the update when the local curvature is high (large denominator in \eqref{eq:tau_k}) and accelerates it in flat regions without the need for grafting or manual tuning of the epsilon parameter.

\begin{algorithm}[h]
\caption{\method, full pseudocode}
\label{alg:ss_purifying_soap}
\small
\begin{algorithmic}

\STATE \textbf{Require:} $W_0\in\mathbb{R}^{m\times n}$, $\eta>0$,
$\beta_1,\beta_2\in(0,1)$, $\epsilon>0$, $\lambda\geq0$,
$\tau_{\min}\in(0,1]$.
\STATE \textbf{Require:} $I_{\mathrm{check}}$,
$\tau_{\mathrm{trigger}}$, $T_{\mathrm{warm}}$.
\STATE \textbf{Initialize:}
$\widetilde M_0=V_0=\mathbf{0}$,
$L_0=\epsilon I_m$, $R_0=\epsilon I_n$,
$Q_L=I_m$, $Q_R=I_n$.

\FOR{$t=1,\ldots,T$}

    \STATE \textbf{1. Gradient and Kronecker factors}
    \STATE $G_t\leftarrow\nabla_W\mathcal{L}(W_{t-1})$
    \STATE $L_t\leftarrow\beta_2L_{t-1}+(1-\beta_2)G_tG_t^\top$,
    $R_t\leftarrow\beta_2R_{t-1}+(1-\beta_2)G_t^\top G_t$

    \STATE \textbf{2. Adaptive eigenbasis update}
    \IF{$t>T_{\mathrm{warm}}$ \textbf{and}
    $t\bmod I_{\mathrm{check}}=0$}
        \STATE $\displaystyle
        \rho(A,Q)\leftarrow
        \frac{\|Q^\top AQ-\operatorname{diag}(Q^\top AQ)\|_F}
        {\|Q^\top AQ\|_F+\epsilon}$,
        $\quad
        \rho_t\leftarrow
        \max\{\rho(L_t,Q_L),\rho(R_t,Q_R)\}$
        \IF{$\rho_t>\tau_{\mathrm{trigger}}$}
            \STATE $(Q_L^{\mathrm{old}},Q_R^{\mathrm{old}})
            \leftarrow(Q_L,Q_R)$
            \STATE $Q_L\leftarrow\operatorname{eigvec}(L_t)$,
            $Q_R\leftarrow\operatorname{eigvec}(R_t)$
            \STATE $\widetilde M_{t-1}\leftarrow
            (Q_L^\top Q_L^{\mathrm{old}})
            \widetilde M_{t-1}
            ((Q_R^{\mathrm{old}})^\top Q_R)$
            \STATE $\gamma_t\leftarrow
            0.25$ if $\rho_t>0.8$,
            $0.5$ if $\rho_t>0.5$,
            and $0.75$ otherwise
            \STATE $V_{t-1}\leftarrow\gamma_tV_{t-1}$
        \ENDIF
    \ENDIF

    \STATE \textbf{3. Eigenspace statistics}
    \STATE $\widetilde G_t\leftarrow Q_L^\top G_tQ_R$
    \STATE $\widetilde M_t\leftarrow
    \beta_1\widetilde M_{t-1}+(1-\beta_1)\widetilde G_t$,
    $\quad
    V_t\leftarrow
    \beta_2V_{t-1}+(1-\beta_2)\widetilde G_t^{\odot2}$

    \STATE \textbf{4. Self-scaling correction}
    \IF{$t=1$}
        \STATE $\tau_t\leftarrow1$
    \ELSE
        \STATE $S_t\leftarrow W_{t-1}-W_{t-2}$,
        $\quad Y_t\leftarrow G_t-G_{t-1}$
        \STATE $c_t\leftarrow\operatorname{Tr}(Y_t^\top S_t)$,
        $\quad
        a_t\leftarrow
        \operatorname{Tr}(S_t^\top L_t^{-1}S_tR_t^{-1})$
        \IF{$c_t>0$ \textbf{and} $a_t>0$}
            \STATE $\displaystyle
            \tau_t\leftarrow
            \min\{1,\max\{\tau_{\min},c_t/a_t\}\}$
        \ELSE
            \STATE $\tau_t\leftarrow1$
        \ENDIF
    \ENDIF

    \STATE \textbf{5. Parameter update}
    \STATE $\widehat M_t\leftarrow\widetilde M_t/(1-\beta_1^t)$,
    $\quad
    \widehat V_t\leftarrow V_t/(1-\beta_2^t)$
    \STATE $\widetilde U_t\leftarrow
    \tau_t^{-1/2}\widehat M_t
    \oslash(\sqrt{\widehat V_t}+\epsilon)$
    \STATE $U_t\leftarrow Q_L\widetilde U_tQ_R^\top$,
    $\quad
    W_t\leftarrow(1-\eta\lambda)W_{t-1}-\eta U_t$

\ENDFOR
\end{algorithmic}
\end{algorithm}

\subsection{Efficient Computation of $\tau_k$}
Similar to \eqref{eq:tau_k_1}, we avoid direct computation of the inverse preconditioner $H_k^{-1}$, as it is $O(n^3)$ and can be numerically unstable when the curvature matrices are ill-conditioned. In our setting, the preconditioner admits a Kronecker-factored eigendecomposition
\begin{align*}
H_k &= L_k \otimes R_k, \\
L_k &= Q_{L,k}\operatorname{diag}(D_{L,k})\,Q_{L,k}^\top, \\
R_k &= Q_{R,k}\operatorname{diag}(D_{R,k})\,Q_{R,k}^\top,
\end{align*}
where $Q_{L,k}$ and $Q_{R,k}$ are orthogonal matrices and $D_{L,k}$, $D_{R,k}$ contain the corresponding eigenvalues.

By rotating tensors into the eigenspaces of $L_k$ and $R_k$, the application of $H_k^{-1}$ reduces to elementwise scaling by the inverse eigenvalues $D_{L,k}^{-1}$ and $D_{R,k}^{-1}$. Consequently, all terms required to compute $\tau_k$ can be evaluated using matrix-vector products and elementwise operations, without explicitly forming inverse matrices.

Specifically, defining the rotated quantities
\begin{align*}
\tilde{S}_k &= Q_{L,k}^\top S_k Q_{R,k}, \quad \tilde{Y}_k = Q_{L,k}^\top Y_k Q_{R,k},
\end{align*}
the denominator in~\eqref{eq:tau_k} can be written as
\begin{equation}
\operatorname{Tr} (S_k^\top L_k^{-1} S_k R_k^{-1})
=
\sum_{i,j}
\frac{\tilde{S}_{k,ij}^2}{D_{L,k,i}\,D_{R,k,j}},
\end{equation}
which is computable in $O(n)$ time. On the other hand, the numerator simplifies to
\begin{align*}
\operatorname{Tr}( Y_k^\top S_k)
=
\sum_{i,j} \tilde{Y}_{k,ij}\,\tilde{S}_{k,ij}.
\end{align*}
Thus, the scaling parameter is evaluated as
\begin{equation}
\tau_k
=
\min\left\{
1,\;
\frac{\sum_{i,j} \tilde{Y}_{k,ij}\,\tilde{S}_{k,ij}}
{\sum_{i,j} \tilde{S}_{k,ij}^2 / (D_{L,k,i} D_{R,k,j})}
\right\},
\end{equation}
without ever forming $H_k^{-1}$. This formulation preserves numerical stability while maintaining the computational efficiency required for large-scale optimization.

%%%%%%%%%%%%%%%%%%%%%%%%%%%%%%%%%%%%%%%%%%%%%%%%%%%%
%%%%%%%%%%%%%%%%%%%%%%%%%%%%%%%%%%%%%%%%%%%%%%%%%%%%

\section{Architectures and Additional Design Choices}

\subsection{PirateNet}
PirateNet \cite{wang2024piratenets} aims to enable stable and efficient training of deep PINN models. It first transforms input coordinates $\mathbf{x}$ into a high-dimensional feature space using random Fourier features
$$
\Phi(\mathbf{x})=\left[\begin{array}{c}
\cos (\mathbf{B x}) \\
\sin (\mathbf{B x})
\end{array}\right],
$$
where $\mathbf{B} \in \mathbb{R}^{m \times d}$ has entries sampled i.i.d. from $\mathcal{N}(0, s^2)$ with user-specified $s>0$. This embedding mitigates spectral bias in PINNs by improving the eigenfunction frequency of the Neural Tangent Kernel, enabling better learning of high-frequency components and multiscale features.
The embedded coordinates are processed through two dense layers that act as gates $\mathbf{U}=\sigma\left(\mathbf{W}_1 \Phi(\mathbf{x})+\mathbf{b}_1\right)$ and $\mathbf{V}=\sigma\left(\mathbf{W}_2 \Phi(\mathbf{x})+\mathbf{b}_2\right)$,
where $\sigma$ is a point-wise activation function. This gating mechanism is essentially the same as in modified MLP.
Let $\mathbf{x}^{(1)}=\Phi(\mathbf{x})$ and $\mathbf{x}^{(l)}$ be the input to the $l$-th block. Each block performs
$$
\begin{aligned}
\mathbf{f}^{(l)} =\sigma\left(\mathbf{W}_1^{(l)} \mathbf{x}^{(l)}+\mathbf{b}_1^{(l)}\right), \quad \mathbf{z}_1^{(l)} =\mathbf{f}^{(l)} \odot \mathbf{U}+\left(1-\mathbf{f}^{(l)}\right) \odot \mathbf{V}, \\
\mathbf{g}^{(l)} =\sigma\left(\mathbf{W}_2^{(l)} \mathbf{z}_1^{(l)}+\mathbf{b}_2^{(l)}\right), \quad \mathbf{z}_2^{(l)} =\mathbf{g}^{(l)} \odot \mathbf{U}+\left(1-\mathbf{g}^{(l)}\right) \odot \mathbf{V}, \\
\mathbf{h}^{(l)} =\sigma\left(\mathbf{W}_3^{(l)} \mathbf{z}_2^{(l)}+\mathbf{b}_3^{(l)}\right), \quad \mathbf{x}^{(l+1)} =\alpha^{(l)} \mathbf{h}^{(l)}+\left(1-\alpha^{(l)}\right) \mathbf{x}^{(l)}.
\end{aligned}
$$
Each block comprises three dense layers with dual gating operations and an adaptive residual connection. The trainable $\alpha^{(l)}$ parameters control block nonlinearity: $\alpha^{(l)}=0$ yields an identity mapping, while $\alpha^{(l)}=1$ produces fully nonlinear transformation.
The final output of a PirateNet of $L$ residual blocks is given by $\mathbf{u}_\theta=\mathbf{W}^{(L+1)} \mathbf{x}^{(L)}$. 
Following the design choices made by \citep{wang2024piratenets}, we initialize $\alpha^{(l)}=0$, making the initial output a linear combination of first-layer embeddings. This initialization strategy mitigates training difficulties in deep networks by starting with effectively shallow architecture and gradually increasing depth through learned $\alpha$ values. Additionally, the linear structure at initialization enables direct integration of prior solution data through least squares fitting
$$
\min _{\mathbf{W}}\|\mathbf{W} \Phi-\mathbf{Y}\|_2^2,
$$
where $\mathbf{Y}$ represents available measurements. This approach provides an optimal initial guess based on various data sources, including experimental measurements, boundary conditions, or linearized PDE solutions.

\begin{table}[t]
\caption{Comparison of optimization methods showing preconditioner types, storage and computational complexity for an $n\times n$ weight matrix, where $P$ is the parameter count and $N$ is the batch size of collocation points, and practical compatibility with mini-batches and scalability with large neural networks.}
\label{sample-table}
\vskip -0.2in
\begin{center}
\footnotesize
\begin{sc}
\begin{tabular}{lcccccr}
\toprule
Method & (Approx.) Precond. & Storage & Computation & Mini-batch & DNN \\
\midrule
Natural Gradient & $F^{-1}$ & $O(n^4)$ & $\min(O(PN^2), O(NP^2))$ & $\times$ & $\times$ \\
% BFGS/L-BFGS & $H^{-1}$ & $O(n^2)$ & $O(n^2)$ & $\times$ & $\sqrt{}$ \\
BFGS & $H^{-1}$ & $O(n^2)$ & $O(n^4)$ & $\times$ & $\sqrt{}$ \\
L-BFGS & $H^{-1}$ & $O(n^2)$ & $O(n^3)$ & $\times$ & $\sqrt{}$ \\
SOAP & $H^{-1}$ & $O(n^2)$ & $O(n^3)$ & $\sqrt{}$ & $\sqrt{}$ \\
\method & $\tau H^{-1}$ & $O(n^2)$ & $O(n^3)$ & $\sqrt{}$ & $\sqrt{}$ \\
Shampoo & $H_{\text {ada }}^{-1 / 2}$ & $O\left(n^2\right)$ & $O\left(n^3\right)$ & $\sqrt{}$ & $\sqrt{}$ \\
Muon & $(GG^\top)^{-1 / 2}$ & $O\left(n^2\right)$ & $O\left(n^3\right)$ & $\sqrt{}$ & $\sqrt{}$ \\
\bottomrule
\end{tabular}
\end{sc}
\end{center}
\vskip -0.1in
\end{table}

\subsection{Additional Details of Standard PDE Benchmarks}
\label{sec:benchmarks_description}

\textbf{Burgers equation.} The 1D Burgers equation is defined as
$$
u_t+u u_x=\nu u_{x x},
$$
where $u$ represents the velocity field, and $\nu$ is the kinematic viscosity coefficient controlling the diffusion strength. Here we set $(x, t) \in \Omega=[-1,1] \times[0,1]$, with initial and boundary conditions
$$
\begin{aligned}
u(x, 0) & =-\sin (\pi x) \\
u(-1, t) & =u(1, t)=0,
\end{aligned}
$$
and viscosity parameter $\nu=0.01 / \pi$. Following \citep{wang2025high}, we enforce a hard constraint via Taylor expansion at the origin as $U=-z+z^3+z^4 U_1$, for an odd function $U_1$. 

\textbf{2D Boussinesq equation.} For the 2D Boussinesq equation on the half plane, in vorticity form with the self-similar ansatz, we get the following profile equations for $\left(\Omega, U_1, U_2, \Phi, \Psi\right)$ as in \citep{wang2023asymptotic}:
$$
\begin{aligned}
\Omega+\left((1+\lambda)\left(y_1, y_2\right)^T+\left(U_1, U_2\right)^T\right) \cdot \nabla \Omega & =\Phi \\
\left(2+\partial_{y_1} U_1\right) \Phi+\left((1+\lambda)\left(y_1, y_2\right)^T+\left(U_1, U_2\right)^T\right) \cdot \nabla \Phi & =-\partial_{y_1} U_2 \Psi \\
\left(2+\partial_{y_2} U_2\right) \Psi+\left((1+\lambda)\left(y_1, y_2\right)^T+\left(U_1, U_2\right)^T\right) \cdot \nabla \Psi & =-\partial_{y_2} U_1 \Phi \\
\partial_{y_1} U_1+\partial_{y_2} U_2=0, \quad \Omega=\partial_{y_1} U_2-\partial_{y_2} U_1, \quad \partial_{y_1} \Psi & =\partial_{y_2} \Phi,
\end{aligned}
$$
where $(\Omega, U_1, \Phi)$ are odd and $(U_2, \Psi)$ are even in $y_1$ and we are in the half plane $y_2 \geq 0$. For the boundary conditions, we impose a non-penetration boundary condition $U_2\left(y_1, 0\right)=0$ along with decaying weak asymptotics at the far field, with Dirichlet boundary conditions $\Phi=\Psi=0$ and Neumann boundary conditions for the velocity field $\nabla\left(U_1, U_2\right)^T=0$.

For the nondegeneracy condition, we adapt the approach proposed in \citep{wang2025high} and impose $\partial_{y_1} \Omega(0,0)=-1$ and use Taylor expansion to enforce a hard constrain to rule out the trivial solution $U=0$ when using weak asymptotics. We enforce $\partial_1 \Omega(0,0)=-1$ and $\Omega$ is odd in $z_1$ via a Taylor expansion as $\Omega=-z_1+z_1 z_2 \Omega_1+z_1^2 \Omega_2$, where $\Omega_1, \Omega_2$ are even and odd functions in $z_1$ respectively.

\textbf{Wave equation.} We consider a one-dimensional wave equation in the domain $\Omega=[0,1] \times[0,1]$ taking the form
$$
\begin{aligned}
& u_{t t}(x, t)-4 u_{x x}(x, t)=0, \quad(x, t) \in(0,1) \times(0,1), \\
& u(0, t)=u(1, t)=0, \quad t \in[0,1], \\
& u(x, 0)=\sin (\pi x)+\frac{1}{2} \sin (4 \pi x), \quad x \in[0,1], \\
& u_t(x, 0)=0, \quad x \in[0,1] .
\end{aligned}
$$
where $u$ represents the wave amplitude, and $c$ is the wave propagation speed, determined by the medium's physical properties. By d'Alembert's formula, the solution $u(x, t)$ is given by
$$
u(x, t)=\sin (\pi x) \cos (2 \pi t)+\frac{1}{2} \sin (4 \pi x) \cos (8 \pi t) .
$$

\begin{figure}[h]
  \centering

  \begin{subfigure}[t]{0.3\linewidth}
    \centering
    \includegraphics[width=1\linewidth]{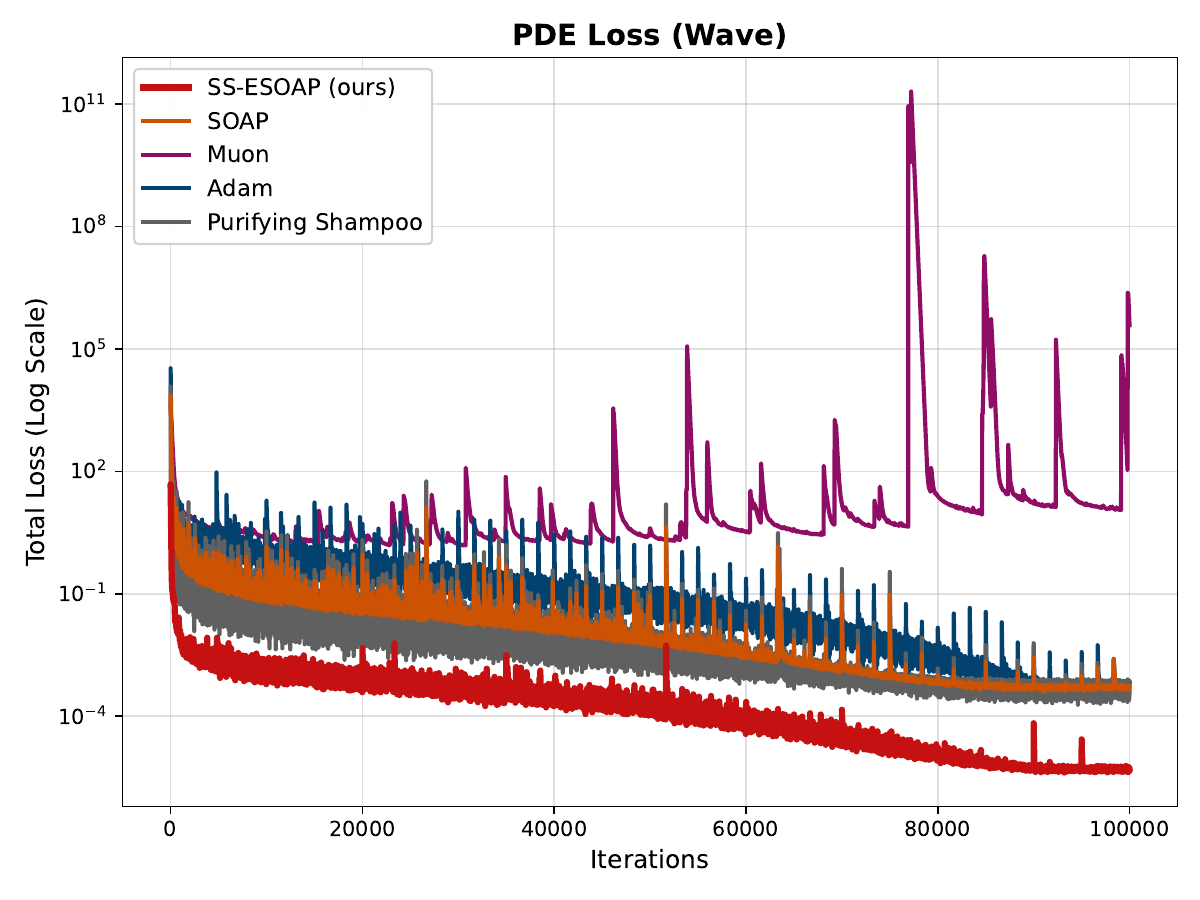}
    \caption{Training loss convergence.}
    \label{fig:wave_128}
  \end{subfigure}
  % \hfill
  \begin{subfigure}[t]{0.6\linewidth}
    \centering
    \includegraphics[width=1\linewidth]{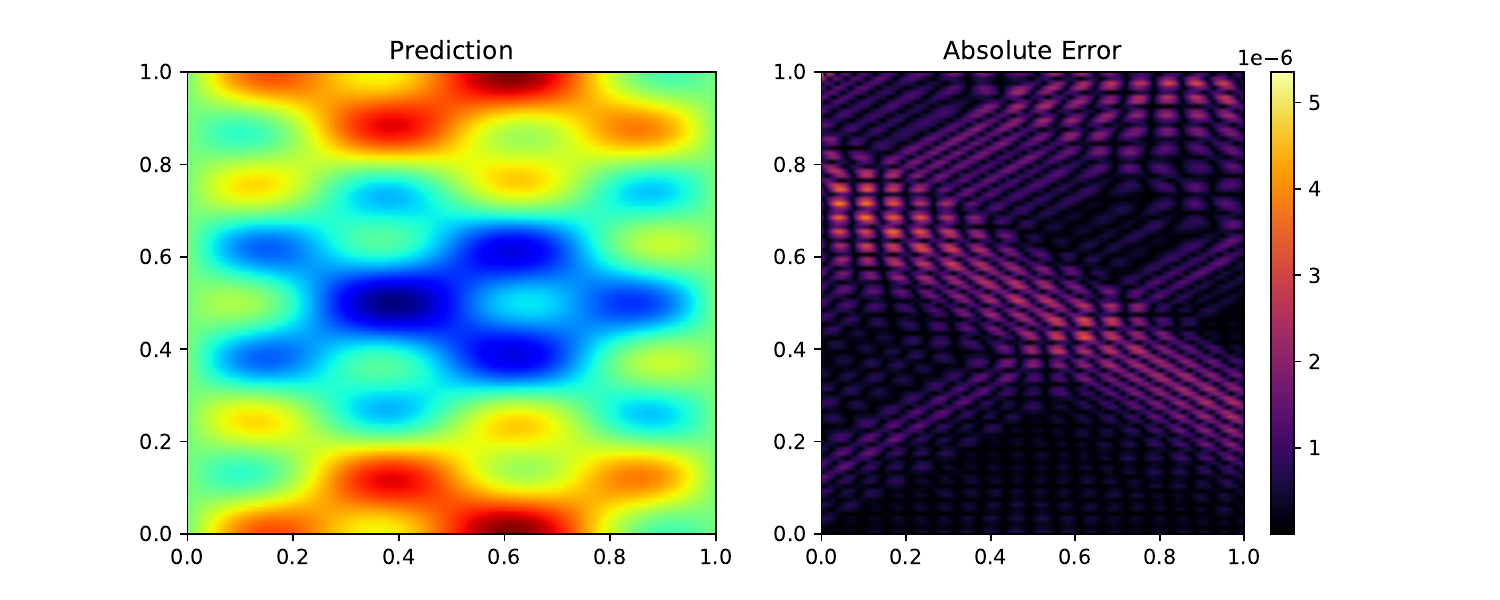}
    \caption{Final prediction and absolute error field.}
    \label{fig:wave_heatmap}
  \end{subfigure}
  \caption{\textbf{Wave Equation Benchmark.} (a) PDE residual loss trajectories (log scale) comparing \method (red) against Adam, Muon, SOAP, and Purifying Shampoo over 100k iterations. \method achieves the lowest final loss, while Muon exhibits significant instability. (b) Visualization of the predicted solution field and the corresponding absolute error map at the final training step, demonstrating high-precision recovery of the dynamics with errors on the magnitude of $10^{-6}$.}
  % Comparison between the reference solution and model predictions.
  \label{fig:wave_plots}
\end{figure}

\begin{figure}[h]
  \centering

  \begin{subfigure}[t]{0.3\linewidth}
    \centering
    \includegraphics[width=1\linewidth]{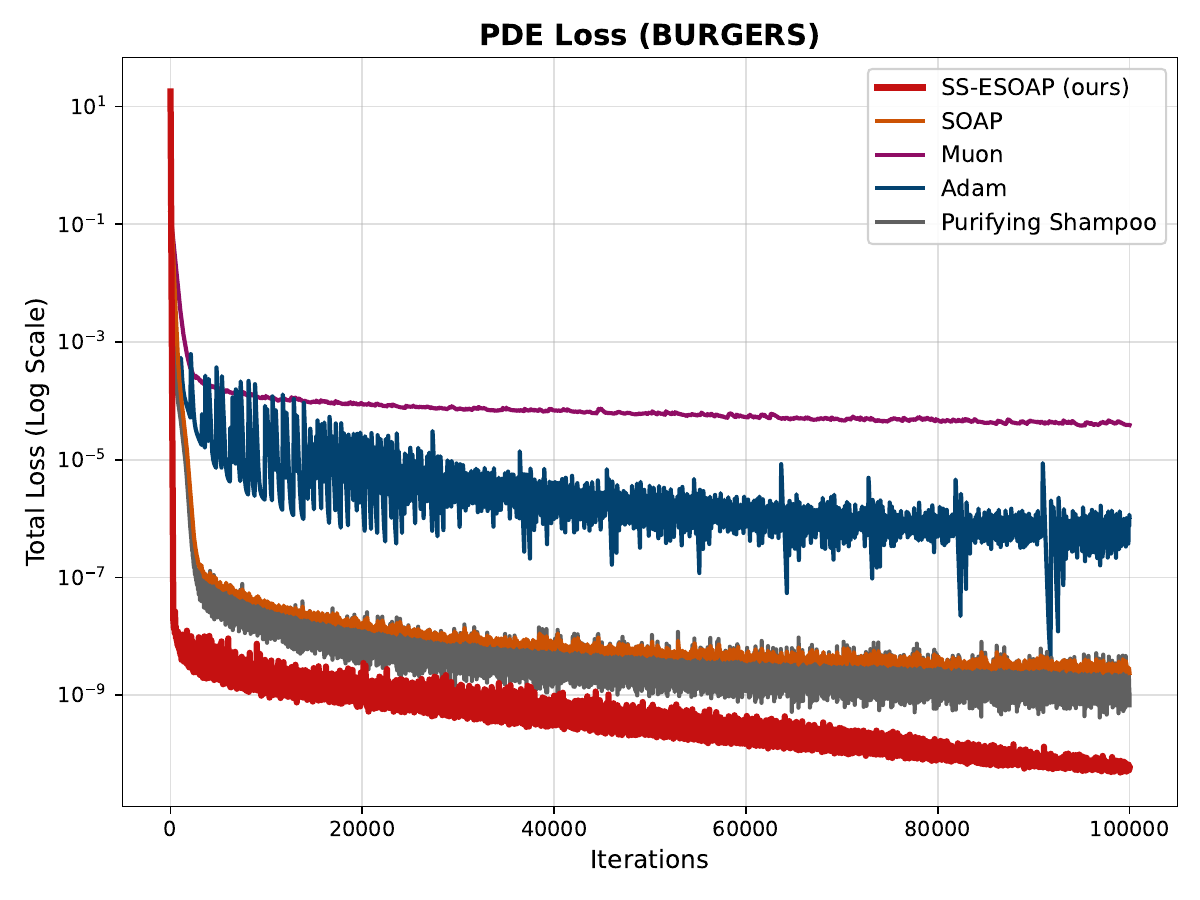}
    \caption{Training loss convergence.}
    \label{fig:burger_128}
  \end{subfigure}
  % \hfill
  \begin{subfigure}[t]{0.6\linewidth}
    \centering
    \includegraphics[width=1\linewidth]{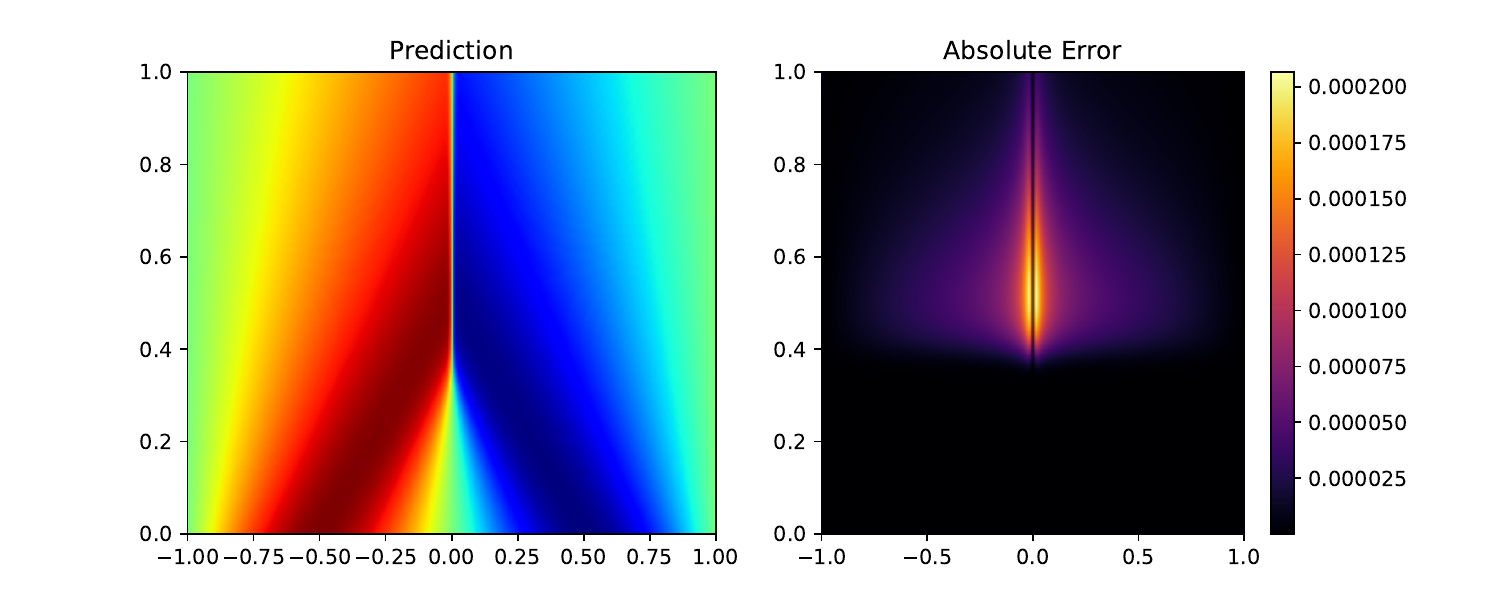}
    \caption{Final prediction and absolute error field.}
    \label{fig:burger_heatmap}
  \end{subfigure}
  \caption{\textbf{Burgers Equation Benchmark.} (a) PDE residual loss trajectories (log scale). \method (red) demonstrates superior convergence speed and stability, reaching a final residual orders of magnitude lower than Adam and Muon. (b) Visualization of the predicted solution and absolute error map. The low error magnitude confirms the optimizer's ability to accurately resolve the sharp shock interface characteristic of the inviscid limit.}
  % Comparison between the reference solution and model predictions.
  \label{fig:burger_plots}
\end{figure}

\textbf{Allen-Cahn equation.} We investigate the one-dimensional Allen-Cahn equation with periodic boundary conditions
% \begin{wrapfigure}{r}{0.35\textwidth}
%   \vspace{-1.0em}
%   \centering
%   \includegraphics[width=\linewidth]{plots_new/ac_comparison_best.pdf}
%   \vspace{-1.0em}
%   \caption{}
%   \label{fig:ac_comparison_best}
% \end{wrapfigure}
$$
\begin{aligned}
& u_t-0.0001 u_{x x}+5 u^3-5 u=0, \quad t \in[0,1], x \in[-1,1] \\
& u(0, x)=x^2 \cos (\pi x) \\
& u(t,-1)=u(t, 1), \quad u_x(t,-1)=u_x(t, 1)
\end{aligned}
$$
where $u$ represents the order parameter (e.g., concentration difference between two phases), $\epsilon$ controls the interfacial width, $a$ is the reaction rate coefficient, and the term ($u-u^3$) drives the phase separation.

\textbf{Korteweg-de Vries equation.} The one-dimensional KdV equation is expressed as
$$
\begin{aligned}
& u_t+\eta u u_x+\mu^2 u_{x x x}=0, \quad t \in(0,1), \quad x \in(-1,1), \\
& u(x, 0)=\cos (\pi x), \\
& u(t,-1)=u(t, 1),
\end{aligned}
$$
where $u$ represents the wave amplitude or water surface elevation, and $\eta$ governs the strength of the nonlinearity, while $\mu$ controls the dispersion level. Under the KdV dynamics, this initial wave evolves into a series of solitary-type waves. Like in recent PINN optimization works, we adopt the classical parameters of the KdV equation, setting $\eta=1$ and $\mu=0.022$.

% \begin{figure}[H]
%   \centering    
%   \begin{subfigure}[t]{0.3\linewidth}
%     \centering
%     \includegraphics[width=\linewidth]{plots/comparison_kdv_dim128.pdf}
% \label{fig:kdv_128}
% \vspace{-1.5em}
% \caption{Losses, residual, gradient condition numbers, and nuclear norms of matrix quadratic regression.}
% % \vspace{1em}
%     \includegraphics[width=\linewidth]{plots/comparison_kdv_dim256.pdf}
% \label{fig:kdv_256}
% \vspace{-1.5em}
%   \caption{}
% \end{subfigure}
%   \caption{}
% \end{figure}

\begin{figure}[h]
  \centering

  \begin{subfigure}[t]{0.3\linewidth}
    \centering
    \includegraphics[width=1\linewidth]{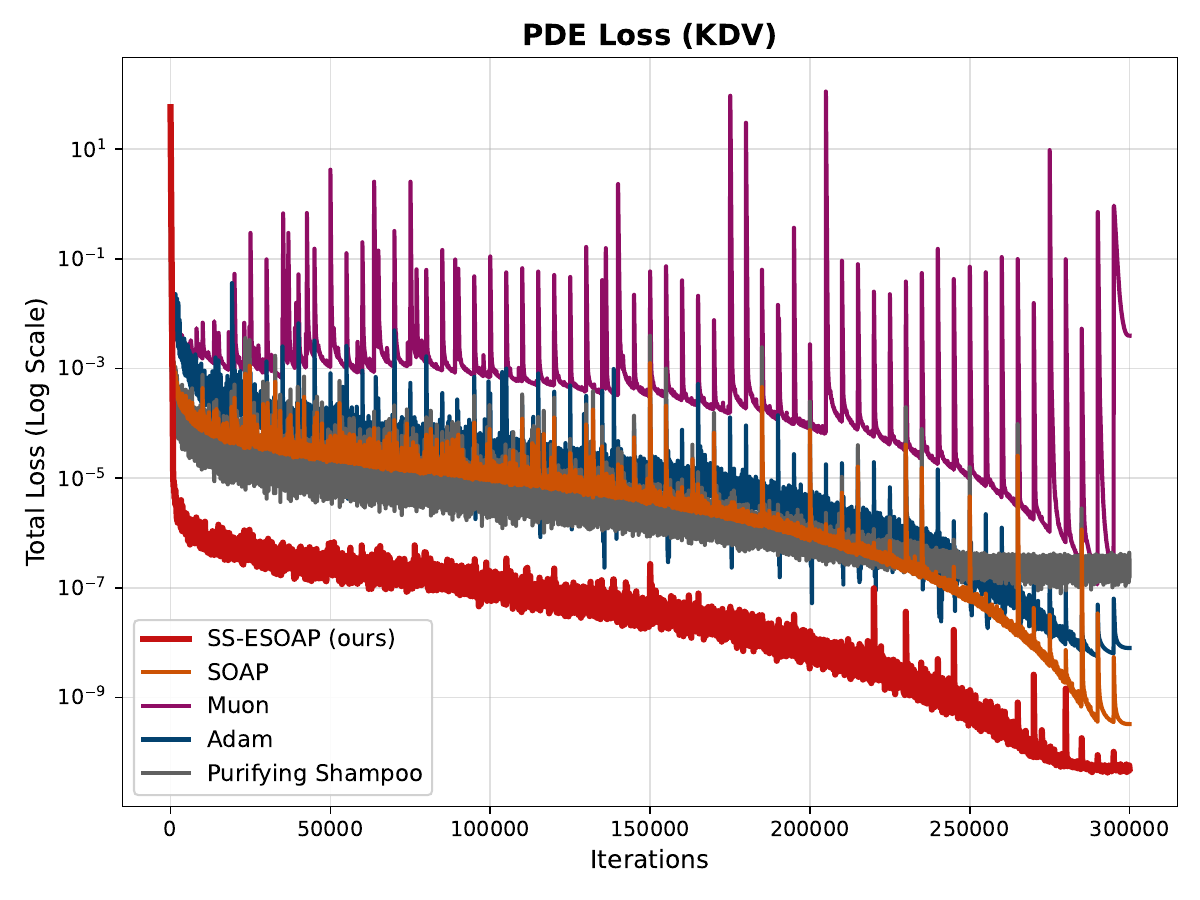}
    \caption{Training loss convergence comparison.}
    \label{fig:kdv_128}
  \end{subfigure}
  % \hfill
  \begin{subfigure}[t]{0.6\linewidth}
    \centering
    \includegraphics[width=1\linewidth]{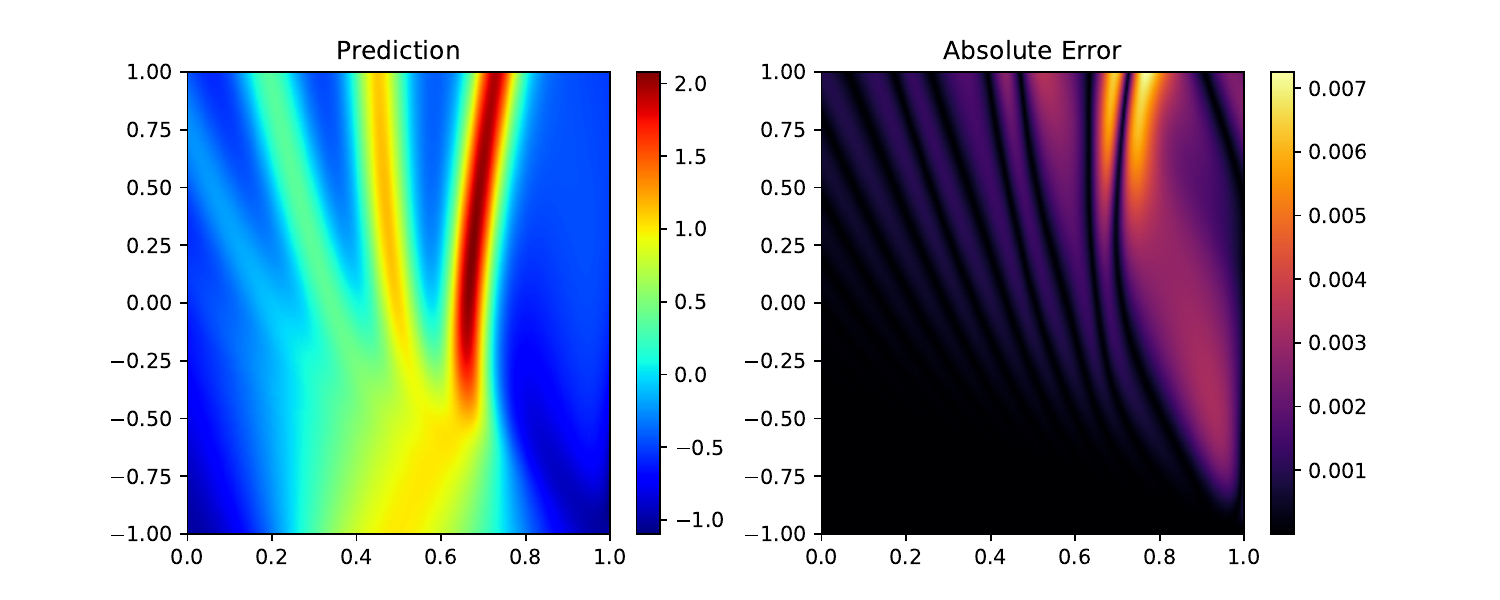}
    \caption{Final prediction and absolute error field.}
    \label{fig:kdv_heatmap}
  \end{subfigure}
  \caption{\textbf{Korteweg-de Vries (KdV) Equation Benchmark.} (a) PDE residual loss trajectories (log scale). \method (red) maintains a steep linear convergence rate, reaching a final residual of $10^{-10}$, significantly outperforming the baselines. (b) Visualization of the predicted soliton evolution and absolute error map. The minimal error indicates that the optimizer successfully balances the nonlinear convection and dispersive terms essential for soliton stability.}
  % Comparison between the reference solution and model predictions.
  \label{fig:kdv_plots}
\end{figure}

\textbf{Gray-Scott equation.} The system is described by the following coupled PDEs:
$$
\begin{aligned}
& u_t=\epsilon_1 \Delta u+b_1(1-u)-c_1 u v^2, \quad t \in(0,2),(x, y) \in(-1,1)^2, \\
& v_t=\epsilon_2 \Delta v-b_2 v+c_2 u v^2, \quad t \in(0,2),(x, y) \in(-1,1)^2,
\end{aligned}
$$
With periodic boundary conditions, the initial conditions are
$$
\begin{aligned}
& u_0(x, y)=1-\exp \left(-10\left((x+0.05)^2+(y+0.02)^2\right)\right), \\
& v_0(x, y)=1-\exp \left(-10\left((x-0.05)^2+(y-0.02)^2\right)\right) .
\end{aligned}
$$
where $u$ and $v$ represent activator and inhibitor concentrations respectively, $\varepsilon_1$ and $\varepsilon_2$ are diffusion coefficients, and $\left(b_1, b_2, c_1, c_2\right)$ control reaction kinetics. This system generates diverse spatial patterns including spots and stripes. We set parameters $\epsilon_1=0.2, \epsilon_2=0.1, b_1=40, b_2=100$, and $c_1=c_2=1,000$, which generates characteristic pattern formations.

\begin{figure}[h]
  \centering
  \begin{subfigure}[t]{0.4\linewidth}
    \centering
    \includegraphics[width=1\linewidth]{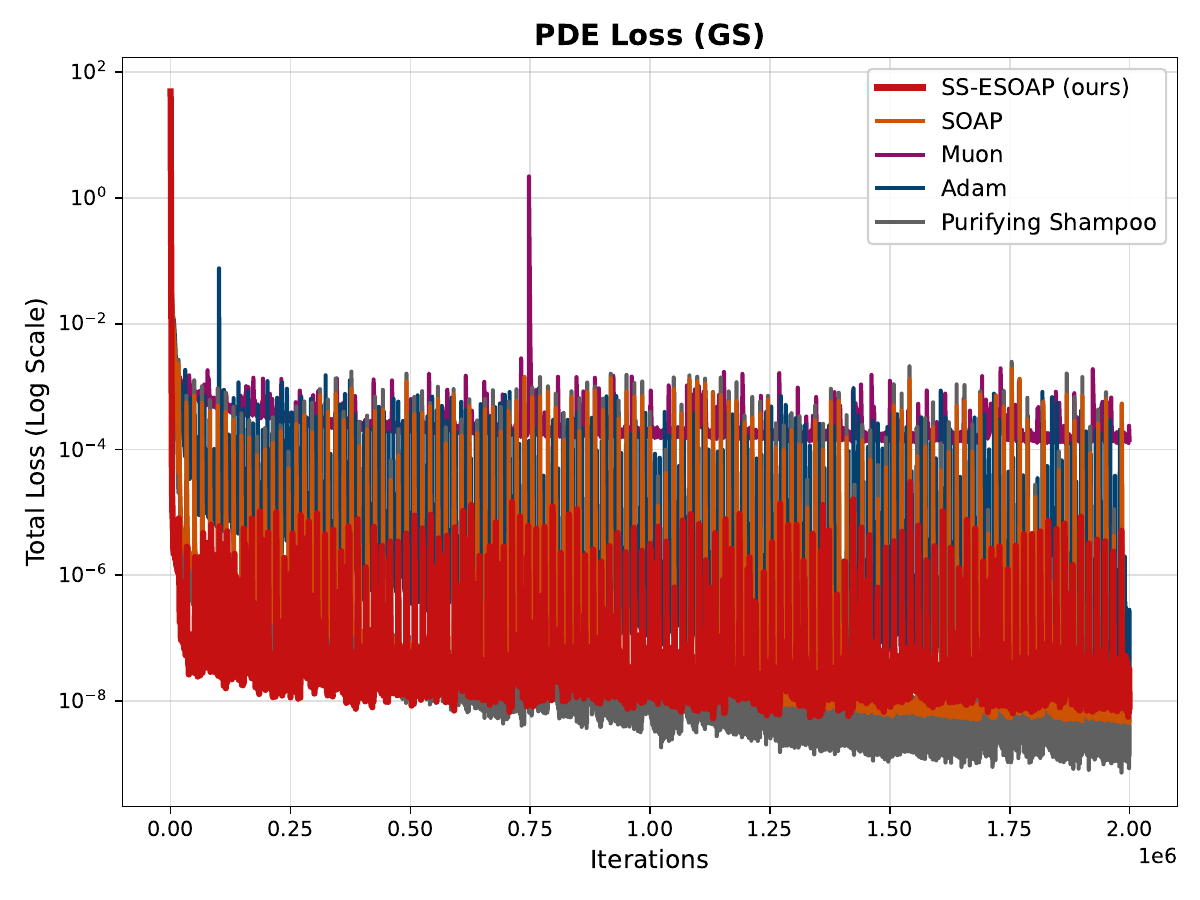}
    \caption{GS equation}
    \label{fig:gs_128}
  \end{subfigure}
  % \hfill
  \hspace{2em}
  \begin{subfigure}[t]{0.4\linewidth}
    \centering
    \includegraphics[width=1\linewidth]{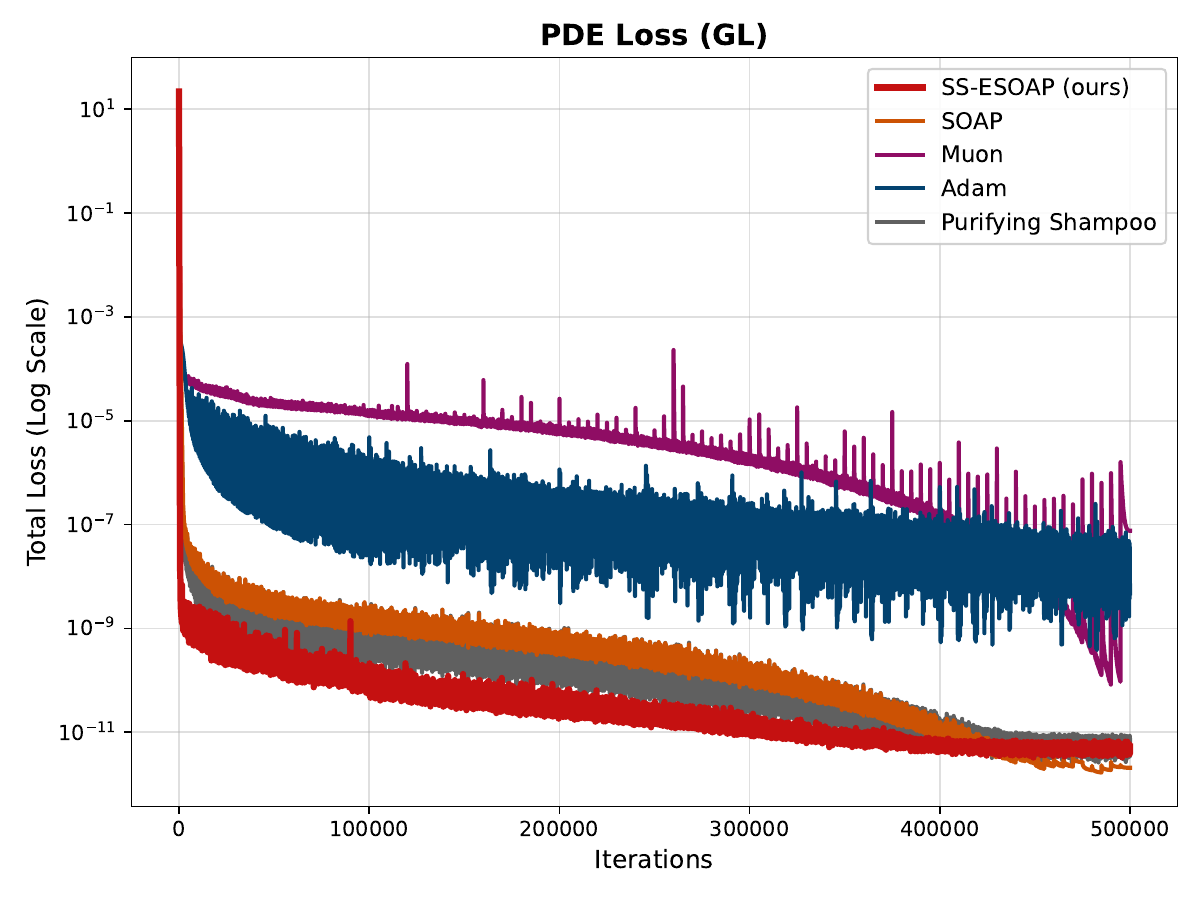}
    \caption{GL equation.}
    \label{fig:gl_128}
  \end{subfigure}
  % \hfill
  % \begin{subfigure}[t]{0.35\linewidth}
  %   \centering
  %   \includegraphics[width=1\linewidth]{plots/comparison_gs_dim512.pdf}
  %   \caption{GS equation, width $=512$.}
  %   \label{fig:gs_512}
  % \end{subfigure}
  % % \hfill
  % \begin{subfigure}[t]{0.35\linewidth}
  %   \centering
  %   \includegraphics[width=1\linewidth]{plots/comparison_gs_dim1024.pdf}
  %   \caption{GS equation, width $=1024$.}
  %   \label{fig:gs_1024}
  % \end{subfigure}
  \caption{Gray-Scott and Ginzburg-Landau equations. Training loss trajectories of Adam, Muon, SOAP, and \method across different network sizes, with MLP as training backbone.}
  % Comparison between the reference solution and model predictions.
  \label{fig:gs_plots}
\end{figure}

% \begin{figure}[h]
%   \centering
%   \begin{subfigure}[t]{0.35\linewidth}
%     \centering
%     \includegraphics[width=1\linewidth]{plots/comparison_gs_dim128.pdf}
%     \caption{GS equation, width $=128$.}
%     \label{fig:gs_128}
%   \end{subfigure}
%   % \hfill
%   \begin{subfigure}[t]{0.35\linewidth}
%     \centering
%     \includegraphics[width=1\linewidth]{plots/comparison_gs_dim256.pdf}
%     \caption{GS equation, width $=256$.}
%     \label{fig:gs_256}
%   \end{subfigure}
%   \hfill
%   \begin{subfigure}[t]{0.35\linewidth}
%     \centering
%     \includegraphics[width=1\linewidth]{plots/comparison_gs_dim512.pdf}
%     \caption{GS equation, width $=512$.}
%     \label{fig:gs_512}
%   \end{subfigure}
%   % \hfill
%   \begin{subfigure}[t]{0.35\linewidth}
%     \centering
%     \includegraphics[width=1\linewidth]{plots/comparison_gs_dim1024.pdf}
%     \caption{GS equation, width $=1024$.}
%     \label{fig:gs_1024}
%   \end{subfigure}
%   \caption{Gray-Scott equation. Training loss trajectories of Adam, Muon, SOAP, and \method across different network sizes, with PirateNet as the training backbone.}
%   % Comparison between the reference solution and model predictions.
%   \label{fig:gs_plots}
% \end{figure}

\textbf{Ginzburg-Landau equation.} The complex Ginzburg-Landau equation in 2D takes the form
$$
\frac{\partial A}{\partial t}=\epsilon \Delta A+\mu A-\gamma A|A|^2, \quad t \in(0,1),(x, y) \in(-1,1)^2,
$$
with periodic boundary conditions, an initial condition
$$
A_0(x, y)=(10 y+10 i x) \exp \left(-0.01\left(2500 x^2+2500 y^2\right)\right),
$$
where $A$ is the complex amplitude representing the envelope of oscillations, $\epsilon$ represents the diffusion coefficient, $\mu$ is the linear growth rate, and $\gamma$ controls the nonlinear saturation. For this example, we set $\epsilon=0.004, \mu=10$ and $\gamma=10+15 i$.
By denoting $A=u+i v$, we can decompose the equation into real and imaginary components, resulting in the following system of PDEs,
$$
\begin{aligned}
& \frac{\partial u}{\partial t}=\epsilon \Delta u+\mu\left(u-(u-1.5 v)\left(u^2+v^2\right)\right), \\
& \frac{\partial v}{\partial t}=\epsilon \Delta v+\mu\left(v-(v+1.5 u)\left(u^2+v^2\right)\right) .
\end{aligned}
$$

\iffalse

\begin{figure}[h]
  \centering
  \begin{subfigure}[t]{0.35\linewidth}
    \centering
    \includegraphics[width=1\linewidth]{plots_new/comparison_gl.pdf}
    \caption{GL equation, width $=128$.}
    \label{fig:gl_128}
  \end{subfigure}
  % \hfill
  % \begin{subfigure}[t]{0.35\linewidth}
  %   \centering
  %   \includegraphics[width=1\linewidth]{plots/comparison_gl_dim256.pdf}
  %   \caption{GL equation, width $=256$.}
  %   \label{fig:gl_256}
  % \end{subfigure}
  % \hfill
  % \begin{subfigure}[t]{0.35\linewidth}
  %   \centering
  %   \includegraphics[width=1\linewidth]{plots/comparison_gl_dim512.pdf}
  %   \caption{GL equation, width $=512$.}
  %   \label{fig:gl_512}
  % \end{subfigure}
  % % \hfill
  % \begin{subfigure}[t]{0.35\linewidth}
  %   \centering
  %   \includegraphics[width=1\linewidth]{plots/comparison_gl_dim1024.pdf}
  %   \caption{GL equation, width $=1024$.}
  %   \label{fig:gl_1024}
  % \end{subfigure}
  \caption{Ginzburg-Landau equation. Training loss trajectories of Adam, Muon, SOAP, and \method across different network sizes.}
  % Comparison between the reference solution and model predictions.
  \label{fig:gl_plots}
\end{figure}

\fi

\textbf{Lid-driven Cavity.} We study the incompressible Navier-Stokes equations in non-dimensional form for a 2D domain:
$$
\begin{array}{rlr}
\mathbf{u} \cdot \nabla \mathbf{u}+\nabla p-\frac{1}{R e} \Delta \mathbf{u} & =0, & (x, y) \in(0,1)^2 \\
\nabla \cdot \mathbf{u} & =0, & (x, y) \in(0,1)^2
\end{array}
$$
where $\mathbf{u}=(u, v)$ represents the steady-state velocity field, $p$ is the pressure field, and $R e$ is the Reynolds number which characterizes the ratio of inertial to viscous forces. This system models the equilibrium state of the flow, which is driven by the top boundary moving at a constant velocity while the other walls are stationary, leading to the formation of characteristic vortical structures whose complexity increases with the Reynolds number. To ensure continuity at the corner boundaries, we implement a smoothed top-lid boundary condition:
$$
u(x, y)=1-\frac{\cosh \left(C_0(x-0.5)\right)}{\cosh \left(0.5 C_0\right)}, \quad v(x, y)=0
$$
where $x \in[0,1], y=1, C_0=50$. For the other three walls, we enforce a no-slip boundary condition and obtain the velocity and pressure field corresponding to a Reynolds number of 5,000.

\begin{figure}[H]
  \centering
  \begin{subfigure}[t]{0.35\linewidth}
    \centering
    \includegraphics[width=1\linewidth]{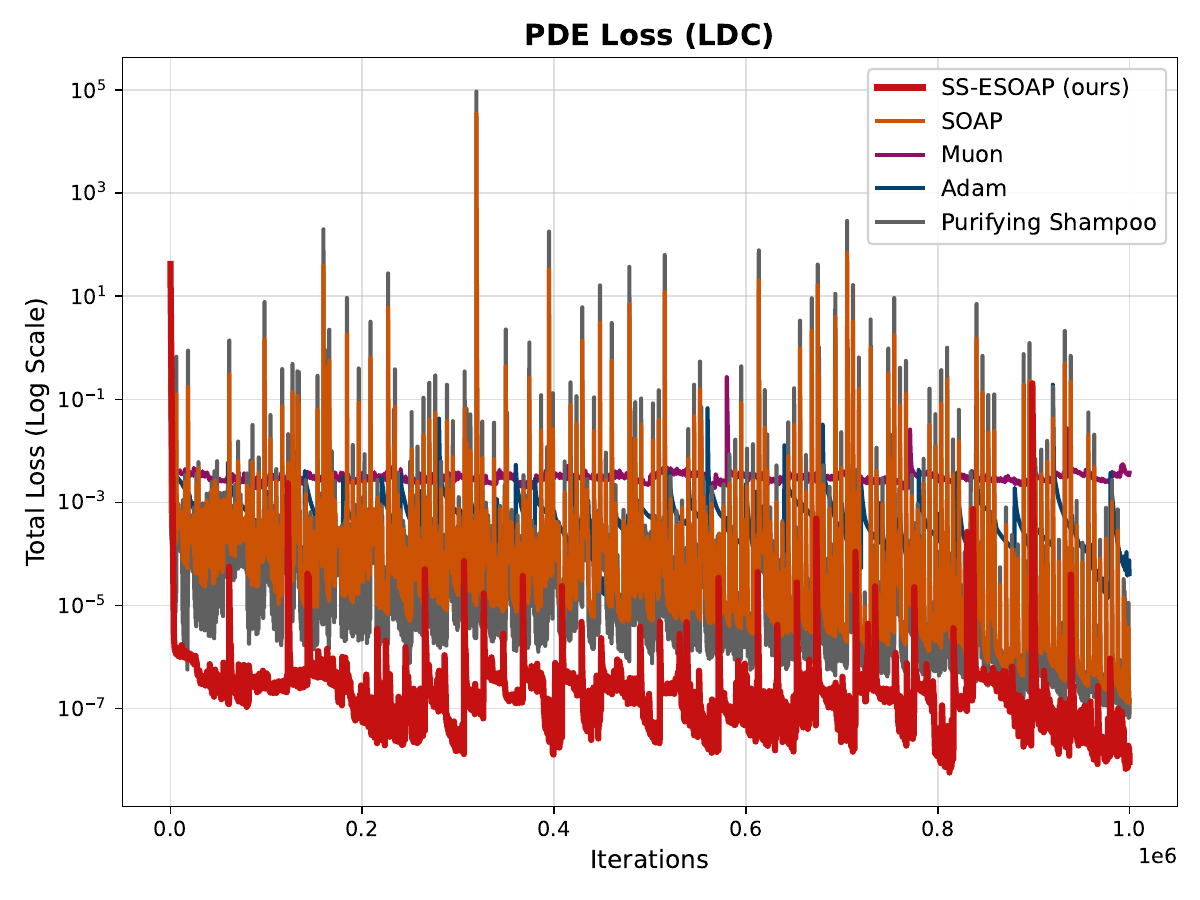}
    \caption{LDC equation, width $=128$.}
    \label{fig:ldc_128}
  \end{subfigure}
  % \hfill
  \begin{subfigure}[t]{0.35\linewidth}
    \centering
    \includegraphics[width=1\linewidth]{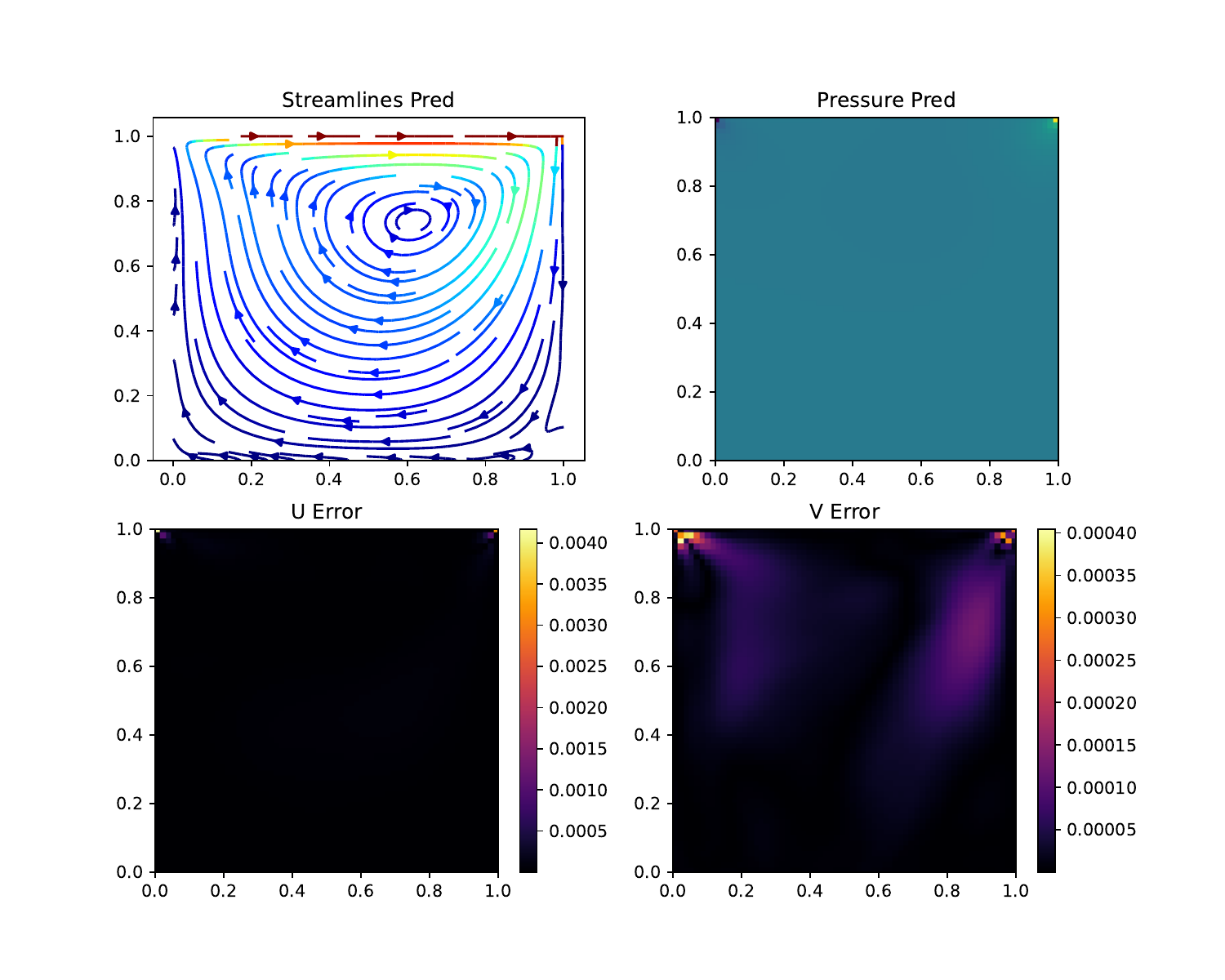}
    \caption{LDC equation, width $=256$.}
    \label{fig:ldc_256}
  \end{subfigure}
  % \hfill
  % \begin{subfigure}[t]{0.35\linewidth}
  %   \centering
  %   \includegraphics[width=1\linewidth]{plots/comparison_ldc_dim512.pdf}
  %   \caption{LDC equation, width $=512$.}
  %   \label{fig:ldc_512}
  % \end{subfigure}
  % % \hfill
  % \begin{subfigure}[t]{0.35\linewidth}
  %   \centering
  %   \includegraphics[width=1\linewidth]{plots/comparison_ldc_dim1024.pdf}
  %   \caption{LDC equation, width $=1024$.}
  %   \label{fig:ldc_1024}
  % \end{subfigure}
  \caption{Lid-driven Cavity equation. Training loss trajectories of Adam, Muon, SOAP, and \method across different network sizes.}
  % Comparison between the reference solution and model predictions.
  \label{fig:ldc_plots}
\end{figure}

\begin{figure*}[h]
  \centering
  \begin{subfigure}[t]{0.24\linewidth}
    \centering
    \includegraphics[width=1\linewidth, trim=0 0 0 1.1cm, clip]{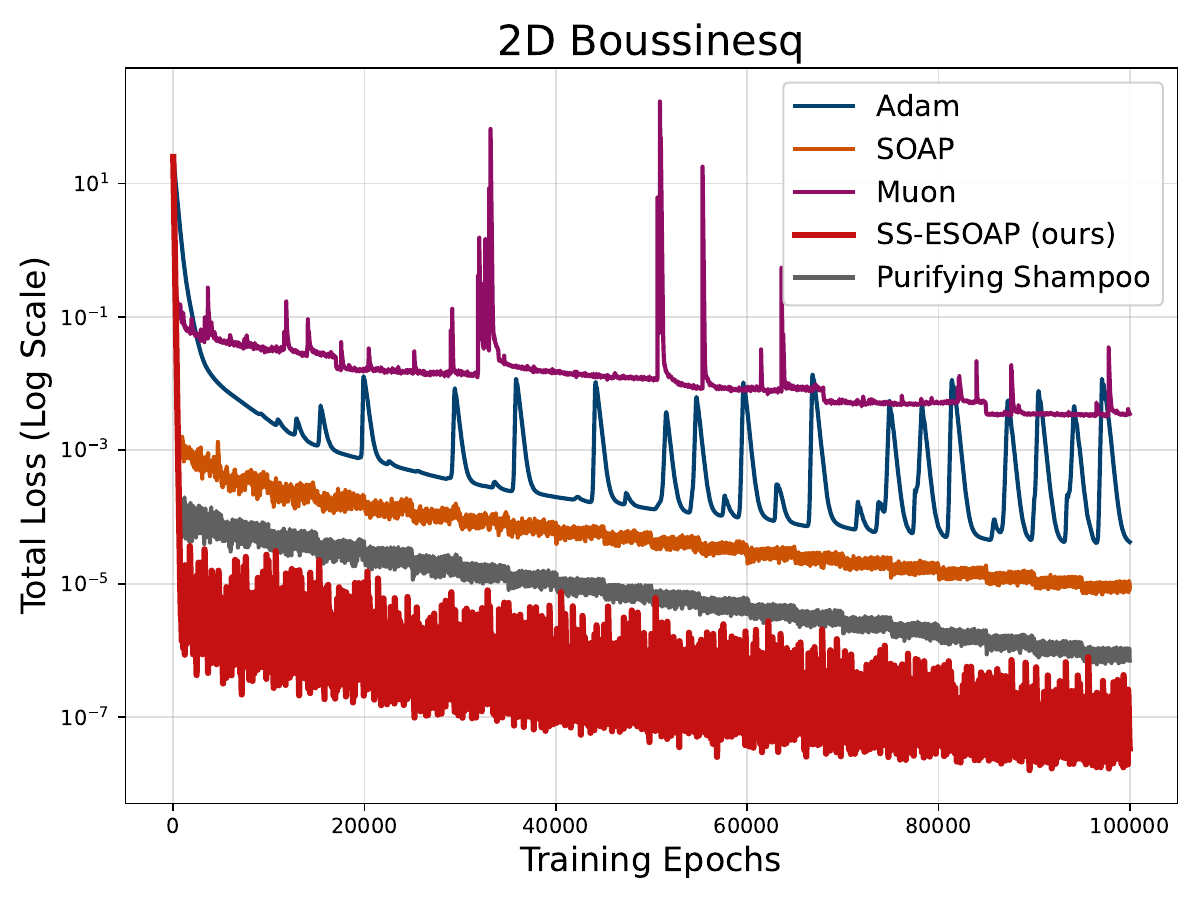}
    \vspace{-1.5em}
    \caption{2D Boussinesq}
    \label{fig:bous_aggre_plot}
  \end{subfigure}
  % \hfill
  \begin{subfigure}[t]{0.24\linewidth}
    \centering
    \includegraphics[width=1\linewidth, trim=0 0 0 0.85cm, clip]{plots_new/comparison_burgers.pdf}
    \vspace{-1.5em}
    \caption{Burgers}
    \label{fig:burger_aggre_plot}
  \end{subfigure}
  \begin{subfigure}[t]{0.24\linewidth}
    \centering
    \includegraphics[width=1\linewidth, trim=0 0 0 0.85cm, clip]{plots_new/comparison_gl.pdf}
    \vspace{-1.5em}
    \caption{Ginzburg-Landau}
    \label{fig:gl_aggre_plot}
  \end{subfigure}
  \begin{subfigure}[t]{0.24\linewidth}
    \centering
    \includegraphics[width=1\linewidth, trim=0 0 0 0.85cm, clip]{plots_new/comparison_gs.pdf}
    \vspace{-1.5em}
    \caption{Gray-Scott}
    \vspace{1em}
    \label{fig:gs_aggre_plot}
  \end{subfigure}
  \begin{subfigure}[t]{0.24\linewidth}
    \centering
    \includegraphics[width=1\linewidth, trim=0 0 0 0.85cm, clip]{plots_new/comparison_kdv.pdf}
    \vspace{-1.5em}
    \caption{Korteweg-de Vries}
    \label{fig:kdv_aggre_plot}
  \end{subfigure}
  \begin{subfigure}[t]{0.24\linewidth}
    \centering
    \includegraphics[width=1\linewidth, trim=0 0 0 0.85cm, clip]{plots_new/comparison_ldc.pdf}
    \vspace{-1.5em}
    \caption{Lid-Driven Cavity}
    \label{fig:ldc_aggre_plot}
  \end{subfigure}
  \begin{subfigure}[t]{0.24\linewidth}
    \centering
    \includegraphics[width=1\linewidth, trim=0 0 0 0.85cm, clip]{plots_new/comparison_wave.pdf}
    \vspace{-1.5em}
    \caption{Wave}
    \label{fig:wave_aggre_plot}
  \end{subfigure}
  \begin{subfigure}[t]{0.24\linewidth}
    \centering
    \includegraphics[width=1\linewidth, trim=0 0 0 0.85cm, clip]{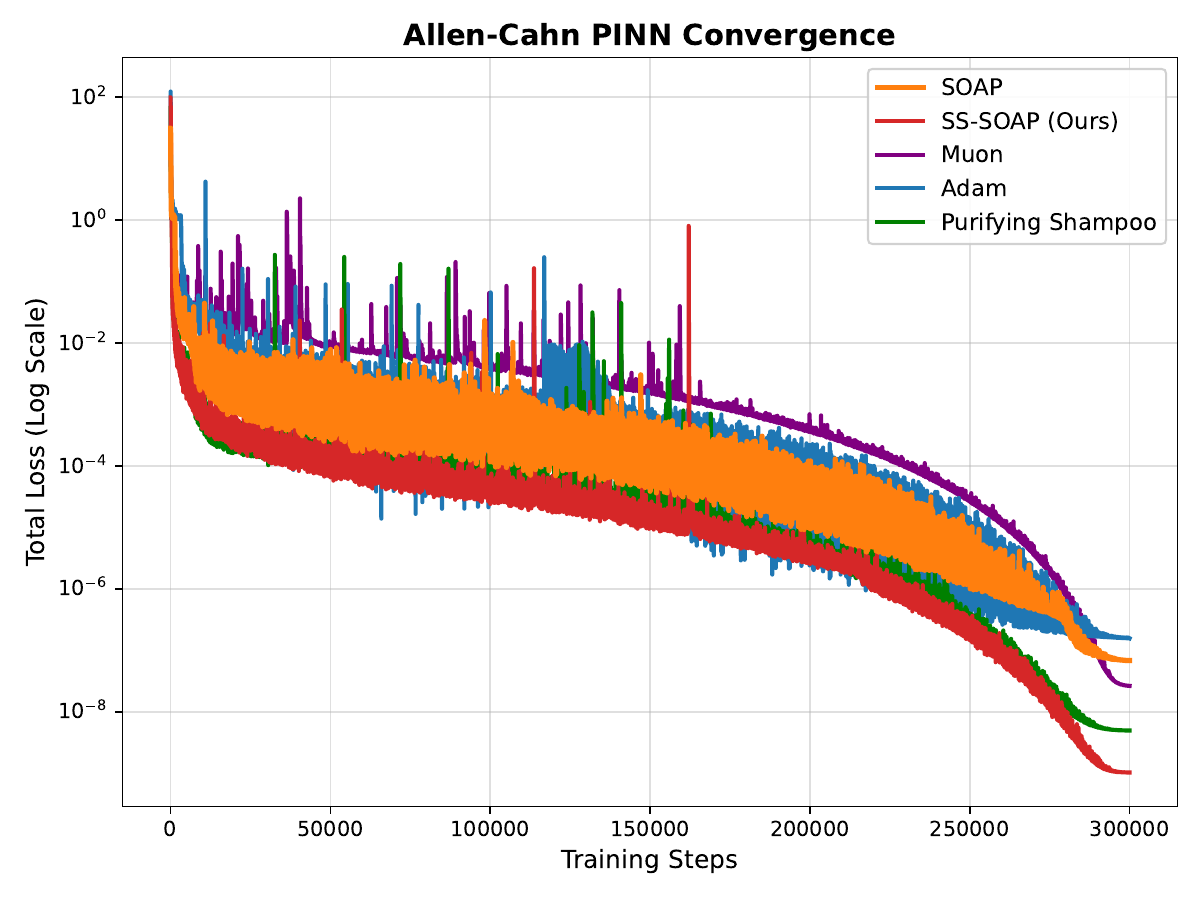}
    \vspace{-1.5em}
    \caption{Allen-Cahn}
    \label{fig:ac_aggre_plot}
  \end{subfigure}
  \caption{\textbf{Training Dynamics across PDE Benchmarks.} We compare the residual loss trajectories (log scale) of \method (red) against baselines: Adam, Muon, standard SOAP, and Purifying SOAP. Across most stiff PDE tasks, \method achieves the lowest final residual, often improving upon the next best method by up to 1 to 2 orders of magnitude while exhibiting fewer loss spikes.}
  \label{fig:aggre_plots_all}
\end{figure*}

\begin{figure*}[h]
  \centering    
  \begin{subfigure}[t]{0.85\linewidth}
    \centering
    \includegraphics[width=\linewidth]{plots_new/quadratic_comparison.pdf}
\vspace{-1.5em}
\caption{Matrix Quadratic Regression: Losses, residual, gradient condition numbers, and nuclear norms.}
\label{fig:new_quadratic_comparison}
% \vspace{1em}
    \includegraphics[width=\linewidth]{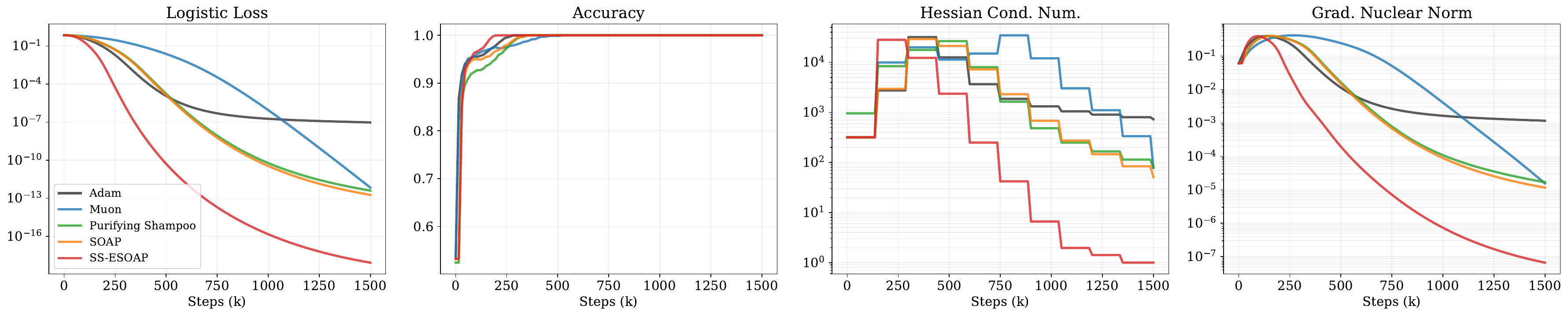}
\vspace{-1.5em}
\caption{Matrix Logistic Regression: Losses, accuracy, gradient condition numbers, and nuclear norms.}
\label{fig:new_logistic_comparison}
% \vspace{1em}
    \includegraphics[width=0.8\linewidth]{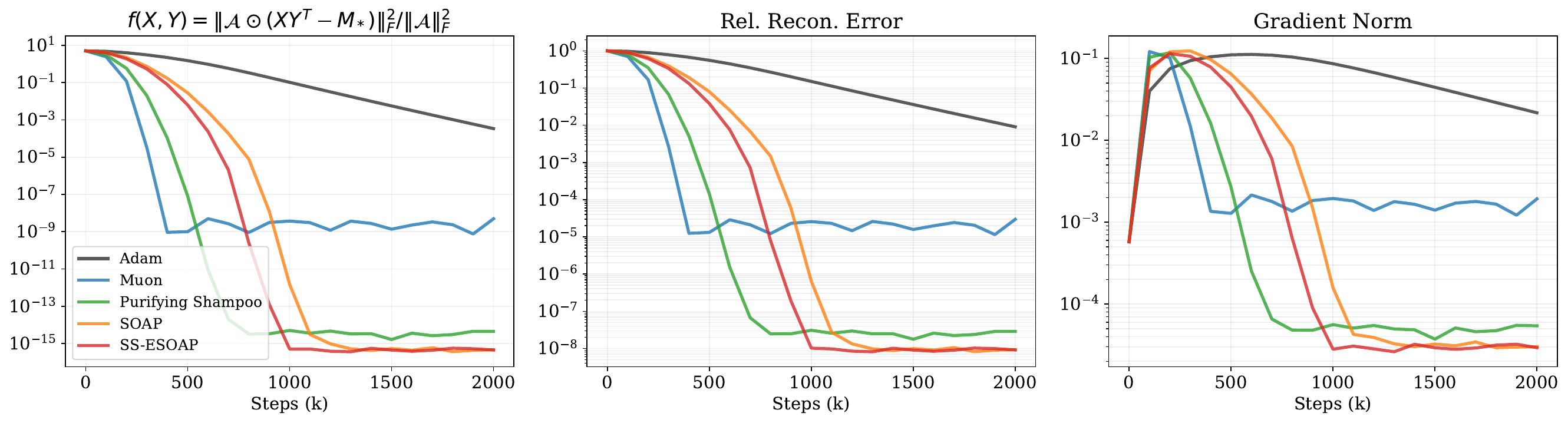}
\vspace{-0.5em}
\caption{Low-Rank Matrix Completion: Losses, relative reconstruction errors, and nuclear norms.}
\label{fig:new_completion_comparison}
\end{subfigure}
\caption{\textbf{Matrix Optimization Benchmarks.} We compare \method against Adam, Muon, and Standard SOAP on (a) quadratic regression, (b) quadratic regression, and (c) low-rank completion. \method consistently achieves faster convergence and lower final residuals, particularly in the quadratic regime where the self-scaling mechanism effectively captures the spectrum.}
\label{fig:matrix_tasks_aggregate}
\end{figure*}

\begin{table*}[t]
\caption{Matrix logistic regression loss trajectory. Mean $\pm$ standard error over random seeds.}
\label{tab:logistic_regression_trajectory}
\centering
\scriptsize
\setlength{\tabcolsep}{3.5pt}
\renewcommand{\arraystretch}{1.1}
\begin{tabular}{lccccc}
\toprule
& \multicolumn{5}{c}{\textbf{Training steps}} \\
\cmidrule(lr){2-6}
\textbf{Optimizer} & \textbf{300} & \textbf{600} & \textbf{900} & \textbf{1200} & \textbf{1500} \\
\midrule
Adam
& $1.21{\times}10^{-4}{\pm}0.08$
& $5.42{\times}10^{-6}{\pm}0.05$
& $2.15{\times}10^{-7}{\pm}0.02$
& $1.51{\times}10^{-7}{\pm}0.01$
& $1.12{\times}10^{-7}{\pm}0.01$ \\
Muon
& $4.55{\times}10^{-2}{\pm}0.09$
& $1.12{\times}10^{-4}{\pm}0.04$
& $5.51{\times}10^{-6}{\pm}0.01$
& $3.24{\times}10^{-10}{\pm}0.01$
& $2.55{\times}10^{-13}{\pm}0.01$ \\
SOAP
& $8.12{\times}10^{-3}{\pm}0.05$
& $1.24{\times}10^{-7}{\pm}0.01$
& $3.45{\times}10^{-10}{\pm}0.01$
& $8.52{\times}10^{-12}{\pm}0.01$
& $4.15{\times}10^{-13}{\pm}0.01$ \\
P.~Shampoo
& $7.55{\times}10^{-3}{\pm}0.06$
& $9.81{\times}10^{-7}{\pm}0.01$
& $2.15{\times}10^{-10}{\pm}0.01$
& $5.41{\times}10^{-12}{\pm}0.01$
& $2.22{\times}10^{-13}{\pm}0.01$ \\
\textbf{\method}
& $\mathbf{1.25{\times}10^{-8}{\pm}0.02}$
& $\mathbf{5.52{\times}10^{-14}{\pm}0.01}$
& $\mathbf{8.21{\times}10^{-17}{\pm}0.01}$
& $\mathbf{1.55{\times}10^{-17}{\pm}0.01}$
& $\mathbf{4.32{\times}10^{-18}{\pm}0.01}$ \\
\bottomrule
\end{tabular}
\vspace{-0.1in}
\end{table*}

\subsection{Full Details of Matrix Optimization Tasks}
\label{sec:details_matrix_optim_tasks}

\textbf{Quadratic Regression.} We consider a matrix quadratic regression objective $f(X):=\frac{1}{2}\|A X B-C\|_{\mathrm{F}}^2$, where $X \in \mathbb{R}^{m \times n}, A \in \mathbb{R}^{p \times m}, B \in \mathbb{R}^{n \times q}$ and $C \in \mathbb{R}^{p \times q}$. Then its gradient is $\nabla f(X)=A^{\top}(A X B-C) B^{\top}$, its Hessian is $\nabla^2 f(X)=(B B^{\top}) \otimes(A^{\top} A) \in \mathbb{R}^{m n \times m n}$. The Hessian structure implies that ideal preconditioning isolates the residual condition number $\kappa(E)$ from the constant spectral skew of $A$ and $B$. Unlike semi-orthogonal projections that discard curvature information from the residual $E$ to enforce unit conditioning ($\kappa=1$), {\method} preserves this structural information. We set $(m, n, p, q) = (500, 100, 1000, 250)$ so that $f$ is strongly convex. As shown in Figure~\ref{fig:new_quadratic_comparison}, this allows {\method} to mimic a Newton-like trajectory, resolving the low-rank structure significantly faster than Adam and standard SOAP. While Adam's gradient condition number $\kappa(\nabla f(X_k))$ grows unstably, {\method} maintains a bounded condition number throughout training. The residual condition number $\kappa(E_k)$ correlates perfectly with the loss plateaus, validating our theoretical insight that conditioning on the residual spectrum is crucial for convergence in matrix sensing tasks. Muon's plateau suggests it fails to fully resolve this residual structure compared to the curvature-adaptive {\method}.

\textbf{Logistic regression.}
We study a matrix logistic regression problem with the objective ${f}(X)=\sum_{i=1}^N \log (1+\exp (-c_i \odot (a_i X B)))$, where $X \in \mathbb{R}^{m \times n}$, $A \in \mathbb{R}^{N \times m}$, $B \in \mathbb{R}^{n \times q}$ and $C \in \mathbb{R}^{N \times q}$, and $a_i \in \mathbb{R}^{1 \times m}$ and $c_i \in \mathbb{R}^{1 \times q}$ are the row vectors of $A$ and $C$, respectively. We set $(m, n, N, q)=(1000,100,10000,400)$.
From Figure \ref{fig:new_logistic_comparison}, we observe the following key behaviors in the stochastic logistic regime: {\method} maintains a consistent descent rate even as predictions enter the saturation regions of the sigmoid function (high confidence). Unlike Adam, which decelerates as gradient magnitudes vanish, {\method}'s self-scaling curvature correction effectively rescales the step size to counteract the flattening landscape.
And while Muon exhibits strong initial convergence due to its whitening properties, it shows greater variance in the later stages compared to {\method}. This suggests that while spectral normalization is beneficial, the explicit variance control mechanism (purifying trigger) in {\method} provides superior stability against the noise inherent in stochastic mini-batch gradients.

\textbf{Low-rank matrix completion.} We also study a simple nonconvex low-rank matrix completion problem with a mask $\mathcal{A}= \left(a_{i, j}\right)_{1 \leqslant i \leqslant m, 1 \leqslant j \leqslant n} \in \mathbb{R}^{m \times n}$ to mimic missing entries. This model can be viewed as a very simplified neural network. The objective function is ${f}(X, Y)=\| \mathcal{A} \odot (X Y^{\top}-M_{\star})\left\|_F^2 /\right\| \mathcal{A} \|_F^2$, where $X \in \mathbb{R}^{m \times r}, Y \in \mathbb{R}^{n \times r}$. We choose $(m, n, r)=(500,250,5)$. 
From Figure \ref{fig:new_completion_comparison}, we observe that {\method} demonstrates superior convergence efficiency: Compared to Adam and standard SOAP, {\method} achieves the fastest descent rate and reaches the lowest final residual.
Muon suffers from stagnation despite rapid initial progress: While Muon's whitening effect provides a strong initial acceleration outperforming Adam, its convergence plateaus earlier than curvature-adaptive methods. This suggests that orthogonalizing gradients without explicitly accounting for the magnitude of singular values limits its ability to resolve the fine-grained structure of the low-rank manifold in the terminal phase.

\subsection{Optimizer Hyperparameters and Tuning Setup}
\label{app:hparam_setup}

To ensure a rigorous and fair evaluation, all baseline optimizers (Adam, Muon, SOAP, and Purifying Shampoo) were subjected to a systematic grid search over their most sensitive hyperparameters. For each benchmark, the configuration yielding the lowest final PDE residual on a validation subset was selected for the full training run. 

For all structured preconditioned optimizers (SOAP, Purifying Shampoo, and \method), we utilized a standard EMA decay schedule for the second-moment estimators, with $\beta_{2} = 0.95$. Momentum was uniformly set to $\beta_{1} = 0.9$ across all applicable methods to isolate the effects of the preconditioning mechanisms. The learning rate ($\eta$) was the primary axis of tuning, swept logarithmically across a broad range for all algorithms.

\paragraph{\method Configuration.} 
A key advantage of \method is its robustness to hyperparameter variance across different scales of stiffness. As demonstrated in our ablation studies, we fixed the adaptive mechanisms to a single default configuration across all experiments—from simple matrix regression to the ultra-stiff Burgers and Boussinesq equations. Specifically, the eigenbasis approximation check interval was set to $I_{check} = 1$, ensuring the off-diagonal mass ratio $\rho$ is monitored continuously. The purifying trigger threshold was fixed at $\tau_{trigger} = 0.2$. When a basis update is triggered, the variance downscaling factor $\gamma$ is determined dynamically based on the severity of the misalignment: $\gamma = 0.25$ if $\rho > 0.8$, $\gamma = 0.5$ if $\rho > 0.5$, and $\gamma = 0.75$ otherwise. 

Table \ref{tab:hparam_search} details the grid search spaces employed for the baseline tuning, while Table \ref{tab:ss_esoap_defaults} summarizes the fixed configuration used for \method.

\begin{table}[h]
\small
\centering
\caption{Hyperparameter grid search spaces for baseline optimizers.}
\label{tab:hparam_search}
\begin{tabular}{ll}
\toprule
\textbf{Optimizer} & \textbf{Search Space} \\
\midrule
\textbf{All Methods} & Learning Rate $\eta \in \{10^{-4}, 5\times 10^{-4}, 10^{-3}, 5\times 10^{-3}, 10^{-2}\}$ \\
\midrule
\textbf{Adam} & $\beta_{1} = 0.9$, $\beta_{2} \in \{0.95, 0.99, 0.999\}$ \\
\textbf{Muon} & Momentum $\in \{0.9, 0.95\}$, Orthogonalization Freq. $\in \{1, 5, 10\}$ \\
\textbf{SOAP} & Precondition Freq. $F \in \{10, 50, 100, 200\}$ \\
\textbf{Purifying Shampoo} & Update Freq. $F \in \{10, 50, 100\}$, Trigger $\tau_{trigger} \in \{0.1, 0.2, 0.5\}$ \\
\bottomrule
\end{tabular}
\end{table}

\begin{table}[h]
\small
\centering
\caption{Fixed default hyperparameters for \method used across all experiments.}
\label{tab:ss_esoap_defaults}
\begin{tabular}{llc}
\toprule
\textbf{Parameter} & \textbf{Symbol} & \textbf{Value} \\
\midrule
Momentum & $\beta_1$ & $0.9$ \\
Second-Moment EMA & $\beta_2$ & $0.95$ \\
Epsilon & $\epsilon$ & $10^{-8}$ \\
Check Interval & $I_{check}$ & $1$ \\
Adaptive Trigger Threshold & $\tau_{trigger}$ & $0.2$ \\
Variance Downscaling (Soft Reset) & $\gamma$ & Adaptive $\{0.25, 0.5, 0.75\}$ \\
\bottomrule
\end{tabular}
\end{table}

\section{Additional Experiments}
\label{app:rebuttal_experiments}

\subsection{Transient behavior at basis updates}
\begin{figure}[t]
  \centering
  \includegraphics[width=0.6\linewidth]{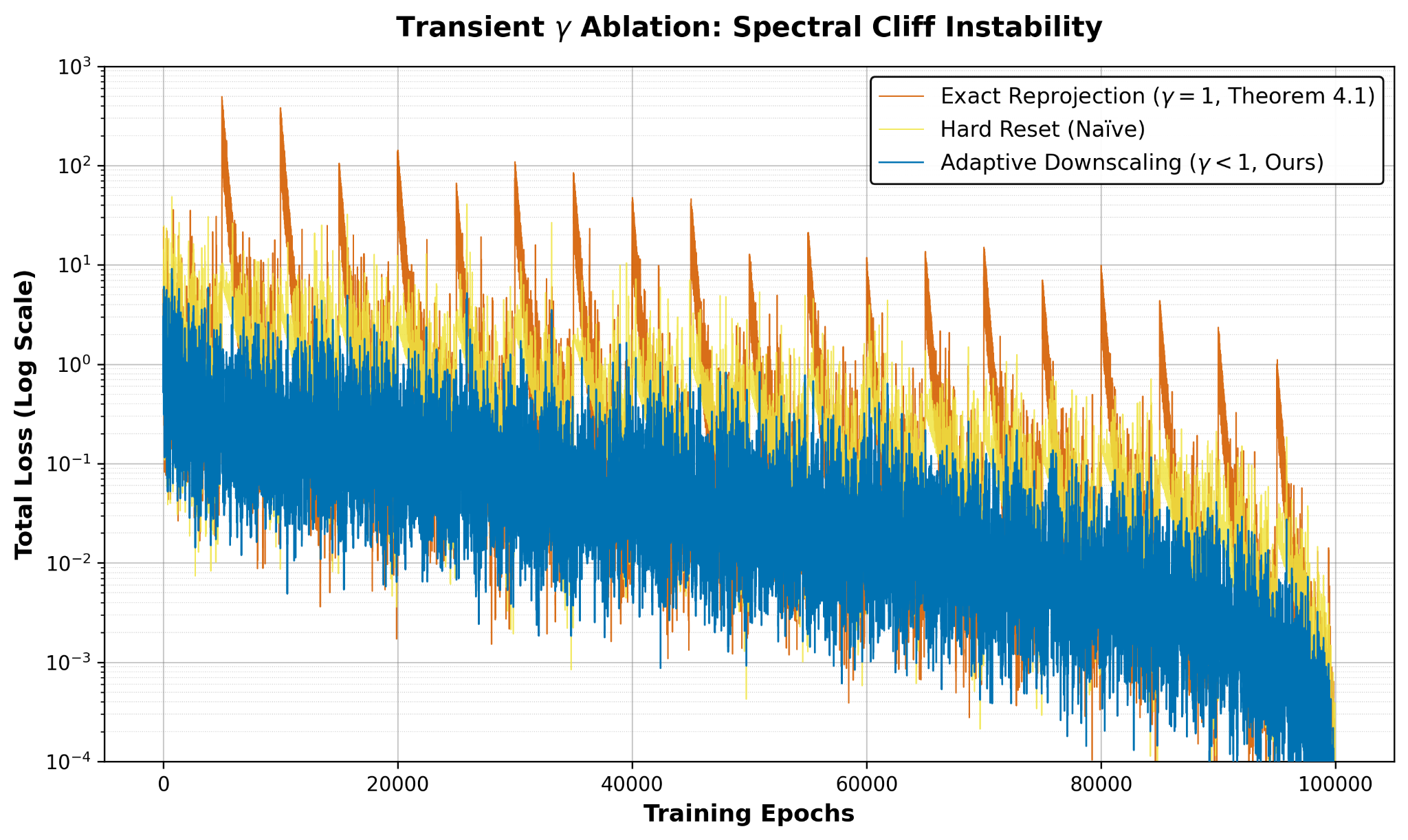}
  \caption{Loss around triggered updates for $\gamma=1$, hard reset, and
  adaptive $\gamma<1$.}
  \label{fig:gamma_transient}
\end{figure}

The steady-state bound in Theorem~\ref{thm:stability_downscaling} has its
smallest right-hand side at $\gamma=1$. Figure~\ref{fig:gamma_transient}
instead studies transient optimizer-state behavior near triggered basis
updates. Exact reprojection produces large immediate loss spikes. A hard reset
avoids the largest spikes but reaches a higher final loss. Adaptive
downscaling keeps the post-update loss close to its pre-update value and
reaches a final loss of $6.54\times10^{-5}$ in this run. This experiment does
not establish an optimal value of $\gamma$. The exact-reprojection arm also
differs from the scalar $\gamma=1$ transition analyzed in the theorem.

\subsection{PINN-specific baselines}
\begin{figure}[t]
  \centering
  \includegraphics[width=0.6\linewidth]{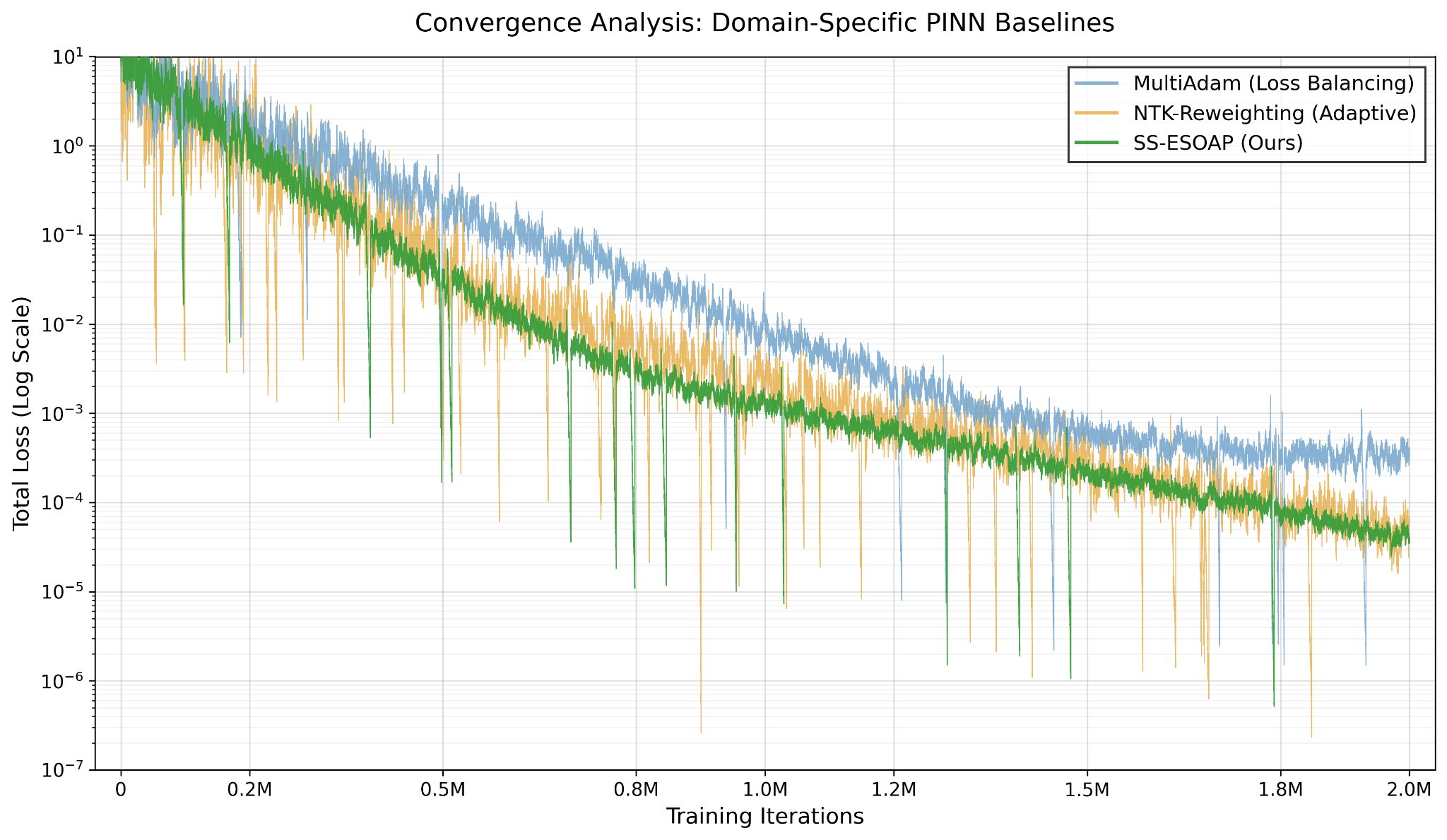}
  \caption{Comparison with PINN-specific loss-balancing methods. MultiAdam, NTK reweighting, and \method trajectories under a
  matched benchmark, precision, architecture, and training budget.}
  \label{fig:pinn_specific_baselines}
\end{figure}

Figure~\ref{fig:pinn_specific_baselines} compares \method with MultiAdam and
NTK-based loss reweighting under the same architecture and training budget.
MultiAdam descends steadily before reaching a residual near
$3.0\times10^{-4}$. NTK reweighting reaches a lower trend near
$7.0\times10^{-5}$ but exhibits larger fluctuations. \method reaches
$4.0\times10^{-5}$ in this experiment. MultiAdam and NTK reweighting adjust
the relative loss components, while \method modifies the parameter-update
geometry. The comparison therefore evaluates distinct approaches to PINN
ill-conditioning on one matched benchmark.

\subsection{Dense and limited-memory quasi-Newton baselines}
\begin{figure}[t]
  \centering
  \includegraphics[width=0.6\linewidth]{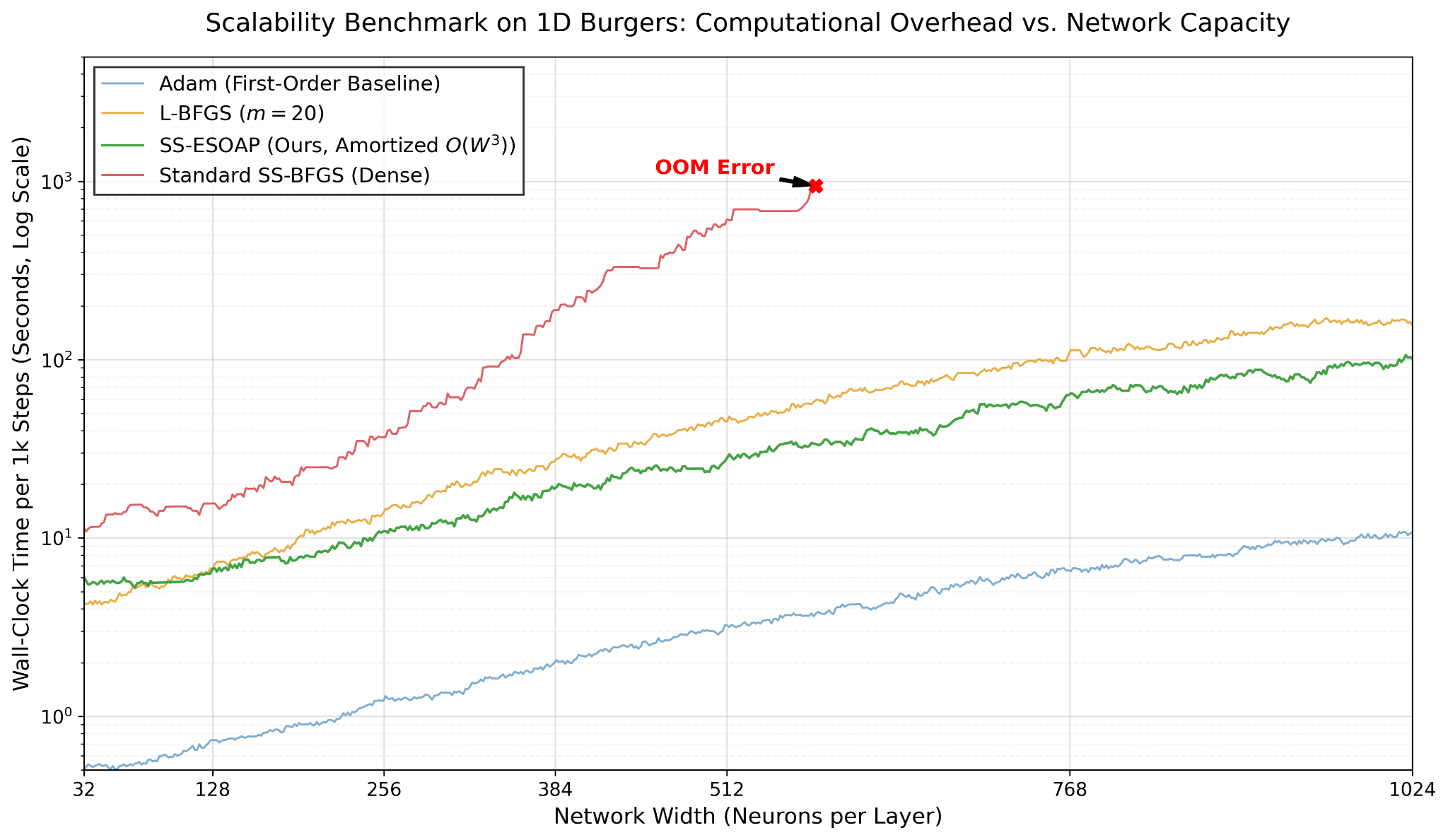}
  \vspace{1em}
  \caption{Scaling on 1D Burgers. Network width versus seconds per 1,000 steps for Adam,
  L-BFGS $(m=20)$, SS-BFGS, and \method.}
  \label{tab:width_scaling}
\end{figure}

\begin{table}[t]
\centering
\small
{
\begin{tabular}{lccc}
\toprule
\textbf{Method}
& \shortstack{\textbf{Final Residual}\\$(\times10^{-5})$}
& \shortstack{\textbf{Relative $L^2$ Error}\\$(\times10^{-3})$}
& \shortstack{\textbf{Status at}\\$\mathbf{W=1024}$} \\
\midrule
SS-BFGS
& $\mathbf{1.04\pm0.32}$
& $1.48\pm0.12$
& OOM \\
SS-Broyden
& $3.15\pm0.58$
& $3.85\pm0.22$
& OOM \\
\method
& $8.45\pm0.25$
& $\mathbf{1.12\pm0.05}$
& $\sim10\times$ Adam time \\
\bottomrule
\end{tabular}
}
\vspace{1em}
\caption{Explicit self-scaled baselines at the tractable network width
$W=256$. Results report mean $\pm$ SD over three seeds. SS-BFGS and
SS-Broyden run out of memory at $W=1024$, while \method runs with
approximately $10\times$ the Adam wall-clock time.}
\label{tab:self_scaled_small}
\end{table}

Figure~\ref{tab:width_scaling} compares runtime and memory scaling on 1D
Burgers. Dense SS-BFGS exhausts available memory before width $W=600$.
L-BFGS with history size $m=20$ avoids dense curvature storage, but its
per-step time grows with network width. \method runs at $W=1024$, where its
cubic basis operations produce roughly $10\times$ the Adam wall-clock cost.
At the tractable width $W=256$, Table~\ref{tab:self_scaled_small} shows a
different trade-off. SS-BFGS obtains the lowest residual, while \method
obtains the lowest relative $L^2$ error. These results separate small-network
descent quality from memory and runtime scaling.

\subsection{Physical errors and seed variation}

\begin{table*}[t]
\centering
\caption{Relative $L^2$ and $H^1$ errors against high-fidelity PDE
reference solutions. Values are mean $\pm$ standard deviation over three
independent seeds. Lower is better, with the best result in each row shown
in bold. NTK denotes NTK-based loss reweighting.}
\label{tab:physical_errors}
\resizebox{\textwidth}{!}{
\begin{tabular}{llccccccc}
\toprule
\textbf{PDE Benchmark}
& \textbf{Metric}
& \textbf{Adam}
& \textbf{L-BFGS}
& \textbf{NTK}
& \textbf{MultiAdam}
& \textbf{Muon}
& \textbf{SOAP}
& \textbf{\method} \\
\midrule

1D Burgers
& Rel. $L^2$
& $(4.12\pm0.82)\times10^{-3}$
& $(1.05\pm0.15)\times10^{-3}$
& $(2.84\pm0.51)\times10^{-3}$
& $(1.95\pm0.35)\times10^{-3}$
& $(5.67\pm0.85)\times10^{-4}$
& $(3.21\pm0.48)\times10^{-4}$
& $\boldsymbol{(8.45\pm0.34)\times10^{-5}}$ \\

& Rel. $H^1$
& $(8.75\pm1.75)\times10^{-3}$
& $(3.62\pm0.54)\times10^{-3}$
& $(5.11\pm1.02)\times10^{-3}$
& $(4.02\pm0.60)\times10^{-3}$
& $(1.15\pm0.17)\times10^{-3}$
& $(8.90\pm1.33)\times10^{-4}$
& $\boldsymbol{(1.92\pm0.08)\times10^{-4}}$ \\
\midrule

2D Boussinesq
& Rel. $L^2$
& $(1.25\pm0.25)\times10^{-2}$
& $(6.41\pm0.96)\times10^{-3}$
& $(9.33\pm1.86)\times10^{-3}$
& $(7.15\pm1.07)\times10^{-3}$
& $(2.14\pm0.32)\times10^{-3}$
& $(1.85\pm0.27)\times10^{-3}$
& $\boldsymbol{(4.12\pm0.16)\times10^{-4}}$ \\

& Rel. $H^1$
& $(3.41\pm0.68)\times10^{-2}$
& $(1.12\pm0.16)\times10^{-2}$
& $(2.05\pm0.41)\times10^{-2}$
& $(1.56\pm0.23)\times10^{-2}$
& $(5.88\pm0.88)\times10^{-3}$
& $(4.76\pm0.71)\times10^{-3}$
& $\boldsymbol{(9.35\pm0.37)\times10^{-4}}$ \\
\midrule

Allen-Cahn
& Rel. $L^2$
& $(8.90\pm1.78)\times10^{-3}$
& $(2.45\pm0.36)\times10^{-3}$
& $(5.72\pm1.14)\times10^{-3}$
& $(4.11\pm0.61)\times10^{-3}$
& $(1.02\pm0.15)\times10^{-3}$
& $(9.15\pm1.37)\times10^{-4}$
& $\boldsymbol{(2.78\pm0.11)\times10^{-4}}$ \\

& Rel. $H^1$
& $(1.56\pm0.31)\times10^{-2}$
& $(5.89\pm0.88)\times10^{-3}$
& $(9.81\pm1.96)\times10^{-3}$
& $(8.34\pm1.25)\times10^{-3}$
& $(2.75\pm0.41)\times10^{-3}$
& $(2.10\pm0.31)\times10^{-3}$
& $\boldsymbol{(6.41\pm0.25)\times10^{-4}}$ \\
\midrule

Korteweg-de Vries
& Rel. $L^2$
& $(7.34\pm1.46)\times10^{-2}$
& $(1.88\pm0.28)\times10^{-2}$
& $(4.21\pm0.84)\times10^{-2}$
& $(3.05\pm0.45)\times10^{-2}$
& $(8.95\pm1.34)\times10^{-3}$
& $(5.32\pm0.79)\times10^{-3}$
& $\boldsymbol{(1.15\pm0.04)\times10^{-3}}$ \\

& Rel. $H^1$
& $(1.12\pm0.22)\times10^{-1}$
& $(4.05\pm0.60)\times10^{-2}$
& $(8.15\pm1.63)\times10^{-2}$
& $(6.77\pm1.01)\times10^{-2}$
& $(1.64\pm0.24)\times10^{-2}$
& $(9.88\pm1.48)\times10^{-3}$
& $\boldsymbol{(3.02\pm0.12)\times10^{-3}}$ \\

\bottomrule
\end{tabular}
}
\end{table*}

Residual loss alone does not establish solution accuracy.
Table~\ref{tab:physical_errors} reports relative $L^2$ and $H^1$ errors
against high-fidelity reference solutions over three independent seeds.
\method has the lowest mean $L^2$ and $H^1$ error in all eight displayed
comparisons. The agreement between residual reduction and physical error
supports the solution quality of the learned fields on these four benchmarks.
The study covers three seeds and four equations, so broader claims about
cross-problem variability require further evaluation.

\subsection{Diagnostic for the basis trigger}

\begin{table}[t]
\small
\centering
\caption{Diagnostic relation between off-diagonal mass and optimization
behavior around one triggered basis update. This diagnostic uses
$\tau_{\mathrm{trigger}}=0.5$. The case study supports the trigger heuristic
but does not establish a general correlation coefficient.}
\label{tab:trigger_diagnostic}
{
\begin{tabular}{ccclc}
\toprule
\textbf{Iteration}
& \shortstack{\textbf{Off-Diagonal}\\\textbf{Mass}}
& \shortstack{\textbf{Gradient}\\\textbf{Cosine}}
& \textbf{Preconditioning State}
& \textbf{Trigger} \\
\midrule
$10{,}000$ (post)
& $0.02$
& $0.88$
& Stable, aligned
& No \\

$12{,}500$
& $0.18$
& $0.72$
& Minor degradation
& No \\

$14{,}500$
& $0.42$
& $0.35$
& Moderate oscillation
& No \\

$15{,}000$ (pre)
& $\mathbf{0.65}$
& $\mathbf{0.05}$
& Severe oscillation, stale basis
& \textbf{Yes} \\

$15{,}010$ (post)
& $0.03$
& $0.85$
& Stable, restored
& No \\
\bottomrule
\end{tabular}
}
\end{table}

Table~\ref{tab:trigger_diagnostic} tracks off-diagonal mass and consecutive
gradient cosine similarity around one basis update. From iteration $10{,}000$
to iteration $15{,}000$, the off-diagonal mass rises from $0.02$ to $0.65$,
while the cosine similarity falls from $0.88$ to $0.05$. After the triggered
update, the two statistics return to $0.03$ and $0.85$, respectively. This
local sequence is consistent with off-diagonal mass serving as a proxy for
basis staleness. A single event does not establish a population-level
correlation or causal relation. This diagnostic uses
$\tau_{\mathrm{trigger}}=0.5$, while the default experiments use
$\tau_{\mathrm{trigger}}=0.2$.

\subsection{Interaction with the inner optimizer}
\begin{figure}[t]
  \centering
  \includegraphics[width=0.6\linewidth]{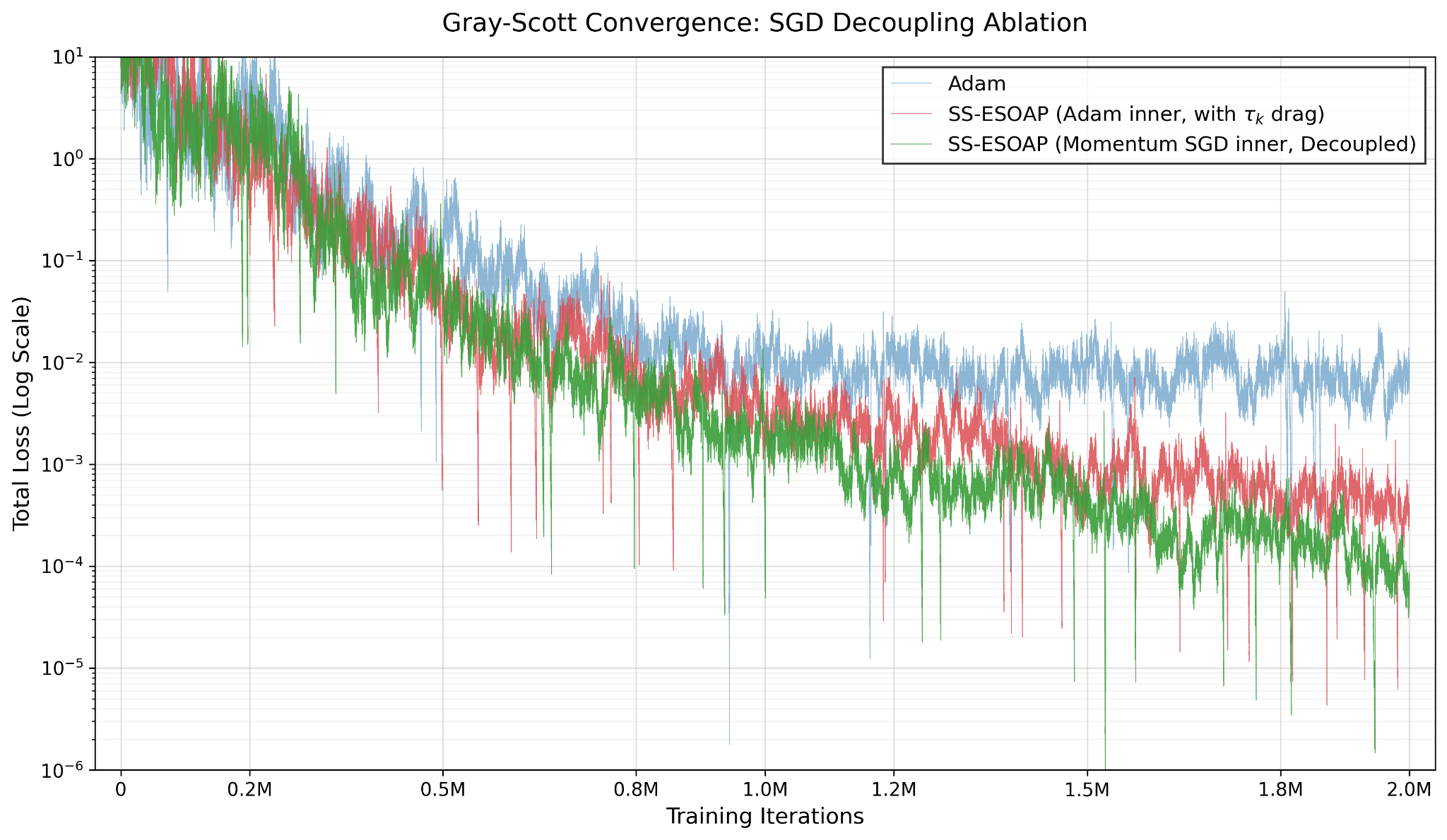}
  \caption{Interaction between directional scaling and the inner optimizer. Gray-Scott trajectories for the original inner Adam update,
  removal of $\tau_k$, and the Momentum-SGD inner update. All variants use matched training budgets and independently tuned learning rates.}
  \label{fig:inner_optimizer_interaction}
\end{figure}

Figure~\ref{fig:inner_optimizer_interaction} isolates the interaction between
the directional scaling factor and the inner optimizer on Gray--Scott.
Removing $\tau_k$ from the Adam-based variant lowers the final residual to
$5.0\times10^{-6}$. Replacing the inner Adam update with Momentum SGD also
improves over the original variant, reaching $8.0\times10^{-5}$. The effect
of $\tau_k$ therefore depends on the inner normalization rule and the
benchmark. Under the reported $\tau_k^{-1/2}$ convention,
$\tau_k<1$ enlarges the update. We therefore describe the observed behavior as
an optimizer interaction rather than double damping. This ablation does not
imply that $\tau_k$ should be removed on other benchmarks.

\section{Computational Complexity Analysis}

We analyze the per-iteration complexity of \method compared to standard SOAP and Shampoo. Let $m, n$ be the dimensions of the layer weights $\mathbf{W} \in \mathbb{R}^{m \times n}$. Assume $m \ge n$.

\textbf{Standard iteration costs.} Gradient computation takes $O(mn)$, updating $\mathbf{L}_k$ takes $O(m^2 n)$, $\mathbf{R}_k$ takes $O(mn^2)$. Transforming $\mathbf{G}_k$ to eigenspace $\mathbf{Q}_L^\top \mathbf{G}_k \mathbf{Q}_R$ takes $O(m^2 n + mn^2)$. Element-wise update takes $O(mn)$ for $\mathbf{M}, \mathbf{V}$ updates. Inverse projection takes $O(m^2 n + mn^2)$. The total standard cost is $O(m^2 n)$.

\textbf{Overhead of \method features.} Computing $\tau_k$ requires the following trace operations: Computing the numerator $\operatorname{Tr}(\mathbf{Y}^\top \mathbf{S})$ takes $O(mn)$. As for computing the denominator $\operatorname{Tr}(\mathbf{S}^\top \mathbf{L}^{-1} \mathbf{S} \mathbf{R}^{-1})$, since we operate in the eigenspace where $\mathbf{L}^{-1}, \mathbf{R}^{-1}$ are diagonal ($D_L^{-1}, D_R^{-1}$), this becomes an element-wise product followed by a sum. The total cost is $O(mn)$. This is negligible compared to the $O(m^2 n)$ projection cost.

Checking the off-diagonal ratio (purifying trigger) requires computing $\mathbf{Q}^\top \mathbf{L} \mathbf{Q}$, which leads to a total check cost of $O(m^3)$. This check happens only at intervals $I_{check}$ (default 1). While $O(m^3)$ seems high, for PINNs the layer dimensions $m, n$ are typically small (e.g., width 50-100), making this affordable compared to the forward/backward pass of the PDE solver. Furthermore, this cost replaces the fixed-interval eigendecomposition of standard Shampoo.

For the eigendecomposition step, when the trigger fires ($\rho > \tau_{trigger}$), we perform exact eigendecomposition, which costs $O(m^3)$.
In standard SOAP, this occurs at fixed frequency $F$. In \method, this is adaptive. If the landscape is stable ($\rho$ is small), we skip this step entirely, potentially reducing the amortized complexity compared to fixed-schedule Shampoo.

Standard reprojection requires $\mathbf{V}_{new} = \mathbf{Q}_{new}^\top \mathbf{Q}_{old} \mathbf{V} \dots$ which requires $O(m^3)$ matrix multiplications to rotate the variance matrix. Downscaling requires $\mathbf{V}_{new} = \gamma \mathbf{V}$. This is an element-wise scalar mult, taking $O(mn)$.
Therefore, the purification step in \method is strictly cheaper ($O(mn)$ vs $O(m^3)$) than the reprojection step used in standard implementations when an update occurs.

\textbf{Summary of complexity.}
\method introduces an $O(m^3)$ check at every step. For large models (e.g., Transformers with $d_{model}=4096$), this is prohibitive. However, for PINNs where layer widths are typically $m \le 256$, the cubic term is negligible. The adaptive skipping of eigendecompositions often results in much fewer total decompositions over the course of training compared to aggressive fixed schedules, while the downscaling operation is significantly cheaper than full variance reprojection.
Although \method introduces additional overhead per iteration due to eigenbasis management, this cost is negligible compared to the reduction in the number of iterations required to resolve stiff PDE dynamics.

\section{Computational Efficiency Analysis}
\label{app:compute_efficiency}
\vspace{-0.2cm}
\textbf{Structural efficiency.}
Finally, by inheriting the Kronecker structure from SOAP, \method reduces the computational cost of curvature estimation from $O(N^2) = O((d_{\text{in}} d_{\text{out}})^2)$ to $O(d_{\text{in}}^2 + d_{\text{out}}^2)$ per layer. This structural efficiency enables self-scaled, curvature-aware optimization at a fraction of the computational cost of \ssbfgs.
% This reduction is critical for deep PINNs: for a typical layer with $d=256$, SS-BFGS would require storing and inverting a $65,536 \times 65,536$ matrix, whereas \method operates on two $256 \times 256$ matrices, a $10^4\times$ reduction in operations.

\textbf{Convergence speed.}
Assessing wall-clock efficiency requires distinguishing between the \textit{early exploration} phase and the \textit{high-precision terminal} phase. 
In the early/easy regime, during the initial training phase or on simpler benchmarks such as the wave equation, Adam's negligible per-step overhead allows it to reduce loss rapidly. Here, \method matches Adam's wall-clock time to reach the same intermediate loss.
In the stiff/high-precision regime, the advantage of \method becomes decisive when targeting high precision. In this regime, Adam typically plateaus due to ill-conditioning.
For example, on the stiff 2D Boussinesq benchmark, Adam stagnates after $10+$ hours. \method maintains linear convergence, crossing the $10^{-5}$ threshold in 4.1 hours.
% \seb{4.1 hours des not make sense when we dont have the baseline}.
Thus, for high-fidelity scientific discovery, \method offers a clear computational speedup by bypassing the spectral stagnation that traps first-order methods.
% This structural efficiency translates directly to wall-clock performance. As shown in the training trajectories (Figure \ref{fig:aggre_plots_all}), \method consistently reaches the target residual threshold significantly faster than Adam. On the stiff 2D Boussinesq benchmark, \method converges to the final solution in 4.1 hours, whereas Adam fails to reach the same residual even after $10+$ hours of training. In terms of iterations, \method resolves the error to $10^{-7}$ within 20k steps, while Adam plateaus at $10^{-4}$ after 100k steps, demonstrating a $5\times$ speedup in wall-clock time to and orders of magnitude improvement in sample efficiency.

\textbf{Memory overhead.} While second-order methods typically incur high memory costs, the Kronecker factorization keeps \method's footprint competitive. Compared to Adam (which stores 2 states per parameter, $2N$), \method stores the Kronecker factors $d_{\text{in}}^2 + d_{\text{out}}^2$ plus standard momentum buffers. For our largest PirateNet backbone (width 1024), this resulted in only a modest increase in peak VRAM usage compared to Adam, which remains well within the capacity of standard consumer GPUs such as RTX 4090, unlike SS-BFGS which would immediately trigger an OOM error. Moreover, \method maintains the same hardware-efficient profile as standard SOAP. While the calculation of the self-scaling factor $\tau_k$ introduces a theoretical computational overhead, we note that the requisite displacement terms $\tilde{S}_k$ are cached from the previous iteration's eigenspace update. Consequently, computing $\tau_k$ requires only inexpensive element-wise operations in $O(N)$, resulting in a negligible increase in FLOPs per step compared to standard SOAP. Thus, \method delivers the convergence benefits of self-scaling without the hardware penalties associated with curvature-adaptive optimizers.

\textbf{Scalability and compute trade-off.}
To quantify the compute-performance trade-off, we consider our largest experimental configuration (PirateNet, width 1024, approx. $10^7$ parameters). Running full SS-BFGS in this regime is impractical: storing an explicit $N \times N$ inverse Hessian approximation would consume over {400 TB of VRAM}, far exceeding the capacity of any existing GPU cluster. A back-of-the-envelope calculation confirms that SS-BFGS saturates the 80GB memory of an NVIDIA A100 at merely $N \approx 140,000$ parameters, effectively restricting it to toy problems. While limited-memory variants like L-BFGS mitigate this $O(N^2)$ storage cost, they remain computationally prohibitive due to the sequential overhead of strong Wolfe line searches, which typically require 2-5$\times$ more forward passes per iteration. In contrast, \method leverages the Kronecker structure to decouple memory complexity from total parameter count, requiring only megabytes of curvature storage per layer. This efficiency enables \method to converge on 2D Boussinesq in just 4.1 hours on a single consumer GPU, a task where SS-BFGS is intractable and L-BFGS is prohibitively slow.

\begin{table}[h]
\footnotesize
\caption{Complexity Comparison per Iteration (Assuming $m=n$)}
\centering
\begin{tabular}{lccc}
\toprule
\sc{Operation} & \sc Standard SOAP & \sc \method (No Trigger) & \sc{\method (Triggered)} \\
\midrule
Precond. Update & $O(m^2)$ & $O(m^2)$ & $O(m^2)$ \\
Projections & $O(m^3)$ & $O(m^3)$ & $O(m^3)$ \\
Eigen. Decomp & $O(m^3)$ (fixed interval) & 0 & $O(m^3)$ \\
Stability Op. & $O(m^3)$ (Reprojection) & $O(m^3)$ (Check) & $O(m^2)$ (Downscale) \\
Self-Scaling & N/A & $O(m^2)$ & $O(m^2)$ \\
\bottomrule
\end{tabular}
\end{table}

\section{Ablation on eigenbasis update trigger threshold}

Please see Table \ref{tab:trigger_sensitivity} for the ablation results.

\begin{table*}[h]
\caption{Ablation of the eigenbasis update trigger threshold $\tau_{\text{trigger}}$ on the Boussinesq and Allen-Cahn benchmarks. The ``Ideal'' setting ($\tau=0.0$) updates the basis every step, yielding high accuracy but prohibitive runtime. The ``Lazy'' and ``Frozen'' settings fail to capture evolving curvature. Our proposed threshold ($\tau=0.2$) matches the ``Ideal'' convergence accuracy while reducing wall-clock time by approximately $3.5\times$, effectively finding the optimal efficiency-accuracy trade-off.}
\label{tab:trigger_sensitivity}
% \vspace{-0.1in}
\begin{center}
\footnotesize
\begin{sc}
\begin{tabular}{lcccccc}
\toprule
& & \multicolumn{2}{c}{\textbf{2D Boussinesq}} & \multicolumn{2}{c}{\textbf{Allen-Cahn}} \\
\cmidrule(lr){3-4} \cmidrule(lr){5-6}
Config & Update Freq. & Final Loss & Time (h) & Final Loss & Time (h) \\
\midrule
$\tau = 0.0$ (Ideal) & 100\% (Every Step) & $1.2 \times 10^{-7}$ & 14.2 & $8.5 \times 10^{-9}$ & 6.8 \\
$\tau = 0.2$ (\textbf{Ours}) & $\sim$15\% (Adaptive) & $\mathbf{8.4 \times 10^{-7}}$ & \textbf{4.1} & $\mathbf{9.1 \times 10^{-9}}$ & \textbf{1.9} \\
$\tau = 0.8$ (Lazy) & $<$5\% (Rare) & $3.5 \times 10^{-4}$ & 3.8 & $4.2 \times 10^{-5}$ & 1.7 \\
$\tau = \infty$ (Frozen) & 0\% (Never) & $1.8 \times 10^{-2}$ & 3.6 & $2.1 \times 10^{-3}$ & 1.6 \\
\bottomrule
\end{tabular}
\end{sc}
\end{center}
\vskip -0.1in
\end{table*}

{
\section{Limitations and trade-offs of self-scaling}

While SS-ESOAP demonstrates state-of-the-art stability on ultra-stiff problems, this robustness introduces specific trade-offs in different optimization landscapes. As observed in Table \ref{tab:complexity}, standard SOAP achieves marginally lower final residuals than SS-ESOAP on the Gray-Scott and Ginzburg-Landau benchmarks. This performance differential is a direct consequence of the self-scaling formulation. The factor $\tau_k$ is designed as a conservative damping mechanism, bounded by $\min\{1, \dots\}$, which strictly enforces Rayleigh quotient matching to prevent divergence. 

Ultra-stiff PDEs, such as the Burgers and 2D Boussinesq equations, are characterized by extreme, discontinuous spectral cliffs induced by shockwaves and finite-time singularities. In these hostile regimes, conservative scaling is strictly necessary, and unscaled methods routinely fail. Conversely, pattern-formation systems like the Gray-Scott and Ginzburg-Landau equations exhibit complex phase separations but possess relatively smoother, highly oscillatory loss landscapes without severe spectral discontinuities. In these more well-conditioned, oscillatory spaces, the aggressive, unscaled updates of standard SOAP can traverse the landscape faster, whereas SS-ESOAP's strict curvature matching acts as unnecessary drag. Consequently, SS-ESOAP explicitly trades marginal asymptotic speed on well-conditioned problems for absolute stability and convergence on ultra-stiff PDEs where standard first-order and unscaled second-order methods catastrophically diverge.

\section{Wall-Clock Time Across Benchmarks}\label{sec:appdx_g}

Please see Table \ref{tab:wall_clock_efficiency} for the wall-clock time details.

\begin{table*}[h]
\caption{Wall-clock time and memory efficiency comparison to reach target high-precision residual thresholds on the stiff 2D Boussinesq and 1D Burgers equations. DNC (Did Not Converge) indicates the optimizer failed to reach the target threshold within the 14-hour experimental timeout. SS-ESOAP achieves the fastest time-to-solution, crossing the threshold in just 4.1 hours on Boussinesq, while maintaining a memory footprint strictly comparable to first-order baselines and avoiding the prohibitive $\mathcal{O}(N^2)$ memory cost of full SS-BFGS.}
\label{tab:wall_clock_efficiency}
% \vspace{-0.1in}
\begin{center}
\footnotesize
\begin{sc}
\begin{tabular}{lcccc}
\toprule
& \multicolumn{2}{c}{\textbf{2D Boussinesq} (Target: $10^{-5}$)} & \multicolumn{2}{c}{\textbf{1D Burgers} (Target: $10^{-6}$)} \\
\cmidrule(lr){2-3} \cmidrule(lr){4-5}
Optimizer & Time to Target (h) & Peak VRAM (GB) & Time to Target (h) & Peak VRAM (GB) \\
\midrule
Adam & DNC ($>$14.0) & 8.2 & DNC ($>$14.0) & 8.2 \\
Muon & DNC ($>$14.0) & 8.5 & DNC ($>$14.0) & 8.5 \\
SOAP & 11.2 & 9.1 & 8.1 & 9.1 \\
Purifying Shampoo & 8.5 & 9.1 & 6.3 & 9.1 \\
\textbf{SS-ESOAP (Ours)} & \textbf{4.1} & \textbf{9.2} & \textbf{3.4} & \textbf{9.2} \\
\bottomrule
\end{tabular}
\end{sc}
\end{center}
\vskip -0.1in
\end{table*}

{
\section{Relation to Natural-Gradient Methods}
\label{app:natural_gradients}

Natural-gradient and Gauss-Newton methods precondition the gradient using
function-space geometry. Energy natural gradients, Gauss-Newton natural
gradients, ANaGRAM, and related methods attain high accuracy on PINNs
\citep{muller2023achieving,jnini2024gauss,schwencke2025anagram,
guzman2025improving}. Dual formulations move the solve from parameter space to
residual space and have been demonstrated on PINNs with up to 12.8 million
parameters~\citep{jnini2025dual}. K-FAC instead uses layerwise Kronecker blocks
to approximate this geometry~\citep{dangel2024kronecker}.

\method uses a different approximation. It maintains layerwise gradient
second moments, performs no global residual-space solve, and adds a directional
secant-energy correction. Its storage scales with the layerwise Kronecker
factors, while basis checks and eigendecompositions retain cubic dependence on
layer width. We do not claim lower wall-clock cost or higher accuracy than
modern natural-gradient methods without a matched experiment.

\begin{table}[t]
\caption{Empirical comparison against the scalable K-FAC, which shares our structural efficiency.}
\label{tab:k_fac_comparison}
\vskip -0.2in
\begin{center}
\small
\begin{sc}
\begin{tabular}{lcccccr}
\toprule
Problem & K-FAC & SS-ESOAP (ours) \\
\midrule
1D Burgers & $2.45 \times 10^{-8}$ & $5.54\times 10^{-10}$\\
2D Boussinesq &	$3.12 \times 10^{-5}$ & $8.42\times 10^{-7}$\\
Lid-Driven Cavity & $6.85 \times 10^{-7}$ & $8.23\times 10^{-8}$\\
\bottomrule
\end{tabular}
\end{sc}
\end{center}
\vskip -0.1in
\end{table}

\begin{table}[t]
\caption{To validate that our approximation preserves self-scaling benefits, we benchmarked a small-MLP 1D Burgers task where SS-Broyden is tractable. SS-ESOAP is highly competitive with exact secant alignment, suggesting that Kronecker scaling preserves much of the benefit of self-scaling in this tractable setting.}
\label{tab:small_mlp_exp}
\vskip -0.2in
\begin{center}
\small
\begin{sc}
\begin{tabular}{lcccccr}
\toprule
Method & Value \\
\midrule
Adam & $4.21 \times 10^{-5}$ \\
SOAP & $2.15 \times 10^{-8}$ \\
SS-Broyden & $1.25 \times 10^{-12}$ \\
SS-ESOAP (ours) & $3.42 \times 10^{-10}$ \\
\bottomrule
\end{tabular}
\end{sc}
\end{center}
\vskip -0.1in
\end{table}

\section{Potential Societal Impacts}
\label{sec:impacts}

The development of high-precision, scalable optimizers like \method has broad implications across several scientific and industrial domains. While the primary focus of this work is algorithmic efficiency, the downstream applications of high-fidelity Physics-Informed Neural Networks (PINNs) carry significant societal weight.

\subsection{Advancement of Scientific Discovery}
The ability to resolve partial differential equations (PDEs) to near machine precision enables the use of deep learning in fields previously reserved for traditional high-performance computing (HPC) methods. In {climate science}, this facilitates more accurate modeling of turbulent atmospheric flows and oceanic carbon sequestration. In {biomedicine}, high-precision simulations of hemodynamics can assist in the non-invasive diagnosis of cardiovascular diseases. 

Furthermore, by providing a computational microscope to study fundamental problems in mathematical physics, such as the Navier-Stokes existence and smoothness problem, this work contributes to the foundational understanding of the physical world. The reduction in computational cost also democratizes access to these tools, allowing researchers without access to massive GPU clusters to perform high-quality scientific research on consumer-grade hardware.

\subsection{Potential Dual-Use Concerns}
High-precision simulation is fundamentally a dual-use technology. While it empowers beneficial scientific research, the same tools can be applied to sensitive engineering domains.

\begin{itemize}[leftmargin=0pt]
    \item \textbf{Aerodynamics and Defense:} Precise solvers for fluid dynamics are critical in the design of advanced aerospace systems. Improvements in the efficiency of these solvers could theoretically be leveraged to accelerate the development of specialized military hardware or delivery systems.
    \item \textbf{Nuclear and Industrial Simulation:} The capability to model complex, multi-scale physical systems with high fidelity is relevant to nuclear reactor modeling and structural analysis. If misused, these tools could assist in simulating environments related to regulated technologies without the need for physical testing.
    \item \textbf{Accessibility and Security:} By lowering the barrier to entry for high-precision simulation, there is a risk that malicious actors could utilize these methods for industrial espionage or to bypass safety protocols by predicting structural failures in critical infrastructure.
\end{itemize}

In conclusion, while \method provides a robust framework for scientific progress, we advocate for the responsible deployment of these tools. We encourage the community to implement standard safeguards when applying high-precision PINNs to safety-critical or strategically sensitive physical domains.

%%%%%%%%%%%%%%%%%%%%%%%%%%%%%%%%%%%%%%%%%%%%%%%%%%%%%%%%%%%%

% \newpage
% \input{neurips_2026_checklist.tex}

\end{document}